\documentclass[twoside]{article}

\usepackage[accepted]{aistats2026}
\usepackage[round]{natbib}

\usepackage{algorithmic}
\usepackage{algorithm}

\usepackage{mathtools}
\usepackage{amsthm}
\usepackage{dsfont} 
\usepackage{tkz-tab} 

\usepackage{snaptodo}
\snaptodoset{block rise=2em}
\snaptodoset{margin block/.style={font=\tiny}}

\usepackage[normalem]{ulem}
\usepackage{xcolor}

\newcommand{\as}[1]{}
\newcommand{\es}[1]{}
\newcommand{\emmm}[1]{}
\usepackage{url}

\usepackage{wrapfig}
\usepackage{subcaption}
\usepackage{multicol}

\usepackage{booktabs}
\usepackage{hyperref}
\usepackage{cleveref}

\usepackage{minitoc}

\theoremstyle{plain}
\newtheorem{theorem}{Theorem}[section]

\newtheorem{lemma}[theorem]{Lemma}
\newtheorem{corollary}[theorem]{Corollary}
\theoremstyle{definition}

\newtheorem{assumption}[theorem]{Assumption}
\theoremstyle{remark}

\begin{document}

%

%

\twocolumn[

\aistatstitle{Do we need rebalancing strategies?  A theoretical and empirical study around SMOTE and its variants}

\aistatsauthor{ Abdoulaye Sakho \And Emmanuel Malherbe \And Erwan Scornet }

\aistatsaddress{
Artefact Research Center\\
Sorbonne Université, Université \\ 
Paris Cité, CNRS, LPSM, France
\And
Artefact Research Center\\
France
\And
Sorbonne Université, Université \\ 
Paris Cité, CNRS, LPSM, France
}

]

\begin{abstract}
  \textit{Synthetic Minority Oversampling Technique} (SMOTE) is a common rebalancing strategy for handling imbalanced tabular data sets. However, few works analyse SMOTE theoretically. In this paper, we derive several non-asymptotic upper bound on SMOTE density. From these results, we prove that SMOTE (with default parameter) tends to copy the original minority samples asymptotically. We confirm and illustrate this first theoretical behavior numerically.
  Furthermore, we prove that SMOTE density vanishes near the boundary of the support of the minority class distribution. We then adapt SMOTE based on our theoretical findings to introduce two new variants. 
  These strategies are compared on $13$ tabular data sets with $10$ state-of-the-art rebalancing procedures, including deep generative and diffusion models. First, for most data sets, applying no rebalancing strategy is competitive in terms of predictive performances, would it be with LightGBM, tuned random forests or logistic regression. Second, when the imbalance ratio is artificially augmented, one of our two modifications of SMOTE leads to promising predictive performances compared to SMOTE and other strategies. 
\end{abstract}

\section{INTRODUCTION}
Imbalanced data sets for binary classification are encountered in various fields such as fraud detection \citep[][]{ex-fraud}, medical diagnosis  \citep[][]{RF-medical-example} and churn detection \citep[][]{nguyen2021comparison}. In our study, we focus on imbalanced data in the context of binary classification on \textit{tabular data} (thus excluding images and text data sets), for which most machine learning algorithms have a tendency to predict the majority class. This leads to biased predictions, so that several rebalancing strategies have been developed in order to handle this issue, as explained by \citet{krawczyk2016learning} and \citet{ramyachitra2014imbalanced}. These procedures can be divided into model-level and data-level categories.

Model-level approaches modify existing classifiers in order to prevent predicting only the majority class.  
Among such techniques, Class-Weight (CW) works by assigning higher weights to minority samples. Another related procedure proposed by \citet{c-weight-RF-bon} assigns data-driven weights to each tree of a random forest, in order to improve aggregated metrics such as F1 score or ROC AUC. Another model-level technique is to modify the loss function of the classifier. For instance, \citet{cao2019learning}, \citet{lin2017focal} and \citet{ren2018learning}  respectively introduce LDAM, Focal and L2RW losses, in order to produce neural network classifiers that better handle imbalanced data sets. However, model-level approaches are not model agnostic, and thus cannot be applied to a wide variety of machine learning algorithms. Consequently, we focus in this paper on data-level approaches. 

Data-level approaches can be divided into two groups: synthetic and non-synthetic procedures.  
Non-synthetic procedures works by removing or copying original data points. 
\citet{mani2003knn} explain that Random Under Sampling (RUS) is one of the most used resampling strategy and design new adaptive versions called Nearmiss. RUS produces the prespecified balance between classes by dropping uniformly at random majority class samples. The Nearmiss1 strategy \citep{mani2003knn} includes a distinction between majority samples by ranking them with their mean distance to their nearest neighbor from the minority class. Then, low-ranked majority samples are dropped  until a given balancing ratio is reached. 
In contrast, Random Over Sampling (ROS) duplicates original minority samples. 
The main limitation of all these sampling strategies is the fact that they either remove information from the data or do not add new information. 

On the contrary, synthetic procedures generate new synthetic samples in the minority class. For instance, Random Over Sampling Examples \citep[see][]{menardi2014training} is a variant of ROS that produces duplicated samples and then add a noise in order to get these samples slightly different from the original ones. This leads to the generation of new samples on the neighborhood of original minority samples. One of the most famous synthetic strategies is \textit{Synthetic Minority Oversampling Technique }\citep[SMOTE, see][]{chawla2002smote}\footnote{More than $25.000$ papers found in GoogleScholar with a title including ``SMOTE'' over the last decade.}. In SMOTE, new minority samples are generated  via linear interpolation between an original minority sample and one of its nearest neighbor in the minority class. Other approaches are based  on Generative Adversarial Networks  \citep[GAN][]{xu2019modeling,gan-generate} or diffusion \citep[][]{jolicoeur2024generating}, which are computationally expensive and mostly designed for specific data structures, such as images. Furthermore, some recent works (2023) have shown that SMOTE remains competitive even compared with recent diffusion models \citep[see Table 4 and Table 5 in][]{kotelnikov2023tabddpm}. 

\textbf{Contributions} 
We place ourselves in the setting of imbalanced classification on tabular data, which is very common in real-world applications \citep[see ][for details]{shwartz2022tabular}. In this paper:
\begin{itemize}
    \item We prove that, without tuning the hyperparameter $K$ (usually set to $5$), SMOTE asymptotically copies the original minority samples, therefore lacking the intrinsic variability required in any synthetic generative procedure. We provide  numerical illustrations of this limitation (\Cref{section:study_smote}). We also establish that SMOTE density vanishes near the boundary of the support of the minority distribution. 
    \item Our theoretical analysis naturally leads us to introduce two SMOTE alternatives, SMOTE $K$-tuned and Multivariate Gaussian SMOTE (MGS). In \Cref{sec:strategies}, we evaluate our new strategies and state-of-the-art rebalancing strategies on several real-world data sets using random forests, logistic regression and LightGBM. We find 
    that, for most data sets, applying no strategy is competitive in terms of predictive performances. When the imbalance ratio is dramatically increased, our proposed SMOTE modification MGS is among the best strategies (\Cref{sec:experiments2}).
\end{itemize}

\section{RELATED WORKS}

In this section, we focus on the literature that is the most relevant to our work: long-tail learning, SMOTE variants and studies of rebalancing strategies.

Long-tailed learning \citep[see, e.g.,][]{zhang2023deep} is a relatively new field, originally designed to handle image classification with numerous output classes such as for ImageNet-LT \citep[][]{liu2019large} or iNaturalist \citep[][]{van2018inaturalist} data sets.  
Most techniques in long-tailed learning are based on neural networks or use the large number of classes to build or adapt aggregated predictors.
However, in most tabular classification data sets, the number of classes to predict is relatively small, usually equal to two \citep[][]{chawla2004special,he2009learning,grinsztajn2022tree}.
Therefore, long-tailed learning methods are not intended for our setting as $(i)$ we only have two output classes and $(ii)$ state-of-the-art models for tabular data are not neural networks but tree-based methods, such as random forests or gradient boosting \citep[see][]{grinsztajn2022tree,shwartz2022tabular}. However, we decide to include in our study several model-level rebalancing strategies from this framework, such as LDAM, Focal and L2RW losses. Finally, we highlight that SMOTE-related techniques are widely used for image data sets. For example, Mix-up methodology \citep[][]{zhang2017mixup} performs a linear interpolation between two images and is commonly used by practitioners.

SMOTE has seen many variants proposed in the literature.
Several of them focus on generating  synthetic samples near the boundary of the minority class support, such as \text{ADASYN} \citep[][]{he2008adasyn}, \text{SVM-SMOTE} \citep[][]{nguyen2011borderline} or  
Borderline SMOTE  \citep[][]{han2005borderline}.
Many other variants exist such as SMOTEBoost \citep[][]{chawla2003smoteboost}, Adaptive-SMOTE \citep[][]{pan2020learning}, GDO \citep[][]{xie2020gaussian} 
or \text{DBSMOTE}  \citep{bunkhumpornpat2012dbsmote}.
From a computational perspective, several synthetic methods (SMOTE, ADASYN, Borderline SMOTE and SVM-SMOTE) 
are available in the open-source package \textit{imb-learn} \citep[see][]{JMLR:v18:16-365}.
Several papers study experimentally some specificities of the sampling strategies. For example, \citet{kamalov2022partial} study the optimal sampling ratio for imbalanced data sets when using synthetic approaches. \citet{aguiar2023survey} give an overview of imbalance data sets in the context of online learning and propose a standardized framework in order to compare rebalancing strategies in this context. 
Furthermore, \citet{wongvorachan2023comparison} aim at analysing the data-level approaches (ROS, RUS and SMOTE) on educational data in order to predict probability of pursuing higher education or the risk of leaving high school before graduation.  


Rebalancing strategies are studied in several works. \citet{xu2020class} consider the weighted risk of plug-in classifiers, for arbitrary weights. They establish rates of convergence and derive a new robust risk that may in turn improve classification performance in imbalanced scenarios. 
Based on this work, \citet{aghbalou2023sharp} derive a sharp error bound on the balanced risk for binary classification with severe class imbalance. Using extreme value theory, \citet{arjovsky2022throwing} show that applying Random Under Sampling in binary classification improve the worst-group error when learning from imbalanced classes with tails. \citet{wallace2014improving} study the class probability estimates for several rebalancing strategies before introducing a generic methodology in order to improve all these estimates. \citet{dal2015calibrating} focus on the effect of RUS on the posterior probability of the selected classifier. They show that RUS affect the accuracy and the probability calibration of the model.
To the best of our knowledge, there are only few theoretical works dissecting the intrinsic machinery in SMOTE algorithm, with the notable exception of \citet{elreedy2019comprehensive} and \citet{elreedy_theoretical_2023} who established the density of synthetic observations generated by SMOTE, the associated expectation and covariance matrix. \citet{elreedy2019comprehensive} also highlights the effects of the number of minority samples and the input dimension on SMOTE procedure. Indeed, results on simulated and real-world data sets show that the predictive performance of the classifier applied after SMOTE increases with the number of minority samples. Increasing the number of input variables decreases SMOTE ability of regenerating the minority distribution, in terms of Total Variance Difference and Kullback–Leibler divergence.


\section{A STUDY OF SMOTE}
\label{section:study_smote}

\paragraph{Notations} We denote by $\mathcal{U}([a,b])$ the uniform distribution over $[a,b]$. We denote by $\mathcal{N}(\mu,  \Sigma)$ the multivariate normal distribution of mean $\mu \in \mathds{R}^d$ and covariance matrix $\Sigma \in \mathds{R}^{d\times d}$. For any set $A$, we denote by $Vol(A)$, the Lebesgue measure of $A$. For any $z \in \mathds{R}^d$ and $r>0$, let $B(z,r)$ be the ball for the $L_2$ norm, centered at $z$ of radius $r$. We note $c_d = Vol(B(0,1))$ the volume of the unit ball in $\mathds{R}^d$. For any $p,q \in \mathds{N}$, and any $z \in [0,1]$, we denote by \(\mathcal{B}(p,q;z)=\int_{t=0}^{z}t^{p-1}(1-t)^{q-1}\mathrm{d}t\) the incomplete beta function.

\subsection{SMOTE algorithm}
We assume to be given a training sample $\mathcal{D}_N$ composed of $N$ pairs $(X_i,Y_i)$, independent and identically distributed as $(X,Y)$, where $X$ and $Y$ are random variables that take values respectively in $\mathcal{X} \subset \mathds{R}^d$ and $\{0,1\}$. We consider a class imbalance problem, in which the class $Y=1$ is under-represented, compared to the class $Y=0$, and thus called the minority class.  We assume that we have $n$ minority  samples in our training set. We define the imbalance ratio as $n/N$.
In this paper, we consider continuous input variables only, as SMOTE was originally designed for.%
\begin{algorithm}[t]
\caption{Single SMOTE Sample Generation}
\label{alg:smote}
\begin{algorithmic}[1]
\REQUIRE Minority class samples $\{X_1, \dots, X_n\} \subset \mathds{R}^d$, number of neighbors $K$
\ENSURE Synthetic sample $Z$

\STATE Sample a central point $X_c$ uniformly from $\{X_1, \dots, X_n\}$
\STATE Compute the set $I$ of the $K$ nearest neighbors of $X_c$ (Euclidean distance)
\STATE Sample a neighbor $X_k$ uniformly from $I$
\STATE Sample $w \sim \mathcal{U}(0,1)$
    \STATE  Generate $Z \leftarrow  X_c+w (X_k - X_c)$
\RETURN $Z$
\end{algorithmic}
\end{algorithm}

SMOTE procedure generates synthetic data through linear interpolations between two pairs of original samples of the minority class. SMOTE algorithm has a single hyperparameter, $K$, by default set to 5, which stands for the number of nearest neighbors considered when interpolating. We denote by $Z_{K,n}$ the distribution of samples generated by SMOTE. A single SMOTE iteration is detailed in \Cref{alg:smote}. In a machine learning pipeline, SMOTE procedure is repeated in order to obtain a prespecified ratio between the two classes, before training a classifier.

A classical statistical framework  is to let the total number of samples grow to infinity, i.e. $N \to \infty$. In this setting, we assume that the imbalance ratio is asymptotically constant, that is 
\begin{align}
\lim\limits_{N \to \infty} \frac{n}{N} = \mathds{P}(Y=1).   \label{setting_ratio_constant}
\end{align}
In particular, the number of minority samples tends to infinity. Thus, throughout the following section, we study the theoretical behavior of SMOTE, by letting $n \to \infty$ or by providing finite-sample upper bounds.

\subsection{Theoretical results on SMOTE}

SMOTE has been shown to exhibit good performances when combined to standard classification algorithms \citep[see, e.g.,][]{improve-smote-zada}. However, there exist only few works that aim at understanding theoretically SMOTE behavior. In this section, we assume that  $X_1, \hdots, X_n$ are i.i.d samples from the minority class (that is, $Y_i = 1$ for all $i \in [n]$), with a common density $f_X$ with bounded support, denoted by $\mathcal{X}$.
\begin{lemma}[Convexity]
\label{th:support_conv}
Given $f_X$ the distribution density of the minority class, with support $\mathcal{X}$, for all $K,n$, the associated SMOTE density $f_{Z_{K,n}}$ satisfies
\begin{align}
    Supp(f_{Z_{K,n}}) \subseteq Conv(\mathcal{X}). \label{eq_lem1}
\end{align}
\end{lemma}
By construction, synthetic observations generated by SMOTE cannot fall outside the convex hull of $\mathcal{X}$. Equation \eqref{eq_lem1} is not an equality, as SMOTE samples are the convex combination of only two original samples. For example, in dimension two, if $\mathcal{X}$ is concentrated near the vertices of a triangle, then SMOTE samples are distributed near the triangle edges, whereas $\textrm{Conv}(\mathcal{X})$ is the surface delimited by the triangle.  

SMOTE algorithm has only one hyperparameter $K$, which is the number of nearest neighbors taken into account for building the linear interpolation. By default, this parameter is set to $5$. The following theorem describes the behavior of SMOTE distribution asymptotically, as $K/n \to 0$. 
\begin{theorem}
\label{th:regenerate}
For all Borel sets $B \subset \mathds{R}^d$, if $K/n \to 0$, as $n$ tends to infinity, we have
\begin{align}
  \lim_{n \to \infty}  \mathds{P}[Z_{K,n} \in B] = \mathds{P}[X \in B].
\end{align}
\end{theorem}
The proof of \Cref{th:regenerate} can be found in \Cref{proof_th:regenerate}.
\Cref{th:regenerate} proves that the random variables $Z_{K,n}$ generated by SMOTE converges in distribution to the original random variable $X$, provided that $K/n$ tends to zero. From a practical point of view,  \Cref{th:regenerate} guarantees asymptotically the ability of SMOTE to regenerate the distribution of the minority class. 
This highlights a good behavior of the default setting of SMOTE ($K=5$), as it can create more data points, different from the original samples while preserving the original data distribution. 
Note that \Cref{th:regenerate} is very generic, as it makes no assumptions on the distribution of $X$. 

SMOTE distribution has been derived in Theorem 1 and Lemma 1 in \citet{elreedy_theoretical_2023} using geometrical arguments.  
We recall the expression of this density in \Cref{th-densite} in Appendix, with a proof based on random variables (see \Cref{proof_th-densite} ). When no confusion is possible, we simply write $f_Z$ instead of $f_{Z_{K,n}}$. 
A close inspection of the SMOTE density allows us to derive more precise bounds about the behavior of SMOTE, as established in \Cref{thm_asymptot_SMOTE_k}.

\begin{assumption}
\label{ass:bounded_density}
There exists $R>0$ such that $\mathcal{X} \subset B(0,R)$. Besides, there exist $0 < C_2 < \infty$ such that for all $x \in \mathds{R}^d$,  $ f_X(x) \leq C_2 \mathds{1}_{x \in \mathcal{X}}$.
\end{assumption}
\begin{theorem}
\label{thm_asymptot_SMOTE_k}
Grant Assumption~\ref{ass:bounded_density}. Let $x_c \in \mathcal{X}$ and $\alpha \in (0,2R)$. For all $K \leq (n-1) \mu_X \left( B\left(x_c, \alpha \right)\right)$, 
we have
\begin{align}
& \mathds{P}(\|Z_{K,n}-X_c\|_2 \geq \alpha|X_c=x_c) 
  \leq \eta_{\alpha,R,d} \\ 
& \nonumber  \times \exp \left(-2 (n-1) \left( \mu_X \left( B\left(x_c, \alpha \right) \right)  - \frac{K}{n-1}\right)^2\right),
\end{align} 
\begin{align}
\nonumber \eta_{\alpha,R,d} = C_2 c_d R^d \times \left\{
        \begin{array}{ll}
              \ln\left(\frac{2R}{\alpha}\right) & \text{if $d=1$,} \\
             \frac{1}{d-1} \left( \left( \frac{2R}{\alpha} \right)^{d-1} - 1 \right)   & \text{if $d >1$.}              
        \end{array}
    \right. 
\end{align}

Consequently, if $\lim_{n \to \infty} K/n = 0$, we have, for all $x_c \in \mathcal{X}$,   $Z_{K,n}|X_c = x_c \to x_c$ in probability.

\end{theorem}

The proof of \Cref{thm_asymptot_SMOTE_k} can be found in \Cref{proof_thm_asymptot_SMOTE_k}. 
\Cref{thm_asymptot_SMOTE_k} establishes an upper bound on the distance between an observation generated by SMOTE and its central point. Asymptotically, when $K/n$ tends to zero, the new synthetic observation concentrates around the central point. Recall that, by default, $K=5$ in SMOTE algorithm. Therefore, \Cref{th:regenerate} and \Cref{thm_asymptot_SMOTE_k} prove that, with the default settings, SMOTE asymptotically targets the original density of the minority class and generates new observations very close to the original ones. 
Besides, the assumption for all $x \in \mathds{R}^d$,  $  f_X(x) \leq C_2 \mathds{1}_{x \in \mathcal{X}}$ can be removed at the cost of replacing the factor $\left(2R/\alpha\right)^{d-1}$ by $\left(2R/\alpha\right)^{d}$ when $d>1$, or $\ln\left(\frac{2R}{\alpha}\right)$ by $\left(2R/\alpha\right)^{-1}\ln\left(\frac{2R}{\alpha}\right)$ when $d=1$ (see eq \eqref{eq-th34-noC2}). These resulting factors are less sharp than those obtained under the boundedness assumption on $C_2$, which enables a more precise control of the concentration of SMOTE observations around the original data points. The following result establishes the characteristic distance between SMOTE observations and their central points.
\begin{corollary}
\label{proposition_distance_caracteristique}
Grant Assumption~\ref{ass:bounded_density}. 
For all $d \geq 2$, for all $\gamma \in (0,1/d)$, we have
\begin{align}
 \mathds{P} \left[ \| Z_{K,n} - X_c \|_2 > 12R (K/n)^{\gamma} \right] & \leq \left( \frac{K}{n} \right)^{2/d - 2 \gamma}.
\end{align}
\end{corollary} 
The proof of \Cref{proposition_distance_caracteristique} can be found in \Cref{proof_proposition_distance_caracteristique}. 
The characteristic distance between a SMOTE observation and the associated central point is of order $(K/n)^{1/d}$. As expected from the curse of dimensionality, this distance increases with the dimension $d$. Choosing $K$ that increases with $n$ leads to larger characteristic distances: SMOTE observations are more distant from their central points. 
\Cref{proposition_distance_caracteristique} leads us to choose $K$ such that $K/n$ does not tend too fast to zero, so that SMOTE observations are not too close to the original minority samples. However, choosing such a $K$ can be problematic, especially near the boundary of the support, as shown in the following theorem.

\begin{theorem}
\label{prop:boundaries}
There exists $R>0$ such that $\mathcal{X} = B(0,R)$. Besides, there exist $0 < C_1 < C_2 < \infty$ such that for all $x \in \mathds{R}^d$,  $ C_1 \mathds{1}_{x \in \mathcal{X}} \leq f_X(x) \leq C_2 \mathds{1}_{x \in \mathcal{X}}$. 
Let $\varepsilon \in (0,R)$ such that $\left( \frac{\varepsilon}{R} \right)^{1/2} \leq \frac{c_d}{\sqrt{2} d C_2}.$
Then, for all $1 \leq K < n$, and all $z \in B(0,R) \backslash B(0,R-\varepsilon)$, and for all $d>1$, we have
\begin{align}
    f_{Z_{K,n}}(z) 
    & \leq C_2^{3/2} \left( \frac{2^{d+2} c_d^{1/2}}{d^{1/2}} \right) \left( \frac{n-1}{K} \right)   \left( \frac{\varepsilon}{R} \right)^{1/4}.
\end{align}
\end{theorem}
The proof of \Cref{prop:boundaries} can be found in \Cref{proof_prop:boundaries}. \Cref{prop:boundaries} establishes an upper bound of SMOTE  density at points distant from less than $\varepsilon$ from the boundary of the minority class support. More precisely, 
\Cref{prop:boundaries} shows that SMOTE density vanishes as $\varepsilon^{1/4}$ near the boundary of the support. 
Choosing $\varepsilon/R = o ((K/n)^4)$ leads to a vanishing upper bound, which proves that SMOTE density is unable to reproduce the original density $f_X \geq C_1$ in the peripheral area $B(0,R) \backslash B(0,R-\varepsilon)$. Such a behavior was expected since the boundary bias of local averaging methods (kernels, nearest neighbors) has been extensively studied \citep[see, e.g.][]{jones1993simple,arya1995accounting, arlot2014analysis, mourtada2020minimax}.

For default settings of SMOTE (i.e., $K=5$), and large sample size, this area is relatively small ($\varepsilon = o (n^{-4})$). Still, \Cref{prop:boundaries} provides a theoretical ground for understanding the behavior of SMOTE near the boundary, a phenomenon that has led to introduce variants of SMOTE to circumvent this issue \citep[see Borderline SMOTE in][]{han2005borderline}. While increasing $K$ leads to more diverse generated observations (\Cref{thm_asymptot_SMOTE_k}), it increases the boundary bias of SMOTE. Indeed, choosing $K = n^{3/4}$ implies a boundary effect in the peripheral area $B(0,R) \backslash B(0,R-\varepsilon)$ for $\varepsilon = o(1/n)$, which may not be negligible. Finally, note that constants in the upper bounds are of reasonable size. Letting  $d=3$, $K=5$, $X \sim \mathcal{U}(B_d(0,1))$, the upper bound turns into $0.89 n \varepsilon^{1/4}$.

\subsection{Numerical illustrations}
\label{sec:numerical_illustration}
We illustrate \Cref{thm_asymptot_SMOTE_k} and \Cref{proposition_distance_caracteristique} with simulated data.
In \Cref{fig:dist-nn-normalized_U2D}, we define a diversity measure $d(\mathbf{Z}, \mathbf{X})$ for the synthetic samples ($y$-axis, detailed below) and plot it as a function of the number of minority samples $n$ for different values of $K$.
\begin{figure}
    \centering
    \includegraphics[width=0.9\linewidth]{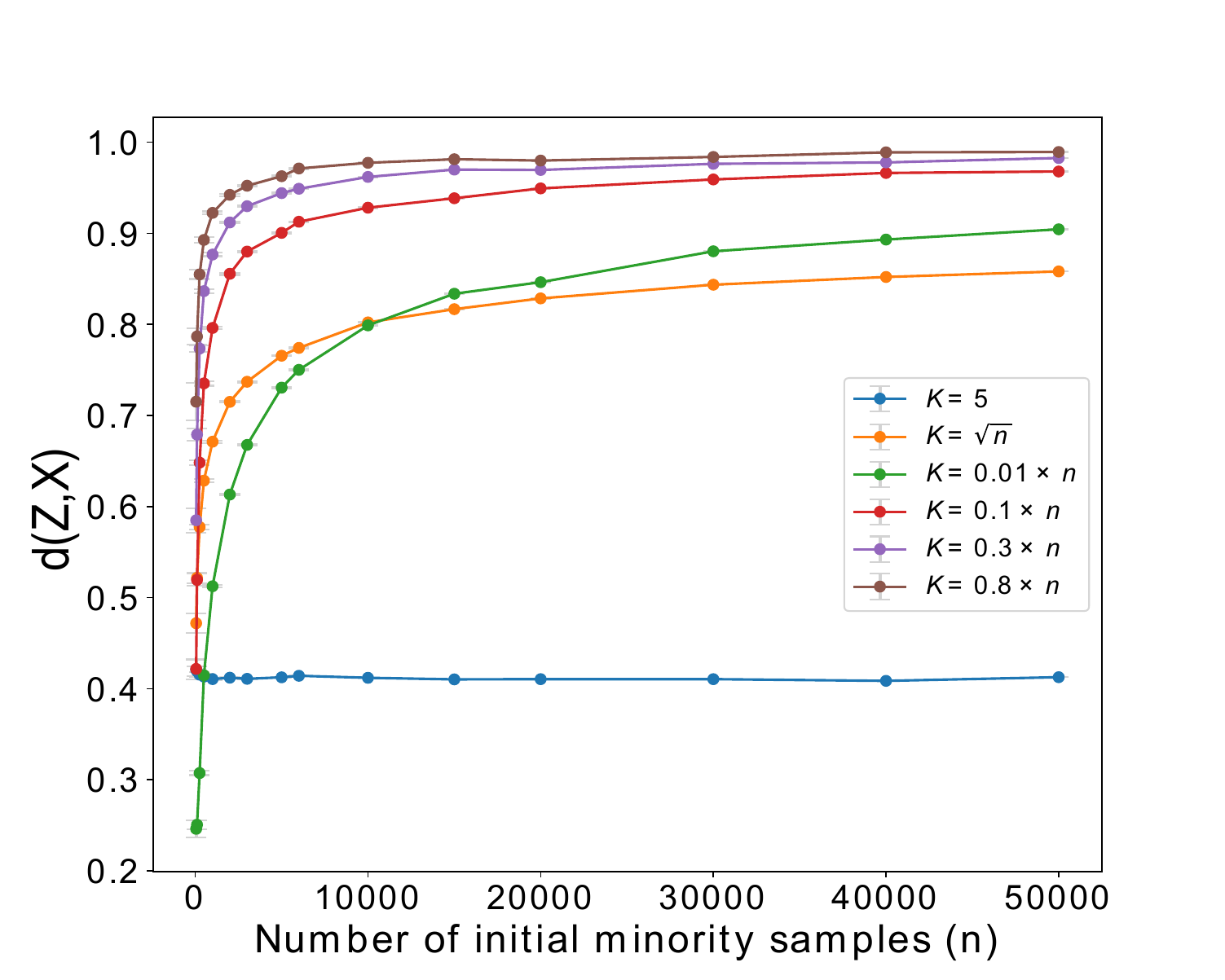}
    \caption{Diversity measure with $\mathcal{U}([-3,3]^2)$}
    \label{fig:dist-nn-normalized_U2D}
\end{figure}

\textbf{Simulated data} \label{subsec:simulated-data}
In order to measure the similarity between any generated data set $\mathbf{Z} = \{Z_1, \hdots, Z_m\}$ and the original data set $\mathbf{X} = \{X_1, \hdots, X_n\}$, we compute 
$
C(\mathbf{Z}, \mathbf{X}) = \frac{1}{m} \sum_{i=1}^m \|Z_i - X_{(1)}(Z_i)\|_2,
$
where $X_{(1)}(Z_i)$ is the nearest neighbor of $Z_i$ among $X_1, \hdots, X_n$. Intuitively, this quantity measures how far the generated data set is from the original observations: if the new data are copies of the original ones, this measure equals zero. We apply the following protocol: for each value of $n$, 
\begin{enumerate}
    \item Generate $\mathbf{X}$ composed of $n$ i.i.d samples distributed as $\mathcal{U}([-3,3]^2)$. 
    \item Generate $\mathbf{Z}$ composed of $m=1000$ new i.i.d observations by applying SMOTE procedure on the original data set $\mathbf{X}$, with different values of $K$. Compute $C(\mathbf{Z}, \mathbf{X})$.
    \item Generate $\tilde{\mathbf{X}}$ composed of $m$ i.i.d new samples distributed as $\mathcal{U}([-3,3]^2)$. Compute $C(\tilde{\mathbf{X}}, \mathbf{X})$, which is a reference value in the ideal case of new points sampled from the same distribution.
\end{enumerate}
Steps $1$-$3$ are repeated  $75$ times. The average of $C(\mathbf{Z}, \mathbf{X})$ (resp. $C(\tilde{\mathbf{X}}, \mathbf{X})$) over these repetitions is computed and denoted by $\bar{C}(\mathbf{Z}, \mathbf{X})$ (resp. $\bar{C}(\tilde{\mathbf{X}}, \mathbf{X})$). 
We consider the diversity measure $d(\mathbf{Z}, \mathbf{X}) = \bar{C}(\mathbf{Z}, \mathbf{X})/\bar{C}(\tilde{\mathbf{X}}, \mathbf{X})$, depicted in \Cref{fig:dist-nn-normalized_U2D} (see also \Cref{fig:dist-nn_U2D} in \Cref{appsec:num_illustrations_theorems} for  $\bar{C}(\mathbf{Z}, \mathbf{X})$).

The diversity measure  close to one corresponds to synthetic SMOTE samples that are independent from the initial samples, and generated as the original samples. On the contrary, a diversity measure  close to zero corresponds to synthetic SMOTE samples that are close from their central points, regenerating the original distribution with less diversity than the true underlying distribution (uniform distribution). 

We note that the diversity measure increases with $n$ for all settings (different choices of $K$), except for the constant choice $K=5$. Thus,
default SMOTE ($K=5$) generates samples that are less diverse, than those generated with other choices for $K$.
We extended our simulations to  real-world data sets and obtained the same conclusions (see \Cref{fig:protocol-3-phoneme} in Appendix~\ref{appsec:num_illustrations_theorems}).


\section{PREDICTIVE EVALUATION ON INITIAL REAL-WORLD DATA SETS}
\label{sec:strategies}
In this section,  we propose two minor modifications of SMOTE, in order to circumvent the theoretical limitations of SMOTE identified in the previous section. Then we analyse and compare their associated predictive performances to the state-of-the-art rebalancing strategies. 
We start by presenting existing strategies to deal with imbalanced data. 



\subsection{Existing strategies}
\textbf{Class-weight (CW)} Class weighting assigns the same weight (chosen as hyperparameter) to each minority samples. The default setting for this strategy is to choose a weight $\rho$ such that $\rho n = N-n$, where $n$ and $N-n$ are respectively  the number of minority and majority samples in the data set.

\textbf{Random Over/Under Sampling} Random Under Sampling (RUS) acts on the majority class by selecting uniformly without replacement several samples in order to obtain a prespecified size for the majority class.   Similarly, Random Over Sampling (ROS) acts on the minority class by selecting uniformly with replacement several samples to be copied. 


\textbf{NearMissOne}
NearMissOne is an undersampling procedure. For each sample $X_i$ in the majority class, the averaged distance of $X_i$ to its $K$ nearest neighbors in the minority class is computed. Then, the samples $X_i$ are ordered according to this averaged distance. Finally, iteratively, the first $X_i$ is dropped (smallest mean first) until the given ratio is reached. 

\textbf{Borderline SMOTE 1 and 2 } Borderline SMOTE 1 \citep[][]{han2005borderline} procedure works as follows. For each individual $X_i$ in the minority class, let $m_-(X_i)$ be the number of samples of the majority class among the $m$ nearest neighbors of $X_i$, where $m$ is a hyperparameter. For all $X_i$ in the minority class such that $m/2 \leq m_-(X_i) < m$, generate successive samples $Z = WX_i + (1-W) X_k$ where $W \sim \mathcal{U}([0,1])$ and $X_k$ is selected among the $K$ nearest-neighbors of $X_i$ in the minority class.
In Borderline SMOTE 2 \citep[][]{han2005borderline}, the selected neighbor $X_k$ is chosen from the neighbors of both positive and negative classes, and $Z$ is sampled with $W \sim \mathcal{U}([0,0.5])$.
%

\paragraph{CTGAN and ForestDiff} CTGAN (Conditional Tabular GAN) \citep[][]{xu2019modeling} is a conditional generative adversarial network that is specially designed for generating new synthetic samples for tabular data sets. CTGAN uses conditional generation to sample specific categories during training and applies a mode-specific normalization. ForestDiffusion \citep[][]{jolicoeur2024generating} is
a score-based diffusion and a conditional flow matching models which substitute the score function learner,  traditionally a neural networks, by gradient-boosted decision trees classifier (XGBoost). This substitution capitalizes on XGBoost proficiency in handling tabular data.

\subsection{Variants of SMOTE}
Now, we introduce two slight modifications of SMOTE in order to alleviate the theoretical limitations (tendency to copy the minority samples for fixed $K$ and artifact boundaries) demonstrated in Section~\ref{section:study_smote}.

\textbf{$K$-tuned SMOTE} The analysis of \Cref{thm_asymptot_SMOTE_k} and \Cref{proposition_distance_caracteristique} suggests that the parameter $K$ of SMOTE should not be fixed to a default value. Through several numerical experiments, we did not identify a clear way to set $K$ as a function of $n$ that systemically induces better predictive performances. 
SMOTE $K$-tuned finds the best hyperparameter $K$ among a prespecified grid via a $5$-fold cross-validation procedure. 
The grid is composed of the set $\{1,2,\hdots, 15\}$ extended with  the values $\lfloor0.01n_{train}\rfloor,  \lfloor0.1n_{train}\rfloor, \lfloor0.5n_{train}\rfloor,  \lfloor0.7n_{train}\rfloor$ and $ \lfloor\sqrt{n_{train}} \rfloor$, where $n_{train}$ is the number of minority samples in the training set. 
SMOTE $K$-tuned is one of the simplest ideas to solve some SMOTE limitations  highlighted theoretically in Section~\ref{section:study_smote}. To the best of our knowledge, there is no previous work that recommends  tuning the hyperparameter $K$ of SMOTE, neither recommend values of $K$ that depends on $n$.

\textbf{Multivariate Gaussian SMOTE (MGS)}
\Cref{th:support_conv} leads us to modify SMOTE algorithm such that the generated samples are no longer contained in the convex hull of the original minority samples. Furthermore, we expect from \Cref{thm_asymptot_SMOTE_k} and \Cref{prop:boundaries} that adopting a different distribution function to generate new samples, given the central point and its neighbors, would limit the copying effect and the boundary artifact of SMOTE. 
Thus, we slightly modify SMOTE to introduce a new oversampling strategy in which new samples are generated from several \textit{multivariate} Gaussian distribution  $\mathcal{N}(\hat{\mu}, \hat{\Sigma})$, where the empirical mean $\hat{\mu}$ and covariance matrix $\hat{\Sigma}$ are estimated using the $K$ neighbors and the central point (see Algorithm \ref{alg:SGD} for details). Note that MGS generated samples follow a mixture of Gaussian distributions. By default, we choose $K=d+1$, so that estimated covariance matrices can be of full rank. 
MGS produces more diverse synthetic observations than SMOTE as they are spread in all directions around the central point (when the covariance matrix is full rank). %
\begin{algorithm}[t]
    \caption{Single MGS Sample Generation.}
    \label{alg:SGD}
    \begin{algorithmic}[1]
       \REQUIRE Minority class samples $\{X_1, \dots, X_n\} \subset \mathds{R}^d$, number of neighbors $K$
       \ENSURE Synthetic sample $Z$
    
       \STATE Select uniformly $X_c$ among $X_1, \hdots, X_n$.
       \STATE Compute the set $I$ of the $K$ nearest neighbors of $X_c$ (Euclidean distance)
       \STATE $\hat{\mu} \gets \frac{1}{K+1} \sum\limits_{x \in I } x $
       \STATE $\hat{\Sigma} \gets \frac{1}{K+1} \sum\limits_{x \in I} \left(x - \hat{\mu}\right)^T\left(x - \hat{\mu}\right) $
       \STATE  Generate synthetic sample $Z \sim  \mathcal{N}\left(\hat{\mu},  \hat{\Sigma}\right)$
       \RETURN $Z$
    \end{algorithmic}
\end{algorithm}

%

\subsection{Protocol and results on initial data sets}

\label{sec:protocol}

\textbf{Protocol}
We compare the different rebalancing strategies on the $13$ initial data sets (described below). We employ a 5-fold stratified cross-validation, and apply each rebalancing strategy on four training folds, in order to obtain the same number of minority/majority samples. Then, we train 
a Random Forest classifier \citep[showing good predictive performance on tabular data sets, see][]{grinsztajn2022tree} on the same folds, and evaluate its performance on the remaining fold, via the PR AUC. Results are averaged over the five test folds and over $20$ repetitions of the cross-validation.
We use the \texttt{RandomForestClassifier} module in \textit{scikit-learn} \citep[][]{pedregosa2011scikit} and tune the tree depth (when desired) via nested cross-validation \cite{cawley2010over}.  We use the implementation of \textit{imb-learn} \citep[][]{JMLR:v18:16-365} for state-of-the-art rebalancing strategies (see Appendix~\ref{sec:supp-results} for details). The procedures implementation, and code for experiments reproducibility are available at the following GitHub repository: \url{https://github.com/artefactory/mgs-grf}.

\paragraph{Metrics} \cite{saito2015precision} show that the ROC AUC might be biased for evaluating classifiers on imbalanced data sets, so that we choose to use the PR AUC.
While the balanced accuracy \citep[][]{brodersen2010balanced} is a widespread alternative to the usual accuracy in the imbalance context, we prefer the PR AUC which is independent of the threshold on classifier output probabilities.
In addition, using the balanced accuracy is equivalent to applying class weights to the test set, as suggested by \Cref{lemma:balanced_accuracy} below, which is not recommended in machine learning as the test set must reflect the true data distribution. Thus, we do not advocate for using the balanced accuracy as a final metric.

\begin{lemma}
\label{lemma:balanced_accuracy}
    Let $\hat{f}: \mathcal{X} \rightarrow \{0,1\}$ be a trained classifier. Let $\mathcal{D}_{test} = \{(X_i,Y_i)\}_{i=1}^{N}$ be a test set for $\hat{f}$, with $n \in \mathds{N}$ and $\sum_{i=1}^{N} Y_i = n$. 
    The balanced accuracy of $\hat{f}$, denoted by $\text{Acc}_{Balanced}(\hat{f})$, is defined as the average of the recall and the specificity of $\hat{f}$ on $\mathcal{D}_{test}$. Let us consider the weighted accuracy of $\hat{f}$, denoted by $Acc_{Weighted}(\hat{f})$, which is nothing but the standard accuracy where samples in $\mathcal{D}_{test}$ are weighted by $\frac{N}{2(N-n)}$ if they belong to the class $Y=0$ and by $\frac{N}{2n}$ if they belong to the class $Y=1$. 
     Then, we have:
    \begin{align}
        \text{Acc}_{Balanced}(\hat{f}) = Acc_{Weighted}(\hat{f}).
    \end{align}
\end{lemma}
The proof of \Cref{lemma:balanced_accuracy} is available in \Cref{proof_lemma:balanced_accuracy}.

\textbf{Initial data sets} We employ $13$  tabular data sets from \cite{grinsztajn2022tree}, UCI Irvine \citep[][]{Dua:2019} and other public data sources (Phoneme \citep{alinat1993periodic} and CreditCard \citep{dal2015calibrating}). All data sets are described in \Cref{table:data-sets} in \Cref{sec:supp-results} and we call them \textit{initial data sets}. As we want to compare several rebalancing methods including SMOTE, originally designed to handle continuous variables only, we have included data sets with continuous input only.

\textbf{First results: None is competitive for slightly imbalanced data sets} 
For $11$ initial data sets out of $13$, applying no strategy (named None in all tables) is the best, probably  highlighting that the imbalance ratio is not low enough or the learning task not difficult enough to require a tailored rebalancing strategy. 
 Therefore, considering only continuous input variables, and measuring the predictive performance with PR AUC, we observe that dedicated rebalancing strategies are not required for most data sets.  
While applying no rebalancing strategy was already perceived as competitive in the literature \citep[see, e.g.,][]{han2005borderline, he2008adasyn,elor2022smote,carriero2025harms}, we believe that our analysis advocates for its broad use in practice, at least as a default method. Note that for these $11$ data sets, qualified as slightly imbalanced, applying no rebalancing strategy is on par with the CW strategy, one of the most common rebalancing strategies (regardless of tree depth tuning, see \Cref{tab:corps_resultats_rf_tuned} and \Cref{tab:annexe_resultats_rf_std}).



\section{EXTREMELY IMBALANCED DATA SETS }
\label{sec:experiments2}

\textbf{Strengthening the imbalance} 
To analyse what could happen for data sets with lower imbalance ratio, we subsample the minority class for each one of the initial data sets mentioned above, so that the resulting imbalance ratio is set to $20\%$, $10\%$ or $1\%$ (when possible, taking into account dimension $d$).
By doing so, we reproduce the extreme imbalance that is often encountered in practice \citep[see][]{he2009learning}.
We apply our subsampling strategy once for each data set and each imbalance ratio in a nested fashion, so that the minority samples of the $1\%$ data set are included in the minority samples of the $10\%$ data set. The $18$ new data sets thus obtained are called \textit{subsampled data sets} and presented in \Cref{table:data-sets-annexe} in \Cref{appendix:tables}.

\subsection{Results}

For the sake of brevity, we display in \Cref{table:data-sets-extremely} the initial and subsampled (with the new final imbalance ratio) data sets for which the None strategy is not the best, up to its standard deviation. The results in terms of PR AUC are displayed in \Cref{tab:corps_resultats_rf_tuned} (see \Cref{tab:annexe_resultats_rf_tuned_std} in Appendix \ref{appendix:tables} for the other data sets).
Thus, we focus our discussion on the performances of rebalancing methods in \Cref{table:data-sets-extremely}. 
We remark that the included data sets, referred as \textit{extremely imbalanced data sets}, correspond to the most imbalanced subsampling, or simply the initial data set in case of extreme initial imbalance.

Whilst in the vast majority of experiments, applying no rebalancing is among the best approaches to deal with imbalanced data (see \Cref{tab:annexe_resultats_rf_tuned_others}), it seems to be outperformed by dedicated rebalancing strategies for extremely imbalanced data sets (\Cref{tab:corps_resultats_rf_tuned}). Surprisingly, most rebalancing strategies do not benefit drastically from tree depth tuning, with the exception of None and CW strategies  (see the differences between \Cref{tab:corps_resultats_rf_tuned} and  \Cref{tab:annexe_resultats_rf_std}).



\textbf{SMOTE and SMOTE $K$-tuned} 
Default SMOTE (with $K=5$) has a tendency to duplicate original observations, as shown by \Cref{thm_asymptot_SMOTE_k}. 
This behavior is illustrated through our experiments when the tree depth is fixed. In this context, SMOTE ($K=5$) has the same behavior as ROS, a method that copies original samples (see \Cref{tab:annexe_resultats_depth_rus}). When the tree depth is tuned, SMOTE may exhibit better performances compared to reweighting methods (ROS, RUS, CW), probably due to a higher tree depth. Indeed, even if synthetic data are close to the original samples, they are distinct and thus allow for more splits in the tree structure. 
However, SMOTE $K$-tuned performances are not systemically higher than those of default SMOTE (see \Cref{tab:corps_resultats_rf_tuned}). This may come from the fact that SMOTE $K$-tuned does not include a way of limiting the boundary artifacts phenomenon described in \Cref{prop:boundaries}.

\begin{table}
    \centering
\caption{Extremely imbalanced data sets, among the $13$ initial data sets, for which None strategy is not the best (up to its standard deviation). 
$N_{new}$ is the total number of samples after strengthening the imbalance.}
\label{table:data-sets-extremely}
\begin{tabular}{lcrrr} 
     \toprule
      & $d$ & $N$ & $n/N$ & $n/N_{new}$ \\ 
     \midrule
     CreditCard & $29$ & $284\,315$ & $0.2\%$ & $0.2\%$\\
     Abalone & $8$ & $4\,177$ & $1\%$ & $1\%$ \\
     \textit{Phoneme} & $5$ & $5\,404$ & $29\%$ &$1\%$ \\ 
     \textit{Haberman} & $3$ & $306$ & $26\%$ &$10\%$ \\
     \textit{MagicTel} & $10$ & $13\,376$ & $50\%$ & $20\%$ \\
     \textit{California} & $8$ & $20\,634$ & $50\%$ &$1\%$ \\
 \bottomrule
\end{tabular}
\end{table}
\textbf{MGS} Our second new available 
strategy exhibits good predictive performances (best performance in $4$ out of $6$ data sets in \Cref{tab:corps_resultats_rf_tuned} and best average). This could be explained by the multivariate Gaussian sampling of synthetic observations that allows generating data points outside the convex hull of the minority class, thus limiting the border phenomenon, established in \Cref{prop:boundaries}. Furthermore, MGS is associated to better performances for $5$ data sets out of $6$ compare to BS1 and BS2. Thus, the boundary artifact phenomenon of SMOTE is better circumvented by MGS. Note that with MGS, there is no need of tuning the tree depth: predictive performances of default RF are on par with tuned RF.  All in all MGS is a promising new strategy compared to deep generative models and diffusion.

\subsection{Supplementary results}
\begin{table*}[t]
\centering
\caption{Random Forests PR AUC (tree depth tuned on PR AUC) on extremely imbalanced data sets. Only data sets whose PR AUC of at least one rebalancing strategy is larger than that of None strategy plus its standard deviation are displayed. Undersampled data sets are in italics. Std are displayed in \Cref{tab:annexe_resultats_rf_tuned_std}. Rebalancing and training times for \textit{Phoneme} ($1\%$) data set of one protocol iteration are displayed.}
\label{tab:corps_resultats_rf_tuned}
\setlength\tabcolsep{4pt}
\begin{tabular}{lcccccccccccr}
     \toprule 
       Strategy &  \small None    &  \small CW    &  \small RUS &  \small ROS &   \small NM1 &  \small BS1  &  \small \small BS2 &  \scriptsize SMOTE &   \scriptsize SMOTE &   \small MGS & \small CT & \small Forest  \\
       &     &      &  & &  &  &  &  &    \scriptsize $K$-tuned  & \tiny ($d+1$) &\small GAN &\small Diff  \\
     \midrule
     \small CreditCard \tiny($0.2\%$) &0.776	&0.783	&0.751	&\textbf{0.787}	&0.586	&0.786	&0.784	&0.773  &0.771 &0.747 &0.757	&0.752 \\
     \small Abalone \tiny($1\%$) &0.048	&0.055	&0.052	&0.052	&0.024	&0.044	&0.039	&0.046	&0.051	&\textbf{0.056} &0.050	&\textbf{0.056}    \\
     \small \textit{Phoneme} \tiny($1\%$) & 0.198	&0.211	&0.086	&0.200	&0.054	&0.222	&0.211	&0.224	&0.240	&\textbf{0.269} &0.130	&0.249\\
     \small \textit{Haberman} \tiny($10\%$) &0.252	&0.268	&0.289	&0.295	&0.268	&0.308	&0.294	&0.306	&0.307	&\textbf{0.311} &0.280	&0.306 \\
     \small \textit{MagicTel} \tiny($20\%$) &0.754	&0.759	&0.741	&\textbf{0.760}	&0.338	&0.739	&0.742	&0.757	&0.757 &0.750 &0.693	&0.756\\
     \small \textit{{California}} \tiny($1\%$) &0.298	&0.283	&0.207	&0.277	&0.019	&0.234	&0.226	&0.242	&0.258	&\textbf{0.322}  &0.152	&0.294\\
     \midrule
     Average & 0.388 &0.393 &0.354 &0.395 &0.214 &0.389 &0.382 &0.391 &0.397 &\textbf{0.409} &0.344 &0.402 \\
     \midrule
     Time (s) & 6.0 &5.1 &1.7 &5.9 &1.7 &5.5 &5.5 &6.1 &120.2 &7.2 &32.8 &306.3 \\
     \bottomrule
\end{tabular}
\end{table*}
\paragraph{CTGAN and ForestDiffusion}
We note in \Cref{tab:corps_resultats_rf_tuned} that CTGAN performances are significantly lower than those of None strategy on $3$ data sets and ForestDiffusion is among the best strategies for most data sets.  Then, note that both CTGAN and ForestDiffusion computation time are among the longest.
Furthermore, we remark that SMOTE and all the others SMOTE-related performances are not outperformed by GAN or diffusion based strategies in \Cref{tab:corps_resultats_rf_tuned}. This observation is corroborated by recent works showing that SMOTE-related methodologies remain competitive even compared to recent diffusion models \citep[see Table 4 and Table 5 in][]{kotelnikov2023tabddpm}.

\paragraph{LightGBM and Logistic Regression}
When replacing random forests with LightGBM \citep[a second-order boosting algorithm, see][]{ke2017lightgbm} or logistic regression in the above protocol, we still do not observe strong benefits of using a rebalancing strategy for most data sets. Thus we keep only extremely imbalanced data sets for those classifiers in \Cref{tab:lgbm_tuned_std} and \Cref{tab:resultats_annexe_logreg}. As shown in \Cref{tab:lgbm_tuned_std}, LightGBM reduces the performance gap between the different rebalancing methodologies, while MGS remains competitive. Additionally, \Cref{tab:resultats_annexe_logreg} shows that the None strategy can still be competitive, even on some extremely imbalanced data sets, when used in conjunction with logistic regression.
We compared in \Cref{tab:reglog_longtailed} the LDAM, Focal and L2RW losses intended for long-tailed learning, using PyTorch. \Cref{tab:reglog_longtailed} shows that LDAM loss performances are on par with the None strategy ones, while the performances of Focal are significantly higher for $3$ data sets. We also observe that L2RW introduces the best PR AUC improvement for Abalone data set but does not seem to be competitive on the majority of the data sets. Methods from long-tailed learning do not seem promising for binary classification on tabular data, for which they were not initially intended. In \Cref{tab:lgbm_tuned_std}, we note that our introduced strategy MGS, display again good predictive results.

\paragraph{Multiclass} We extend our strategies to the multiclass label with 3 data sets that were originally multiclass  
: Wine, Yeast and Ecoli (described in \Cref{sec:supp-results}). For each data sets, the minority classes are augmented until the equilibrium is reached with the majority class. We applied our experimental protocol with LightGBM as classifier and show the results in \Cref{tab:resultats_multiclass_lgbm_pr}. First we remark that for Ecoli data set, None is among the best strategy, strengthening our statement that rebalancing strategies are not always necessary. Borderline SMOTE 1 induces the best predictive performances for Wine while we note that MGS is the best strategy for the last data set.




\section{CONCLUSION AND PERSPECTIVES}
In this paper, we analysed the impact of rebalancing strategies on predictive performance for binary classification tasks on tabular data. First, we prove that default SMOTE tends to copy the original minority samples asymptotically, and exhibits boundary artifacts.
From a computational perspective, we show that applying no rebalancing is competitive for most datasets, when used in conjunction with a tuned random forest/Logistic regression/LightGBM considering the PR AUC.
For extremely imbalanced data sets, rebalancing strategies lead to improved predictive performances, with or without tuning the maximum tree depth. Based on our theoretical findings, we propose two slight modifications of SMOTE. Tuning the hyperparameter $K$ of SMOTE appears to be necessary in theory (\Cref{thm_asymptot_SMOTE_k}) and in simulations (\Cref{sec:numerical_illustration}), but does not improve the predictive performance on the real-world data sets we considered.  
MGS, the other variant we propose, appears promising, with good predictive performances regardless of the tuning of random forests. Thus, a simple modification based on theoretical findings can improve SMOTE performances.

SMOTE was originally designed for continuous features, since it is based on interpolation, which is irrelevant for qualitative variables. That is our primary reason for studying SMOTE for continuous features only. Extensions of SMOTE to generate categorical features have been proposed, based on univariate nearest neighbors vote. However, nearest neighbors struggle in high dimensions. 
We plan to extend MGS to categorical features, replacing nearest neighbor vote by multi-output decision tree predictions, potentially resulting in more plausible generated samples \citep[see][\footnote{A version of this future work was published before the present article.}]{sakho2025harnessing}.


\subsubsection*{Acknowledgements}
The authors would like to express their gratitude to anonymous reviewers for their constructive feedback.

\clearpage
\bibliography{refs}

\begin{thebibliography}{62}
\providecommand{\natexlab}[1]{#1}
\providecommand{\url}[1]{\texttt{#1}}
\expandafter\ifx\csname urlstyle\endcsname\relax
  \providecommand{\doi}[1]{doi: #1}\else
  \providecommand{\doi}{doi: \begingroup \urlstyle{rm}\Url}\fi

\bibitem[Aghbalou et~al.(2024)Aghbalou, Sabourin, and Portier]{aghbalou2023sharp}
Anass Aghbalou, Anne Sabourin, and Fran{\c{c}}ois Portier.
\newblock Sharp error bounds for imbalanced classification: how many examples in the minority class?
\newblock In \emph{International Conference on Artificial Intelligence and Statistics}, pages 838--846. PMLR, 2024.

\bibitem[Aguiar et~al.(2023)Aguiar, Krawczyk, and Cano]{aguiar2023survey}
Gabriel Aguiar, Bartosz Krawczyk, and Alberto Cano.
\newblock A survey on learning from imbalanced data streams: taxonomy, challenges, empirical study, and reproducible experimental framework.
\newblock \emph{Machine learning}, pages 1--79, 2023.

\bibitem[Alinat(1993)]{alinat1993periodic}
Pierre Alinat.
\newblock Periodic progress report 4, roars project esprit ii-number 5516.
\newblock \emph{Technical Thomson Report TS ASM 93/S/EGS/NC}, 79, 1993.

\bibitem[Arlot and Genuer(2014)]{arlot2014analysis}
Sylvain Arlot and Robin Genuer.
\newblock Analysis of purely random forests bias.
\newblock \emph{arXiv preprint arXiv:1407.3939}, 2014.

\bibitem[Arya et~al.(1995)Arya, Mount, and Narayan]{arya1995accounting}
Sunil Arya, David~M Mount, and Onuttom Narayan.
\newblock Accounting for boundary effects in nearest neighbor searching.
\newblock In \emph{Proceedings of the eleventh annual symposium on computational geometry}, pages 336--344, 1995.

\bibitem[Berrett(2017)]{berrett_2017}
Thomas~Benjamin Berrett.
\newblock Modern k-nearest neighbour methods in entropy estimation, independence testing and classification.
\newblock \emph{Apollo - University of Cambridge Repository}, 2017.
\newblock \doi{10.17863/CAM.13756}.
\newblock URL \url{https://www.repository.cam.ac.uk/handle/1810/267832}.

\bibitem[Biau and Devroye(2015)]{biau2015lectures}
G{\'e}rard Biau and Luc Devroye.
\newblock \emph{Lectures on the nearest neighbor method}, volume 246.
\newblock Springer, 2015.

\bibitem[Brodersen et~al.(2010)Brodersen, Ong, Stephan, and Buhmann]{brodersen2010balanced}
Kay~Henning Brodersen, Cheng~Soon Ong, Klaas~Enno Stephan, and Joachim~M Buhmann.
\newblock The balanced accuracy and its posterior distribution.
\newblock In \emph{2010 20th international conference on pattern recognition}, pages 3121--3124. IEEE, 2010.

\bibitem[Bunkhumpornpat et~al.(2012)Bunkhumpornpat, Sinapiromsaran, and Lursinsap]{bunkhumpornpat2012dbsmote}
Chumphol Bunkhumpornpat, Krung Sinapiromsaran, and Chidchanok Lursinsap.
\newblock Dbsmote: density-based synthetic minority over-sampling technique.
\newblock \emph{Applied Intelligence}, 36:\penalty0 664--684, 2012.

\bibitem[Cao et~al.(2019)Cao, Wei, Gaidon, Arechiga, and Ma]{cao2019learning}
Kaidi Cao, Colin Wei, Adrien Gaidon, Nikos Arechiga, and Tengyu Ma.
\newblock Learning imbalanced datasets with label-distribution-aware margin loss.
\newblock \emph{Advances in neural information processing systems}, 32, 2019.

\bibitem[Carriero et~al.(2025)Carriero, Luijken, de~Hond, Moons, van Calster, and van Smeden]{carriero2025harms}
Alex Carriero, Kim Luijken, Anne de~Hond, Karel~GM Moons, Ben van Calster, and Maarten van Smeden.
\newblock The harms of class imbalance corrections for machine learning based prediction models: a simulation study.
\newblock \emph{Statistics in Medicine}, 44\penalty0 (3-4):\penalty0 e10320, 2025.

\bibitem[Cawley and Talbot(2010)]{cawley2010over}
Gavin~C Cawley and Nicola~LC Talbot.
\newblock On over-fitting in model selection and subsequent selection bias in performance evaluation.
\newblock \emph{The Journal of Machine Learning Research}, 11:\penalty0 2079--2107, 2010.

\bibitem[Chakravarti(1989)]{chakravarti1989isotonic}
Nilotpal Chakravarti.
\newblock Isotonic median regression: a linear programming approach.
\newblock \emph{Mathematics of operations research}, 14\penalty0 (2):\penalty0 303--308, 1989.

\bibitem[Chaudhuri et~al.(2023)Chaudhuri, Ahuja, Arjovsky, and Lopez-Paz]{arjovsky2022throwing}
Kamalika Chaudhuri, Kartik Ahuja, Martin Arjovsky, and David Lopez-Paz.
\newblock Why does throwing away data improve worst-group error?
\newblock In \emph{International Conference on Machine Learning}, pages 4144--4188. PMLR, 2023.

\bibitem[Chawla et~al.(2002)Chawla, Bowyer, Hall, and Kegelmeyer]{chawla2002smote}
Nitesh~V Chawla, Kevin~W Bowyer, Lawrence~O Hall, and W~Philip Kegelmeyer.
\newblock Smote: synthetic minority over-sampling technique.
\newblock \emph{Journal of artificial intelligence research}, 16:\penalty0 321--357, 2002.

\bibitem[Chawla et~al.(2003)Chawla, Lazarevic, Hall, and Bowyer]{chawla2003smoteboost}
Nitesh~V Chawla, Aleksandar Lazarevic, Lawrence~O Hall, and Kevin~W Bowyer.
\newblock Smoteboost: Improving prediction of the minority class in boosting.
\newblock In \emph{Knowledge Discovery in Databases: PKDD 2003: 7th European Conference on Principles and Practice of Knowledge Discovery in Databases, Cavtat-Dubrovnik, Croatia, September 22-26, 2003. Proceedings 7}, pages 107--119. Springer, 2003.

\bibitem[Chawla et~al.(2004)Chawla, Japkowicz, and Kotcz]{chawla2004special}
Nitesh~V Chawla, Nathalie Japkowicz, and Aleksander Kotcz.
\newblock Special issue on learning from imbalanced data sets.
\newblock \emph{ACM SIGKDD explorations newsletter}, 6\penalty0 (1):\penalty0 1--6, 2004.

\bibitem[Dal~Pozzolo et~al.(2015)Dal~Pozzolo, Caelen, Johnson, and Bontempi]{dal2015calibrating}
Andrea Dal~Pozzolo, Olivier Caelen, Reid~A Johnson, and Gianluca Bontempi.
\newblock Calibrating probability with undersampling for unbalanced classification.
\newblock In \emph{2015 IEEE symposium series on computational intelligence}, pages 159--166. IEEE, 2015.

\bibitem[Dua and Graff(2017)]{Dua:2019}
Dheeru Dua and Casey Graff.
\newblock Uci machine learning repository, 2017.
\newblock URL \url{http://archive.ics.uci.edu/ml}.

\bibitem[Elor and Averbuch-Elor(2022)]{elor2022smote}
Yotam Elor and Hadar Averbuch-Elor.
\newblock To smote, or not to smote?
\newblock \emph{arXiv preprint arXiv:2201.08528}, 2022.

\bibitem[Elreedy and Atiya(2019)]{elreedy2019comprehensive}
Dina Elreedy and Amir~F Atiya.
\newblock A comprehensive analysis of synthetic minority oversampling technique (smote) for handling class imbalance.
\newblock \emph{Information Sciences}, 505:\penalty0 32--64, 2019.

\bibitem[Elreedy et~al.(2023)Elreedy, Atiya, and Kamalov]{elreedy_theoretical_2023}
Dina Elreedy, Amir~F. Atiya, and Firuz Kamalov.
\newblock A theoretical distribution analysis of synthetic minority oversampling technique ({SMOTE}) for imbalanced learning.
\newblock \emph{Machine Learning}, January 2023.
\newblock ISSN 1573-0565.
\newblock \doi{10.1007/s10994-022-06296-4}.
\newblock URL \url{https://doi.org/10.1007/s10994-022-06296-4}.

\bibitem[Grinsztajn et~al.(2022)Grinsztajn, Oyallon, and Varoquaux]{grinsztajn2022tree}
L{\'e}o Grinsztajn, Edouard Oyallon, and Ga{\"e}l Varoquaux.
\newblock Why do tree-based models still outperform deep learning on typical tabular data?
\newblock \emph{Advances in neural information processing systems}, 35:\penalty0 507--520, 2022.

\bibitem[Han et~al.(2005)Han, Wang, and Mao]{han2005borderline}
Hui Han, Wen-Yuan Wang, and Bing-Huan Mao.
\newblock Borderline-smote: a new over-sampling method in imbalanced data sets learning.
\newblock In \emph{International conference on intelligent computing}, pages 878--887. Springer, 2005.

\bibitem[Harris et~al.(2020)Harris, Millman, van~der Walt, Gommers, Virtanen, Cournapeau, Wieser, Taylor, Berg, Smith, Kern, Picus, Hoyer, van Kerkwijk, Brett, Haldane, del R{\'{i}}o, Wiebe, Peterson, G{\'{e}}rard-Marchant, Sheppard, Reddy, Weckesser, Abbasi, Gohlke, and Oliphant]{harris2020array}
Charles~R. Harris, K.~Jarrod Millman, St{\'{e}}fan~J. van~der Walt, Ralf Gommers, Pauli Virtanen, David Cournapeau, Eric Wieser, Julian Taylor, Sebastian Berg, Nathaniel~J. Smith, Robert Kern, Matti Picus, Stephan Hoyer, Marten~H. van Kerkwijk, Matthew Brett, Allan Haldane, Jaime~Fern{\'{a}}ndez del R{\'{i}}o, Mark Wiebe, Pearu Peterson, Pierre G{\'{e}}rard-Marchant, Kevin Sheppard, Tyler Reddy, Warren Weckesser, Hameer Abbasi, Christoph Gohlke, and Travis~E. Oliphant.
\newblock Array programming with {NumPy}.
\newblock \emph{Nature}, 585\penalty0 (7825):\penalty0 357--362, September 2020.
\newblock \doi{10.1038/s41586-020-2649-2}.
\newblock URL \url{https://doi.org/10.1038/s41586-020-2649-2}.

\bibitem[Hassan and Abraham(2016)]{ex-fraud}
Amira Kamil~Ibrahim Hassan and Ajith Abraham.
\newblock Modeling insurance fraud detection using imbalanced data classification.
\newblock In \emph{Advances in Nature and Biologically Inspired Computing: Proceedings of the 7th World Congress on Nature and Biologically Inspired Computing (NaBIC2015) in Pietermaritzburg, South Africa, held December 01-03, 2015}, pages 117--127. Springer, 2016.

\bibitem[He and Garcia(2009)]{he2009learning}
Haibo He and Edwardo~A Garcia.
\newblock Learning from imbalanced data.
\newblock \emph{IEEE Transactions on knowledge and data engineering}, 21\penalty0 (9):\penalty0 1263--1284, 2009.

\bibitem[He et~al.(2008)He, Bai, Garcia, and Li]{he2008adasyn}
Haibo He, Yang Bai, Edwardo~A Garcia, and Shutao Li.
\newblock Adasyn: Adaptive synthetic sampling approach for imbalanced learning.
\newblock In \emph{2008 IEEE international joint conference on neural networks (IEEE world congress on computational intelligence)}, pages 1322--1328. Ieee, 2008.

\bibitem[Islam and Zhang(2020)]{gan-generate}
Jyoti Islam and Yanqing Zhang.
\newblock Gan-based synthetic brain pet image generation.
\newblock \emph{Brain informatics}, 7:\penalty0 1--12, 2020.

\bibitem[Jolicoeur-Martineau et~al.(2024)Jolicoeur-Martineau, Fatras, and Kachman]{jolicoeur2024generating}
Alexia Jolicoeur-Martineau, Kilian Fatras, and Tal Kachman.
\newblock Generating and imputing tabular data via diffusion and flow-based gradient-boosted trees.
\newblock In \emph{International Conference on Artificial Intelligence and Statistics}, pages 1288--1296. PMLR, 2024.

\bibitem[Jones(1993)]{jones1993simple}
M~Chris Jones.
\newblock Simple boundary correction for kernel density estimation.
\newblock \emph{Statistics and computing}, 3:\penalty0 135--146, 1993.

\bibitem[Kamalov et~al.(2022)Kamalov, Atiya, and Elreedy]{kamalov2022partial}
Firuz Kamalov, Amir~F Atiya, and Dina Elreedy.
\newblock Partial resampling of imbalanced data.
\newblock \emph{arXiv preprint arXiv:2207.04631}, 2022.

\bibitem[Ke et~al.(2017)Ke, Meng, Finley, Wang, Chen, Ma, Ye, and Liu]{ke2017lightgbm}
Guolin Ke, Qi~Meng, Thomas Finley, Taifeng Wang, Wei Chen, Weidong Ma, Qiwei Ye, and Tie-Yan Liu.
\newblock Lightgbm: A highly efficient gradient boosting decision tree.
\newblock \emph{Advances in neural information processing systems}, 30, 2017.

\bibitem[Khalilia et~al.(2011)Khalilia, Chakraborty, and Popescu]{RF-medical-example}
Mohammed Khalilia, Sounak Chakraborty, and Mihail Popescu.
\newblock Predicting disease risks from highly imbalanced data using random forest.
\newblock \emph{BMC medical informatics and decision making}, 11:\penalty0 1--13, 2011.

\bibitem[Kotelnikov et~al.(2023)Kotelnikov, Baranchuk, Rubachev, and Babenko]{kotelnikov2023tabddpm}
Akim Kotelnikov, Dmitry Baranchuk, Ivan Rubachev, and Artem Babenko.
\newblock Tabddpm: Modelling tabular data with diffusion models.
\newblock In \emph{International Conference on Machine Learning}, pages 17564--17579. PMLR, 2023.

\bibitem[Krawczyk(2016)]{krawczyk2016learning}
Bartosz Krawczyk.
\newblock Learning from imbalanced data: open challenges and future directions.
\newblock \emph{Progress in Artificial Intelligence}, 5\penalty0 (4):\penalty0 221--232, 2016.

\bibitem[Lema{{\^i}}tre et~al.(2017)Lema{{\^i}}tre, Nogueira, and Aridas]{JMLR:v18:16-365}
Guillaume Lema{{\^i}}tre, Fernando Nogueira, and Christos~K. Aridas.
\newblock Imbalanced-learn: A python toolbox to tackle the curse of imbalanced datasets in machine learning.
\newblock \emph{Journal of Machine Learning Research}, 18\penalty0 (17):\penalty0 1--5, 2017.
\newblock URL \url{http://jmlr.org/papers/v18/16-365.html}.

\bibitem[Lin et~al.(2017)Lin, Goyal, Girshick, He, and Doll{\'a}r]{lin2017focal}
Tsung-Yi Lin, Priya Goyal, Ross Girshick, Kaiming He, and Piotr Doll{\'a}r.
\newblock Focal loss for dense object detection.
\newblock In \emph{Proceedings of the IEEE international conference on computer vision}, pages 2980--2988, 2017.

\bibitem[Liu et~al.(2019)Liu, Miao, Zhan, Wang, Gong, and Yu]{liu2019large}
Ziwei Liu, Zhongqi Miao, Xiaohang Zhan, Jiayun Wang, Boqing Gong, and Stella~X Yu.
\newblock Large-scale long-tailed recognition in an open world.
\newblock In \emph{Proceedings of the IEEE/CVF conference on computer vision and pattern recognition}, pages 2537--2546, 2019.

\bibitem[Mani and Zhang(2003)]{mani2003knn}
Inderjeet Mani and I~Zhang.
\newblock knn approach to unbalanced data distributions: a case study involving information extraction.
\newblock In \emph{Proceedings of workshop on learning from imbalanced datasets}, volume 126, pages 1--7. ICML, 2003.

\bibitem[Menardi and Torelli(2014)]{menardi2014training}
Giovanna Menardi and Nicola Torelli.
\newblock Training and assessing classification rules with imbalanced data.
\newblock \emph{Data mining and knowledge discovery}, 28:\penalty0 92--122, 2014.

\bibitem[Mohammed et~al.(2020)Mohammed, Hassan, and Kadir]{improve-smote-zada}
Ahmed~Jameel Mohammed, Masoud~Muhammed Hassan, and Dler~Hussein Kadir.
\newblock Improving classification performance for a novel imbalanced medical dataset using smote method.
\newblock \emph{International Journal of Advanced Trends in Computer Science and Engineering}, 9\penalty0 (3):\penalty0 3161--3172, 2020.

\bibitem[Mourtada et~al.(2020)Mourtada, Ga{\"\i}ffas, and Scornet]{mourtada2020minimax}
Jaouad Mourtada, St{\'e}phane Ga{\"\i}ffas, and Erwan Scornet.
\newblock Minimax optimal rates for mondrian trees and forests.
\newblock \emph{Annals of Statistics}, 48\penalty0 (4):\penalty0 2253--2276, 2020.

\bibitem[Nguyen et~al.(2011)Nguyen, Cooper, and Kamei]{nguyen2011borderline}
Hien~M Nguyen, Eric~W Cooper, and Katsuari Kamei.
\newblock Borderline over-sampling for imbalanced data classification.
\newblock \emph{International Journal of Knowledge Engineering and Soft Data Paradigms}, 3\penalty0 (1):\penalty0 4--21, 2011.

\bibitem[Nguyen and Duong(2021)]{nguyen2021comparison}
Nam~N Nguyen and Anh~T Duong.
\newblock Comparison of two main approaches for handling imbalanced data in churn prediction problem.
\newblock \emph{Journal of advances in information technology}, 12\penalty0 (1), 2021.

\bibitem[Pan et~al.(2020)Pan, Zhao, Wu, and Yang]{pan2020learning}
Tingting Pan, Junhong Zhao, Wei Wu, and Jie Yang.
\newblock Learning imbalanced datasets based on smote and gaussian distribution.
\newblock \emph{Information Sciences}, 512:\penalty0 1214--1233, 2020.

\bibitem[Pedregosa et~al.(2011)Pedregosa, Varoquaux, Gramfort, Michel, Thirion, Grisel, Blondel, Prettenhofer, Weiss, Dubourg, et~al.]{pedregosa2011scikit}
Fabian Pedregosa, Ga{\"e}l Varoquaux, Alexandre Gramfort, Vincent Michel, Bertrand Thirion, Olivier Grisel, Mathieu Blondel, Peter Prettenhofer, Ron Weiss, Vincent Dubourg, et~al.
\newblock Scikit-learn: Machine learning in python.
\newblock \emph{the Journal of machine Learning research}, 12:\penalty0 2825--2830, 2011.

\bibitem[Ramyachitra and Manikandan(2014)]{ramyachitra2014imbalanced}
Duraisamy Ramyachitra and Parasuraman Manikandan.
\newblock Imbalanced dataset classification and solutions: a review.
\newblock \emph{International Journal of Computing and Business Research (IJCBR)}, 5\penalty0 (4):\penalty0 1--29, 2014.

\bibitem[Ren et~al.(2018)Ren, Zeng, Yang, and Urtasun]{ren2018learning}
Mengye Ren, Wenyuan Zeng, Bin Yang, and Raquel Urtasun.
\newblock Learning to reweight examples for robust deep learning.
\newblock In \emph{International conference on machine learning}, pages 4334--4343. PMLR, 2018.

\bibitem[Saito and Rehmsmeier(2015)]{saito2015precision}
Takaya Saito and Marc Rehmsmeier.
\newblock The precision-recall plot is more informative than the roc plot when evaluating binary classifiers on imbalanced datasets.
\newblock \emph{PloS one}, 10\penalty0 (3):\penalty0 e0118432, 2015.

\bibitem[Sakho et~al.(2025)Sakho, Malherbe, Gauthier, and Scornet]{sakho2025harnessing}
Abdoulaye Sakho, Emmanuel Malherbe, Carl-Erik Gauthier, and Erwan Scornet.
\newblock Harnessing mixed features for imbalance data oversampling: Application to bank customers scoring.
\newblock In \emph{Joint European Conference on Machine Learning and Knowledge Discovery in Databases}, pages 247--264. Springer, 2025.

\bibitem[Shwartz-Ziv and Armon(2022)]{shwartz2022tabular}
Ravid Shwartz-Ziv and Amitai Armon.
\newblock Tabular data: Deep learning is not all you need.
\newblock \emph{Information Fusion}, 81:\penalty0 84--90, 2022.

\bibitem[Van~Horn et~al.(2018)Van~Horn, Mac~Aodha, Song, Cui, Sun, Shepard, Adam, Perona, and Belongie]{van2018inaturalist}
Grant Van~Horn, Oisin Mac~Aodha, Yang Song, Yin Cui, Chen Sun, Alex Shepard, Hartwig Adam, Pietro Perona, and Serge Belongie.
\newblock The inaturalist species classification and detection dataset.
\newblock In \emph{Proceedings of the IEEE conference on computer vision and pattern recognition}, pages 8769--8778, 2018.

\bibitem[Wadsworth et~al.(1961)Wadsworth, Bryan, and Eringen]{wadsworth1961introduction}
George~Proctor Wadsworth, Joseph~G Bryan, and A~Cemal Eringen.
\newblock Introduction to probability and random variables.
\newblock \emph{Journal of Applied Mechanics}, 28\penalty0 (2):\penalty0 319, 1961.

\bibitem[Wallace and Dahabreh(2014)]{wallace2014improving}
Byron~C Wallace and Issa~J Dahabreh.
\newblock Improving class probability estimates for imbalanced data.
\newblock \emph{Knowledge and information systems}, 41\penalty0 (1):\penalty0 33--52, 2014.

\bibitem[Wongvorachan et~al.(2023)Wongvorachan, He, and Bulut]{wongvorachan2023comparison}
Tarid Wongvorachan, Surina He, and Okan Bulut.
\newblock A comparison of undersampling, oversampling, and smote methods for dealing with imbalanced classification in educational data mining.
\newblock \emph{Information}, 14\penalty0 (1):\penalty0 54, 2023.

\bibitem[Xie et~al.(2020)Xie, Qiu, Zhang, Peng, and Chen]{xie2020gaussian}
Yuxi Xie, Min Qiu, Haibo Zhang, Lizhi Peng, and Zhenxiang Chen.
\newblock Gaussian distribution based oversampling for imbalanced data classification.
\newblock \emph{IEEE Transactions on Knowledge and Data Engineering}, 34\penalty0 (2):\penalty0 667--679, 2020.

\bibitem[Xu et~al.(2019)Xu, Skoularidou, Cuesta-Infante, and Veeramachaneni]{xu2019modeling}
Lei Xu, Maria Skoularidou, Alfredo Cuesta-Infante, and Kalyan Veeramachaneni.
\newblock Modeling tabular data using conditional gan.
\newblock \emph{Advances in neural information processing systems}, 32, 2019.

\bibitem[Xu et~al.(2020)Xu, Dan, Khim, and Ravikumar]{xu2020class}
Ziyu Xu, Chen Dan, Justin Khim, and Pradeep Ravikumar.
\newblock Class-weighted classification: Trade-offs and robust approaches.
\newblock In \emph{International Conference on Machine Learning}, pages 10544--10554. PMLR, 2020.

\bibitem[Zhang et~al.(2018)Zhang, Ciss{\'{e}}, Dauphin, and Lopez{-}Paz]{zhang2017mixup}
Hongyi Zhang, Moustapha Ciss{\'{e}}, Yann~N. Dauphin, and David Lopez{-}Paz.
\newblock mixup: Beyond empirical risk minimization.
\newblock In \emph{6th International Conference on Learning Representations, {ICLR} 2018, Vancouver, BC, Canada, April 30 - May 3, 2018, Conference Track Proceedings}. OpenReview.net, 2018.
\newblock URL \url{https://openreview.net/forum?id=r1Ddp1-Rb}.

\bibitem[Zhang et~al.(2023)Zhang, Kang, Hooi, Yan, and Feng]{zhang2023deep}
Yifan Zhang, Bingyi Kang, Bryan Hooi, Shuicheng Yan, and Jiashi Feng.
\newblock Deep long-tailed learning: A survey.
\newblock \emph{IEEE Transactions on Pattern Analysis and Machine Intelligence}, 2023.

\bibitem[Zhu et~al.(2018)Zhu, Xia, Jin, Yan, Cai, Yan, and Ning]{c-weight-RF-bon}
Min Zhu, Jing Xia, Xiaoqing Jin, Molei Yan, Guolong Cai, Jing Yan, and Gangmin Ning.
\newblock Class weights random forest algorithm for processing class imbalanced medical data.
\newblock \emph{IEEE Access}, 6:\penalty0 4641--4652, 2018.

\end{thebibliography}
\bibliographystyle{plainnat}

\clearpage
\section*{Checklist}



\begin{enumerate}

  \item For all models and algorithms presented, check if you include:
  \begin{enumerate}
    \item A clear description of the mathematical setting, assumptions, algorithm, and/or model. \textbf{Yes.}
    \item An analysis of the properties and complexity (time, space, sample size) of any algorithm. \textbf{Yes, we present computation times. }
    \item (Optional) Anonymized source code, with specification of all dependencies, including external libraries. \textbf{Yes, we provide a Github repository.}
  \end{enumerate}

  \item For any theoretical claim, check if you include:
  \begin{enumerate}
    \item Statements of the full set of assumptions of all theoretical results. \textbf{Yes.}
    \item Complete proofs of all theoretical results. \textbf{Yes.}
    \item Clear explanations of any assumptions. \textbf{Yes.}     
  \end{enumerate}

  \item For all figures and tables that present empirical results, check if you include:
  \begin{enumerate}
    \item The code, data, and instructions needed to reproduce the main experimental results (either in the supplemental material or as a URL). \textbf{Yes.}
    \item All the training details (e.g., data splits, hyperparameters, how they were chosen). \textbf{Yes.}
    \item A clear definition of the specific measure or statistics and error bars (e.g., with respect to the random seed after running experiments multiple times). \textbf{Yes.}
    \item A description of the computing infrastructure used. (e.g., type of GPUs, internal cluster, or cloud provider). \textbf{Yes.}
  \end{enumerate}

  \item If you are using existing assets (e.g., code, data, models) or curating/releasing new assets, check if you include:
  \begin{enumerate}
    \item Citations of the creator If your work uses existing assets. \textbf{Yes.}
    \item The license information of the assets, if applicable. \textbf{Yes.}
    \item New assets either in the supplemental material or as a URL, if applicable. \textbf{Yes.}
    \item Information about consent from data providers/curators. \textbf{Yes.}
    \item Discussion of sensible content if applicable, e.g., personally identifiable information or offensive content. \textbf{Not Applicable.}
  \end{enumerate}

  \item If you used crowdsourcing or conducted research with human subjects, check if you include:
  \begin{enumerate}
    \item The full text of instructions given to participants and screenshots. \textbf{Not Applicable.}
    \item Descriptions of potential participant risks, with links to Institutional Review Board (IRB) approvals if applicable. \textbf{Not Applicable.}
    \item The estimated hourly wage paid to participants and the total amount spent on participant compensation. \textbf{Not Applicable.}
  \end{enumerate}

\end{enumerate}

\clearpage

\thispagestyle{empty}

\onecolumn
\doparttoc
\faketableofcontents
\part{}
\appendix

\aistatstitle{Supplementary Materials to ``Do we need rebalancing strategies?  A theoretical and empirical study around SMOTE and its variants"}

\vspace{1cm}
\parttoc 
\newpage

\section{PREDICTIVE EVALUATION REAL-WORLD DATA SETS: REMARKS}
\label{sec:supp-results}

\begin{table}[h]
\center
\caption{Initial data sets.
}
    \label{table:data-sets}
    \begin{tabular}{lccr} 
         \toprule
          &  $N$ & $n/N$ & $d$   \\ 
         \midrule
         Haberman & $306$ & $26\%$ & $3$\\
         Ionosphere & $351$ & $36\%$ & $32$\\
         \shortstack{Breast cancer 
         } & $630$ & $36\%$ & $9$ \\
         Pima & $768$ & $35\% $ & $8$   \\
         \shortstack{Vehicle 
         } & $846$ & $23\%$ & $18$ \\ 
         Yeast 
         & $1\,462$ & $11\%$ & $8$ \\
         Abalone 
         &$4\,177$ & $1\% $ &$8$   \\
         \shortstack{Wine 
         }  & $4\,974$ & $4\% $ & $11$ \\  
         Phoneme & $5\,404$ & $29\%$ & $5 $     \\ 
         MagicTel &$13\,376$ & $50\%$ &$10$ \\
         California &$20\,634$ & $50\%$ &$8$ \\
         House\_16H &$22\,784$ & $30\%$ &$16$ \\ 
         CreditCard & $284\,315$ & $0.2\% $ & $29$  \\
 \bottomrule
\end{tabular}
\end{table}

\begin{table}[h]
\center
\caption{Subsampled data sets.}
\label{table:data-sets-annexe}
\begin{tabular}{lccr} 
     \toprule
      &  $N$ & $n/N$ & $d$   \\ 
     \midrule
     \textit{Haberman} ($10\%$) & $250$ & $10\%$ & $3$\\
     \textit{Ionosphere} ($20\%$) & $281$ & $20\%$ & $32$\\
     \textit{Ionosphere} ($10\%$) & $250$ & $10\%$ & $32$\\
     \shortstack{\textit{Breast cancer} ($20\%$) 
     } & $500$ & $20\%$ & $9$ \\
     \shortstack{\textit{Breast cancer} ($10\%$) 
     } & $444$ & $10\%$ & $9$ \\
     
     Pima ($20\%$) & $625$ & $20\%$ & $8$   \\
     \shortstack{\textit{Vehicle} ($10\%$) } & $718$ & $10\%$ & $18$ \\ 
     \textit{Yeast} ($1\%$) 
     & $1\,334$ & $1\%$ & $8$ \\ 
     \textit{Phoneme} ($20\%$) & $4\,772$ & $20\%$ & $5 $     \\
     \textit{Phoneme} ($10\%$) & $4\,242$ & $10\%$ & $5 $     \\ 
     \textit{Phoneme} ($1\%$) & $3\,856$ & $1\%$ & $5 $     \\ 
     \textit{MagicTel} ($20\%$) & $8\,360$ & $20\%$ & $10$ \\
     \textit{House\_16H} ($20\%$) & $20\,050$ & $20\%$ & $16$ \\
     \textit{House\_16H} ($10\%$) & $17\,822$ & $10\%$ & $16$ \\
     \textit{House\_16H} ($1\%$) & $16\,202$ & $1\%$ & $16$ \\
     \textit{California} ($20\%$) & $12\,896$ & $20\%$ & $8$ \\
     \textit{California} ($10\%$) & $11\,463$ & $10\%$ & $8$ \\
     \textit{California} ($1\%$) & $10\,421$ & $1\%$ & $8$ \\
 \bottomrule
\end{tabular}
\end{table}

\paragraph{General comment on the protocol for real-world data sets} 

The state-of-the-art rebalancing strategies\citep[see][]{JMLR:v18:16-365} are used with their default hyperparameter values. 

The subsampled data sets (see \Cref{table:data-sets-annexe}) can be obtained through the repository (the functions and the seeds are given in a jupyter notebook). For the CreditCard data set, a Time Series split is performed instead of a Stratified $5-$fold, because of the temporality of the data. Furthermore, a group out is applied on the different scope time value. 

For MagicTel and California data sets, the initial data sets are already balanced, leading to no opportunity for applying a rebalancing strategy. This is the reason why we do not include these original data sets in our study but only their subsampled associated data sets.

The max\_depth hyperparameter is tuned using GridSearch function from scikit-learn. The grids minimum is 5 and the grid maximum is the mean depth of the given strategy for the given data set (when random forest is used without tuning depth hyperparameter). Then, using numpy \citep[][]{harris2020array}, a list of integer of size 8 between the minimum value and the maximum is value is built. Finally, the "None" value is added to this list.

\paragraph{Mean standard deviation} For each protocol run, we computed the mean of the PR AUC over the 5-fold. Then, all of these $20$ means are averaged in order to get what we call in some of our tables the mean standard deviations.

\paragraph{Multiclass data sets}
The Ecoli data set is composed of $327$ samples with $7$ features. The target is formed of $5$ categories. The multiclass Wine data set is composed of $11$ features with $6492$ samples. The target is formed of $6$ categories. The multiclass Yeast data set is composed of $8$ features with $1394$ samples. The target is formed of $6$ categories.

\paragraph{More results about None strategy from seminal papers}
It appears in the results of two seminal papers that the None strategy is competitive in terms of predictive performances when applied with weak classifiers, a contrario of our work. Furthermore, only \citet{he2008adasyn} highlight this phenomenon in section III.C.  
\citet{he2008adasyn} compare the None strategy, ADASYN and SMOTE, followed by a decision tree classifier on $5$ data sets (including Vehicle, Pima, Ionosphere and Abalone).  In terms of Precision and F1 score, the None strategy is on par with the two other rebalancing methods. \citet{han2005borderline} study the impact of Borderline SMOTE and others SMOTE variant on $4$ data sets (including Pima and Haberman) when using C4.5 classifier. The None strategy is competitive (in terms of F1 score) on two of these data sets. 

\cite{elor2022smote} performs an empirical comparison of oversampling strategies combined with various classifiers. The purpose is to find in which setting SMOTE, and others oversampling strategies, can improve predictive performances. While informative, it suffers from a key limitation: categorical variables are ordinally encoded, making linear interpolation via SMOTE producing not “plausible” samples. Moreover, their evaluation relies on F1 score, tuned only for the “no rebalancing” case (in Table 3 from \cite{elor2022smote}), and ROC AUC, which is not relevant in imbalanced settings as stated above and by \cite{saito2015precision}. The F1-score is tuned for all classifiers in Figure 2 from \cite{elor2022smote} but we observe that oversampling improves LightGBM and XGBoost performances. 
Finally, the authors of \cite{elor2022smote} claim that applying no rebalancing strategy performs well but their analysis excludes both undersampling and deep-learning-based oversampling methods \citep[such as CTGAN or ForestDiffusion][]{xu2019modeling,jolicoeur2024generating}. 

\cite{carriero2025harms} investigates classifier calibration post-rebalancing on both synthetic and MIMIC-III data sets. \cite{carriero2025harms} state that when applying no rebalancing strategy, the quality of the calibration is better. While useful, its conclusions are based on a single real data set and focus solely on calibration quality, not on predictive performance. A model can be poorly calibrated yet perform well on classification tasks with respect to a classification metric. Well-calibrated models are required for use cases where output probabilities are combined with other sources of information for decision making. Furthermore, post-processing calibration methods \citep[such as isotonic regression][]{chakravarti1989isotonic} exist and were not considered in their study to circumvent drawbacks of rebalancing strategies.

\paragraph{Hardware}
For all the numerical experiments, we use the following processor : AMD Ryzen Threadripper PRO 5955WX: 16 cores, 4.0 GHz, 64 MB cache, PCIe 4.0. We also add access to $250$GB of RAM.

\paragraph{Computation time}
The computation times reported in \Cref{tab:corps_resultats_rf_tuned} were obtained using the Phoneme data set.

\section{NUMERICAL ILLUSTRATIONS OF THEOREMS: REMARKS}
\label{appsec:num_illustrations_theorems}
Through \Cref{section:study_smote}, we highlighted that SMOTE asymptotically regenerates the distribution of the minority class, by tending to copy the minority samples. The purpose of this section is to numerically illustrate the theoretical limitations of SMOTE, typically with the default value $K=5$. 

$
$

\subsection{Simulated data}

\begin{figure}[b]
    \begin{minipage}{0.5\linewidth}
        \centering
        \includegraphics[width=0.8\columnwidth]{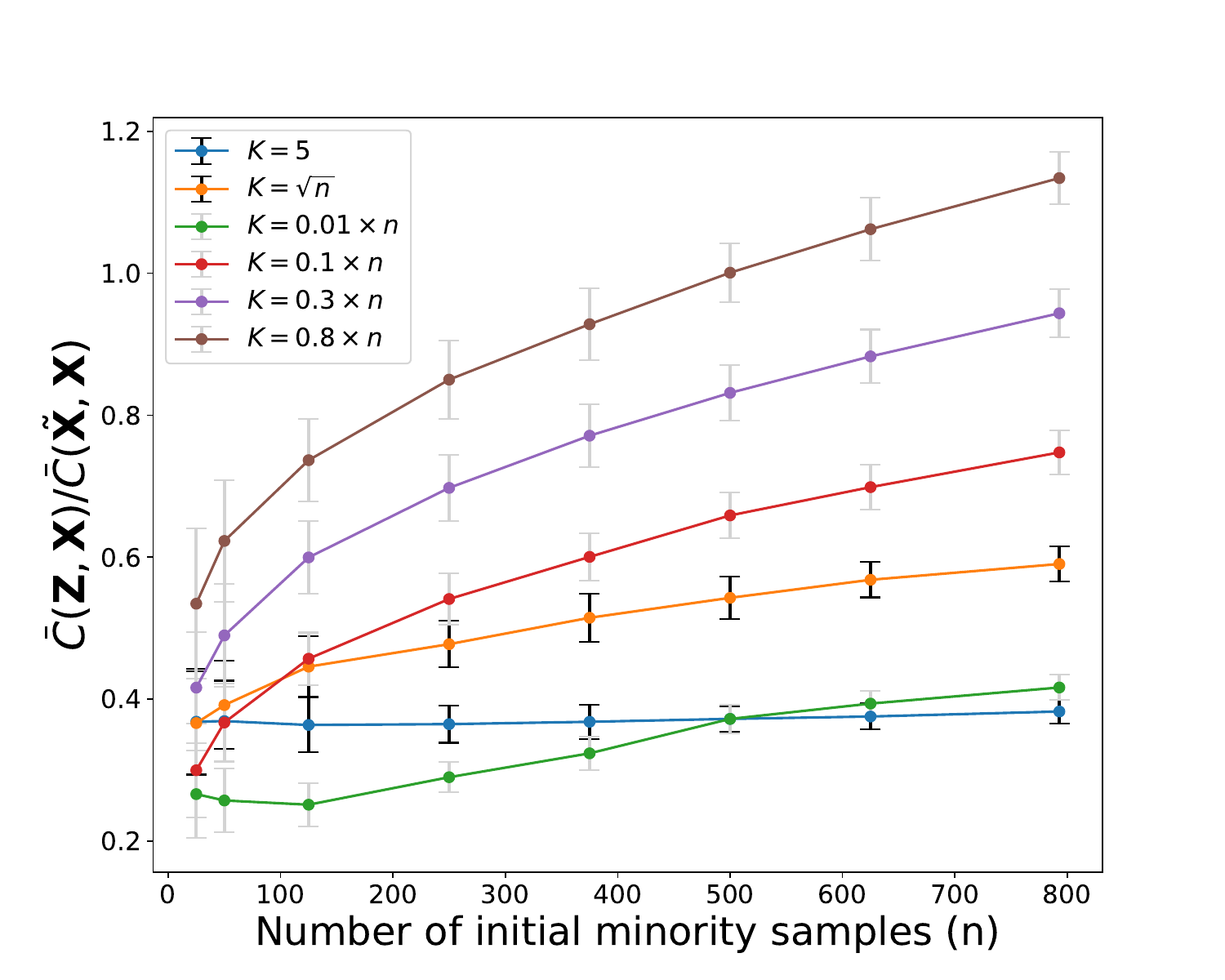}
        \caption{$\bar{C}(\mathbf{Z}, \mathbf{X})/\bar{C}(\tilde{\mathbf{X}}, \mathbf{X})$with Phoneme data.}
        \label{fig:protocol-3-phoneme}
    \end{minipage}  
    \begin{minipage}{0.5\linewidth}
        \centering
        \includegraphics[width=0.8\columnwidth]{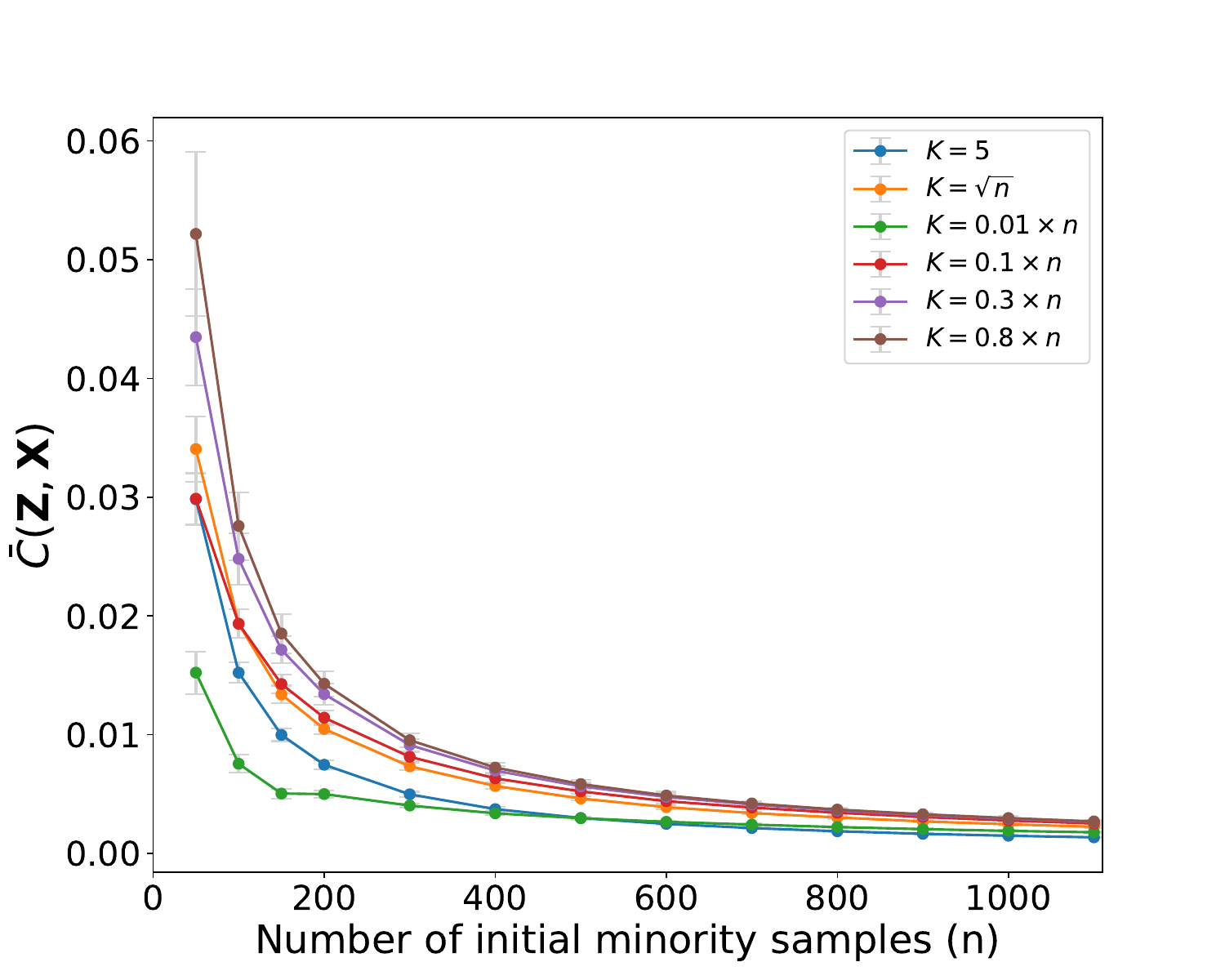}
        \caption{Average distance $\bar{C}(\mathbf{Z}, \mathbf{X})$.}
        \label{fig:dist-nn_U2D}
    \end{minipage}
\end{figure}

The quantity $\bar{C}(Z, X)$ measures how far SMOTE observations $Z_i$ are from the original observations $X_i$. The quantity $\bar{C}(\tilde{X}, X)$ measures how far data $\tilde{X}_i$, distributed as the original observations $X_i$, are from these observations ($X_i$). The metric $\bar{C}(Z, X)/\bar{C}(\tilde{X}, X)$ quantify the ability of the synthetic samples to fill the space in the same way as samples generated with the true distribution.  When $\bar{C}(Z, X)/\bar{C}(\tilde{X}, X)$ is close to one, SMOTE observations are comparable to observations $\tilde{X}$, in terms of the metric $\bar{C}$. In this case, data generated via SMOTE behave in the same way as the original data, with respect to the quantity $\bar{C}$. Similarly, when the quantity $\bar{C}(Z, X)/\bar{C}(\tilde{X}, X)$ is close to zero, observations generated via SMOTE are much closer to their central point than observations generated with the true distribution. In this case, this highlights that SMOTE has a tendency to generate data close to the original observations.  Conversely, when  $\bar{C}(Z, X)/\bar{C}(\tilde{X}, X)$ is larger than one, data generated via SMOTE are further away from their central point than observations generated with the true distribution. In this case, this highlight that diversity of newly created samples is too high. 


\textbf{Results.}
\Cref{fig:dist-nn-normalized_U2D} shows the renormalized quantity  $\bar{C}(\mathbf{Z}, \mathbf{X})/\bar{C}(\tilde{\mathbf{X}}, \mathbf{X})$ as a function of $n$. We notice that the asymptotic for $K=5$ is different since it is the only one where the distance between SMOTE data points and original data points does not vary with $n$. Besides, this distance is smaller than the other ones, thus stressing out that the SMOTE data points are very close to the original distribution for $K=5$. 
Note that, for the other asymptotics in $K$, the diversity of SMOTE observations increases with $n$, meaning $\bar{C}(\mathbf{Z}, \mathbf{X})$ gets closer from $\bar{C}(\tilde{\mathbf{X}}, \mathbf{X})$. This behavior in terms of average distance is ideal, since $\tilde{\mathbf{X}}$ is drawn from the same theoretical distribution as $\mathbf{X}$. On the contrary, $K=5$ keeps a lower average distance, showing a lack of diversity of generated points. Besides, this diversity is asymptotically more important for $K=0.1n$ and $K=0.01n$. This corroborates our theoretical findings (\Cref{th:regenerate}) as these asymptotics do not satisfy $K/n \to 0$. Indeed, when $K$ is set to a fraction of $n$, the SMOTE distribution does not converge to the original distribution anymore, therefore generating data points that are not simple copies of the original uniform samples. 
By construction, SMOTE data points are close to central points, which may explain why the quantity of interest in \Cref{fig:dist-nn-normalized_U2D} is smaller than $1$. 

\subsection{Real-world data sets}

\textbf{Extension to real-world data sets}
We extended our protocol to a real-world data set by splitting the data into two sets of equal size $\mathbf{X}$ and $\tilde{\mathbf{X}}$. The first one is used for applying SMOTE strategies to sample $\mathbf{Z}$ and the other set is used to compute the normalization factor $\bar{C}(\tilde{\mathbf{X}}, \mathbf{X})$. 

\textbf{Results} We apply the adapted protocol to Phoneme data set, described in \Cref{table:data-sets}. \Cref{fig:protocol-3-phoneme} displays the quantity $\bar{C}(\mathbf{Z}, \mathbf{X})/\bar{C}(\tilde{\mathbf{X}}, \mathbf{X})$ as a function of the size $n$ of the minority class. As above, we observe in \Cref{fig:protocol-3-phoneme} that the average normalized distance $\bar{C}(\mathbf{Z}, \mathbf{X})/\bar{C}(\tilde{\mathbf{X}}, \mathbf{X})$  increases for all strategies but the one with $K=5$. The strategies using a value of hyperparameter $K$ such that $K/n \to 0$ seem to converge to a value smaller than all the strategies with $K$ such that $K/n \not \to 0$.

\paragraph{Results with $\bar{C}(\mathbf{Z}, \mathbf{X})$}
\Cref{fig:dist-nn_U2D} depicts the quantity $\bar{C}(\mathbf{Z}, \mathbf{X})$ as a function of the size of the minority class, for different values of $K$. 
The metric $\bar{C}(\mathbf{Z},\mathbf{X})$ is consistently smaller for $K=5$ than for other values of $K$, therefore highlighting that data generated by SMOTE with $K=5$ are closer to the original data sample. This phenomenon is strengthened as $n$ increases. This is an artifact of the simulation setting as the original data samples fill the input space as $n$ increases. 

\paragraph{More details on the numerical illustrations protocol applied to real-world data sets } We apply SMOTE on real-world data and compare the distribution of the generated data points to the original distribution, using the metric $\bar{C}(\mathbf{Z}, \mathbf{X})/\bar{C}(\tilde{\mathbf{X}}, \mathbf{X})$.

For each value of $n$, we subsample $n$ data points from the minority class. Then,  
\begin{enumerate}
\item We uniformly split the data set into $X_1, \hdots, X_{n/2}$ (denoted by $\mathbf{X}$) and $\tilde{X}_1, \hdots, \tilde{X}_{n/2}$ (denoted by $\tilde{\mathbf{X}}$). 
    \item We generate a data set $\mathbf{Z}$ composed of $m=n/2$ i.i.d new observations $Z_1, \hdots, Z_m$ by applying SMOTE procedure on the original data set $\mathbf{X}$, with different values of $K$. We compute $C(\mathbf{Z}, \mathbf{X})$.
    \item We use $\tilde{\mathbf{X}}$ in order to compute
    $C(\tilde{\mathbf{X}}, \mathbf{X})$.
\end{enumerate}
This procedure is repeated $B=100$ times to compute averages values as in Section \ref{subsec:simulated-data}.

\clearpage

\section{ADDITIONAL THEORETICAL RESULTS}

\subsection{\Cref{th-densite}}
\begin{lemma}
\label{th-densite}
Let $X_c$ be the central point chosen in a SMOTE iteration. Then, for all $x_c \in \mathcal{X}$, the random variable $Z_{K,n}$ generated by SMOTE has a conditional density $f_{Z_{K,n}}(.|X_c=x_c)$ which satisfies
\begin{align}
    f_{Z_{K,n}}(z|X_c=x_c) &= (n-K-1)\binom{n-1}{K}\int_0^1\frac{1}{w^d}f_X\left(x_c+\frac{z-x_c}{w}\right) \nonumber \\
    & \quad \times \mathcal{B}\left(n-K-1,K ; 1 - \beta_{x_c, z, w} \right)  \mathrm{d}w,
    \end{align}
    where 
    $\beta_{x_c, z, w} = \mu_X \left( B\left(x_c, ||z-x_c||/w \right) \right)$ and $\mu_X$ is the probability measure associated to $f_X$. 
Using the following substitution $w = \|z-x_c\|/r$, we have,
\begin{align}
    f_{Z_{K,n}}(z|X_c=x_c)  &=  (n-K-1) \quad \times  \binom{n-1}{K}    \int_{r=\|z-x_c\|}^{\infty} f_X\left(x_c+ \frac{(z-x_c)r}{\|z-x_c\|}\right) \nonumber \\
    & \quad \times \frac{r^{d-2}  \mathcal{B}\left(n-K-1,K ;1-\mu_X \left( B\left(x_c,r\right) \right)\right)}{\|z-x_c\|^{d-1}}  \mathrm{d}r.
\end{align}
\end{lemma}

The proof of \Cref{th-densite} is available in \Cref{proof_th-densite}.


\section{MAIN PROOFS}
\label{appendix_proofs}

This section contains the main proof of our theoretical results. The technicals lemmas used by several proofs are available on \Cref{section:tech_lemmas}.

\subsection{Proof of \Cref{th:support_conv}}
\begin{proof}[Proof of \Cref{th:support_conv}]
\label{proof_th:support_conv}

Let $\mathcal{X}$ be the support of $P_X$.
SMOTE generates new points by linear interpolation of the original minority sample. This means that for all $x,y$ in the minority samples or generated by SMOTE procedure, we have $(1-t)x +ty \in Conv(\mathcal{X})$ by definition of $Conv(\mathcal{X})$.
This leads to the fact that precisely, all the new SMOTE samples are contained in $Conv(\mathcal{X})$. This implies $ Supp(P_Z) \subseteq Conv(\mathcal{X})$.

\end{proof}

\subsection{Proof of \Cref{th:regenerate}}
\begin{proof} [Proof of \Cref{th:regenerate}]
\label{proof_th:regenerate}
For any event $A, B$, we have
\begin{align}
1 - \mathds{P}[A \cap B] = \mathds{P} [A^c \cup B^c] \leq \mathds{P}[A^c] + \mathds{P}[B^c],
\end{align}
which leads to 
\begin{align}
    \mathds{P}[A\cap B] & \geq 1 - \mathds{P}[A^c] - \mathds{P}[B^c]\\
     & = \mathds{P}[A] - \mathds{P}[B^c].
\end{align}
%
By construction,
\begin{align}
    \| X_c - Z \| \leq \| X_c - X_{(K)}(X_c) \|.
\end{align}
Let $x \in  \mathcal{X}$ and $\eta >0$. Let $\alpha, \varepsilon >0$.
We have, 
\begin{align}
   & \mathds{P}[ X_c \in B(x, \alpha - \varepsilon)] - \mathds{P}[\| X_c - X_{(K)}(X_c) \| > \varepsilon] \\ 
   & \leq   \mathds{P}[ X_c \in B(x, \alpha - \varepsilon), \| X_c - X_{(K)}(X_c) \| \leq \varepsilon] \\
   & \leq   \mathds{P}[ X_c \in B(x, \alpha - \varepsilon), \| X_c - Z \| \leq \varepsilon] \\
   & \leq \mathds{P}[ Z \in B(x, \alpha)].
\end{align}
Similarly, we have 
\begin{align}
   & \mathds{P}[ Z \in B(x, \alpha)] - \mathds{P}[\| X_c - X_{(K)}(X_c) \| > \varepsilon] \\
   & \leq   \mathds{P}[ Z \in B(x, \alpha), \| X_c - X_{(K)}(X_c) \| \leq \varepsilon] \\
   & \leq   \mathds{P}[ Z \in B(x, \alpha), \| X_c - Z \| \leq \varepsilon] \\
   & \leq \mathds{P}[ X_c \in B(x, \alpha + \varepsilon)].
\end{align}
Since $X_c$ admits a density, for all $\varepsilon>0$ small enough
\begin{align}
     \mathds{P}[ X_c \in B(x, \alpha + \varepsilon)] \leq \mathds{P}[ X_c \in B(x, \alpha)] + \eta, \label{eq_proof_epsilon1}
\end{align}
and
\begin{align}
    \mathds{P}[ X_c \in B(x, \alpha)] - \eta \leq \mathds{P}[ X_c \in B(x, \alpha - \varepsilon)]. \label{eq_proof_epsilon2}
\end{align}
Let $\varepsilon$ such that \eqref{eq_proof_epsilon1} and  \eqref{eq_proof_epsilon2} are verified. According to Lemma 2.3 in \cite{biau2015lectures}, since $X_1, \hdots, X_n$ are i.i.d., if $K/n$  tends to zero as $n \to \infty$, we have
\begin{align}
    \mathds{P}[\| X_c - X_{(K)}(X_c) \| > \varepsilon] \to 0.
\end{align}
Thus, for all $n$ large enough, 
\begin{align}
 \mathds{P}[ X_c \in B(x, \alpha)] - 2 \eta   \leq \mathds{P}[ Z \in B(x, \alpha ) ]
\end{align}
and
\begin{align}
   \mathds{P}[ Z \in B(x, \alpha)] 
   & \leq 2 \eta + \mathds{P}[ X_c \in B(x, \alpha)].
\end{align}
 Finally, for all $\eta >0$, for all $n$ large enough,  we obtain 
 \begin{align}
 \mathds{P}[ X_c \in B(x, \alpha)] - 2 \eta   \leq \mathds{P}[ Z \in B(x, \alpha ) ] \leq 2 \eta + \mathds{P}[ X_c \in B(x, \alpha)],
 \end{align}
 which proves that 
 \begin{align}
     \mathds{P}[ Z \in B(x, \alpha ) ] \to \mathds{P}[ X_c \in B(x, \alpha)].
 \end{align}
Therefore, by the Monotone convergence theorem, for all Borel sets $B \subset \mathds{R}^d$, 
 \begin{align}
     \mathds{P}[ Z \in B] \to \mathds{P}[ X_c \in B].
 \end{align}
\end{proof}

\subsection{Proof of \Cref{th-densite}}

\begin{proof}[Proof of \Cref{th-densite}]
\label{proof_th-densite}
 We consider a single SMOTE iteration. Recall that the central point $X_c$ (see \Cref{alg:smote}) is fixed, and thus denoted by $x_c$. 

The random variables $X_{(1)}(x_c), \hdots, X_{(n-1)}(x_c)$ denote a reordering of the initial observations $X_1,X_2, \hdots, X_n$ such that 
\begin{equation*}
    ||X_{(1)}(x_c) - x_c|| \le ||X_{(2)}(x_c) - x_c|| \le \hdots \le ||X_{(n-1)}(x_c) - x_c||. 
\end{equation*}
For clarity, we remove the explicit dependence on $x_c$. Recall that SMOTE builds a linear interpolation between $x_c$ and one of its $K$ nearest neighbors chosen uniformly. Then the newly generated point $Z$ satisfies
\begin{equation}
    Z = (1-W)x_c + W \sum_{k=1}^{K} X_{(k)} \mathds{1}_{\{I=k\}},
\end{equation}
where $W$ is a uniform random variable over $[0,1]$, independent of $I, X_1, \hdots, X_n$, with $I$ distributed as  $\mathcal{U}(\{1,\hdots, K\})$. 

From now, consider that the $k$-th nearest neighbor of $x_c$, $X_{(k)}(x_c),$ has been chosen (that is $I=k$). Then $Z$ satisfies
\begin{align}
    Z & = (1-W)x_c + WX_{(k)}\\
    & = x_c -Wx_c + WX_{(k)},
\end{align}
which implies 
\begin{align}
Z - x_c = W(X_{(k)}-x_c).
\end{align}

Let $f_{Z-x_c}, f_W$ and $f_{X_{(k)}-x_c}$ be respectively the density functions of $Z-x_c$, $W$ and $X_{(k)}-x_c$. 
Let $z,z_1, z_2 \in \mathds{R}^d$. Recall that $z \leq z_1$ means that each component of $z$ is lower than the corresponding component of $z_1$. Since $W$ and $X_{(k)}-x_c$ are independent, we have,
\begin{align}
    \mathds{P}(z_1 \leq Z-x_c \leq z_2) &= \int_{w \in \mathds{R}}\int_{x \in \mathds{R}^d} f_{W,X_{(k)}-x_c}(w,x)\mathds{1}_{\left\{z_1\leq wx\leq z_2 \right\}}\mathrm{d}w\mathrm{d}x \\
    &=\int_{w \in \mathds{R}}\int_{x \in \mathds{R}^d} f_W(w) f_{X_{(k)}-x_c}(x) \mathds{1}_{\left\{z_1\leq wx\leq z_2 \right\}}\mathrm{d}w\mathrm{d}x \\
    &= \int_{w \in \mathds{R}} f_W(w) \left(\int_{x \in \mathds{R}^d} f_{X_{(k)}-x_c}(x) \mathds{1}_{\left\{z_1\leq wx\leq z_2 \right\}}\mathrm{d}x \right)\mathrm{d}w.
\end{align}
Besides, let $u=wx$. Then $x=(\frac{u_1}{w}, \hdots, \frac{u_d}{w})^{T}$.The Jacobian of such transformation equals:
\begin{align}
    &\begin{vmatrix}
        \frac{\partial x_1}{\partial u_1} & \hdots &  \frac{\partial x_1}{\partial u_d}\\
        \vdots & \ddots & \vdots \\
         \frac{\partial x_d}{\partial u_1} &  \hdots  &\frac{\partial x_d}{\partial u_d} 
    \end{vmatrix} =
    \begin{vmatrix}
        \frac{1}{w} &  &  0\\
         & \ddots \\
         0 &  \hdots  &\frac{1}{w} 
    \end{vmatrix} 
    = \frac{1}{w^d}
\end{align}
Therefore, we have  $x = u/w$ and $\mathrm{d}x = \mathrm{d}u / w^d$, which leads to 
\begin{align}
    &\mathds{P}(z_1 \leq Z-x_c \leq z_2) \\ 
    & =\int_{w \in \mathds{R}} \frac{1}{w^d} f_W(w) \left(\int_{u \in \mathds{R}^d} f_{X_{(k)}-x_c}\left(\frac{u}{w}\right) \mathds{1}_{\left\{z_1\leq u\leq z_2 \right\}}\mathrm{d}u \right)\mathrm{d}w.
\end{align}
Note that a random variable $Z'$ with density function 
\begin{align}
    f_{Z'}(z') = \int_{w \in \mathds{R}} \frac{1}{w^d} f_W(w)f_{X_{(k)}-x_c}\left(\frac{z'}{w}\right) \mathrm{d}w
\end{align}
satisfies, for all $z_1, z_2 \in \mathds{R}^d$, 
\begin{align}
    \mathds{P}(z_1 \leq Z-x_c \leq z_2) =\int_{w \in \mathds{R}} \frac{1}{w^d} f_W(w) \left(\int_{u \in \mathds{R}^d} f_{X_{(k)}-x_c}\left(\frac{u}{w}\right) \mathds{1}_{\left\{z_1\leq u\leq z_2 \right\}}\mathrm{d}u \right)\mathrm{d}w.
\end{align}
Therefore, the variable $Z-x_c$ admits the following density 
\begin{align}
    f_{Z-x_c}(z'|X_c=x_c,I=k) = \int_{w \in \mathds{R}} \frac{1}{w^d} f_W(w)f_{X_{(k)}-x_c}\left(\frac{z'}{w}\right) \mathrm{d}w.
\end{align}

Since $W$ follows a uniform distribution on $[0,1]$, we have

\begin{align}
\label{res-intermediaire-1}
    f_{Z-x_c}(z'|X_c=x_c,I=k) = \int_0^1\frac{1}{w^d}f_{X_{(k)}-x_c}\Big(\frac{z'}{w}\Big)\mathrm{d}w.
\end{align}
The density $f_{X_{(k)} - x_c}$ of the $k$-th nearest neighbor of $x_c$ can be computed exactly \citep[see, Lemma 6.1 in][]{berrett_2017}, that is
\begin{align}
\label{eq-berret-used}
    f_{X_{(k)}-x_c}(u) & = (n-1)\binom{n-2}{k-1}f_X(x_c+u)\big[ \mu_X \left( B(x_c,||u||) \right)\big]^{k-1} \nonumber \\
    & \quad \times \big[1- \mu_X \left( B(x_c,||u||) \right)\big]^{n-k-1},
\end{align}
where 
\begin{align}
     \mu_X \left( B(x_c,||u||) \right) & = \int_{B(x_c, \|u\|)} f_X(x) \mathrm{d}x. 
\end{align}
We recall that $B(x_c, \|u\|)$ is the ball centered on $x_c$ and of radius $\|u\|$. Hence we have 
\begin{equation}
\label{res-berret-used-final}
    f_{X_{(k)}-x_c}(u)= (n-1)\binom{n-2}{k-1}f_X(x_c+u)\mu_X \left( B(x_c,||u||) \right)^{k-1}\left[1-\mu_X \left( B(x_c,||u||) \right)\right]^{n-k-1}.
\end{equation}
Since $Z-x_c$ is a translation of the random variable $Z$, we have
\begin{align}
f_{Z}(z|X_c=x_c,I=k) = f_{Z-x_c}(z-x_c|X_c=x_c,I=k).
\end{align}
Injecting \Cref{res-berret-used-final} in \Cref{res-intermediaire-1}, we obtain 

\begin{align}
    & f_{Z}(z|X_c=x_c,I=k)  \\
    &=f_{Z-x_c}(z-x_c|X_c=x_c,I=k) \\
    & = \int_0^1\frac{1}{w^d}f_{X_{(k)}-x_c}\Big(\frac{z-x_c}{w}\Big)\mathrm{d}w\\
    & =(n-1)\binom{n-2}{k-1}\int_0^1\frac{1}{w^d}f_X\left(x_c+\frac{z-x_c}{w}\right)\mu_X \left( B\left(x_c,\frac{||z-x_c||}{w}\right) \right)^{k-1} \\ & \qquad \times \left[1-\mu_X \left( B\left(x_c,\frac{||z-x_c||}{w}\right) \right)\right]^{n-k-1}\mathrm{d}w
\label{res-intermediaire-2}
\end{align}
Recall that in SMOTE, $k$ is chosen at random in $\{1, \hdots, K\}$ through the uniform random variable $I$. So far, we have considered $I$ fixed. Taking the expectation with respect to $I$, we have 
\begin{align}
\label{expectation-k-1}
   &  f_{Z}(z|X_c=x_c) \\
    &= \sum_{k=1}^K f_{Z}(z|X_c=x_c, I=k) \mathds{P}[I = k] \\
    &= \frac{1}{K} \sum_{k=1}^{K}\int_0^1\frac{1}{w^d}f_{X_{(k)}-x_c}\Big(\frac{z-x_c}{w}\Big)\mathrm{d}w \\
     &= \frac{1}{K} \sum_{k=1}^{K} (n-1)\binom{n-2}{k-1}\int_0^1\frac{1}{w^d}f_X\left(x_c+\frac{z-x_c}{w}\right)\mu_X \left( B\left(x_c,\frac{||z-x_c||}{w}\right) \right)^{k-1} \\ & \qquad \times[1-\mu_X \left( B\left(x_c,\frac{||z-x_c||}{w}\right) \right)]^{n-k-1}\mathrm{d}w \\
    &= \frac{(n-1)}{K} \int_0^1\frac{1}{w^d}f_X\left(x_c+\frac{z-x_c}{w}\right)\sum_{k=1}^{K} \binom{n-2}{k-1}\mu_X \left( B\left(x_c,\frac{||z-x_c||}{w}\right) \right)^{k-1} \\ & \qquad \times [1-\mu_X \left( B\left(x_c,\frac{||z-x_c||}{w}\right) \right)]^{n-k-1}\mathrm{d}w \\
    & = \frac{(n-1)}{K} \int_0^1\frac{1}{w^d}f_X\left(x_c+\frac{z-x_c}{w}\right) \sum_{k=0}^{K-1} \binom{n-2}{k}\mu_X \left( B\left(x_c,\frac{||z-x_c||}{w}\right) \right)^{k} \\ & \qquad \times \left[1-\mu_X \left( B\left(x_c,\frac{||z-x_c||}{w}\right) \right)\right]^{n-k-2} \mathrm{d}w.
\end{align}
Note that the sum can be expressed as the cumulative distribution function of a Binomial distribution parameterized by $n-2$ and $\mu_X(B(x_c, \|z-x_c\|/w))$, so that

\begin{align}
     & \sum_{k=0}^{K-1} \binom{n-2}{k}\mu_X \left( B\left(x_c,\frac{||z-x_c||}{w}\right) \right)^{k}\left[1-\mu_X \left( B\left(x_c,\frac{||z-x_c||}{w}\right) \right)\right]^{n-k-2} \\
    = & (n-K-1)\binom{n-2}{K-1} \mathcal{B}\left(n-K-1,K ;1-\mu_X \left( B\left(x_c,\frac{||z-x_c||}{w}\right) \right)\right),
    \label{expectation-k-2}
\end{align}
(see Technical Lemma~\ref{lem_binomial} for details).  
We inject \Cref{expectation-k-2} in \Cref{expectation-k-1} 
\begin{align}
    \nonumber f_{Z}(z|X_c=x_c)
    &=(n-K-1)\binom{n-1}{K}\int_0^1\frac{1}{w^d}f_X\left(x_c+\frac{z-x_c}{w}\right) \\
     & \qquad \times \mathcal{B}\left(n-K-1,K ;1-\mu_X \left( B\left(x_c,\frac{||z-x_c||}{w}\right) \right)\right)\mathrm{d}w.
\label{eq_beta}
\end{align}

We know that
\begin{align*}
    f_{Z}(z) = \int_{x_c \in \mathcal{X}} f_{Z}(z|X_c=x_c)f_X(x_c)\mathrm{d}x_c.
\end{align*}
Combining this remark with the result of \Cref{eq_beta} we get
\begin{align}
    \nonumber f_{Z}(z)&=(n-K-1)\binom{n-1}{K} \int_{x_c \in \mathcal{X}} \int_0^1\frac{1}{w^d}f_X\left(x_c+\frac{z-x_c}{w}\right) \\ & \qquad \times \mathcal{B}\left(n-K-1,K ;1-\mu_X \left( B\left(x_c,\frac{||z-x_c||}{w}\right) \right)\right) f_X(x_c) \mathrm{d}w \mathrm{d}x_c.
\end{align}

\paragraph{Link with Elreedy's formula} According to the Elreedy formula
\begin{align}
    f_Z(z|X_c=x_c) & = (n-K-1)\binom{n-1}{K} \int_{r=\|z-x_c\|}^{\infty} f_X\left(x_c+ \frac{(z-x_c)r}{\|z-x_c\|}\right) \frac{r^{d-2}}{\|z-x_c\|^{d-1}} \nonumber \\
     &  \quad \times \mathcal{B}\left(n-K-1,K ;1-\mu_X \left( B\left(x_c,r\right) \right)\right) \mathrm{d}r.
\end{align}

Now, let $r =\|z-x_c\|/w$ so that $\mathrm{d}r = -\|z-x_c\| \mathrm{d}w /w^2$. Thus, 
\begin{align}
    &\nonumber f_Z(z|X_c=x_c) \\ 
    &= (n-K-1)\binom{n-1}{K}  \int_{0}^{1} f_X\left(x_c+ \frac{z-x_c}{w}\right) \frac{1}{w^{d-2}}\frac{1}{\|z-x_c\|}\\ 
    & \qquad \times \mathcal{B}\left(n-K-1,K ;1-\mu_X \left( B\left(x_c,\frac{z-x_c}{w}\right) \right)\right)\frac{\|z-x_c\|}{w^2}\mathrm{d}w \\
    &= (n-K-1)\binom{n-1}{K}  \int_{0}^{1} \frac{1}{w^{d}} f_X\left(x_c+ \frac{z-x_c}{w}\right) \nonumber \\
    & \qquad \times \mathcal{B}\left(n-K-1,K ;1-\mu_X \left( B\left(x_c,\frac{z-x_c}{w}\right) \right)\right)\mathrm{d}w.
\end{align}
\end{proof}

\subsection{Proof of \Cref{thm_asymptot_SMOTE_k}}
\begin{proof}[Proof of \Cref{thm_asymptot_SMOTE_k}]
\label{proof_thm_asymptot_SMOTE_k}
Let $x_c \in \mathcal{X}$ be a central point in a SMOTE iteration. From \Cref{th-densite}, we have,

\begin{align}
    \nonumber & f_{Z}(z|X_c=x_c) \\
    &=(n-K-1)\binom{n-1}{K}\int_0^1\frac{1}{w^d}f_X\left(x_c+\frac{z-x_c}{w}\right)  \nonumber\\
    & \quad \times \mathcal{B}\left(n-K-1,K ;1-\mu_X \left( B\left(x_c,\frac{||z-x_c||}{w}\right) \right)\right)\mathrm{d}w \\
    \nonumber&=(n-K-1)\binom{n-1}{K}\int_0^1\frac{1}{w^d}f_X\left(x_c+\frac{z-x_c}{w}\right) \mathds{1}_{\{x_c+\frac{z-x_c}{w} \in \mathcal{X}\}} \\
    & \qquad  \times \mathcal{B}\left(n-K-1,K ;1-\mu_X \left( B\left(x_c,\frac{||z-x_c||}{w}\right) \right)\right)\mathrm{d}w.
\end{align}
Let $R \in \mathds{R}$ such that $\mathcal{X} \subset \mathcal{B}(0,R)$. For all $u = x_c + \frac{z-x_c}{w}$, we have  
\begin{equation}
    w = \frac{||z-x_c||}{||u-x_c||}.
\end{equation}
If $u \in \mathcal{X}$, then $u \in \mathcal{B}(0,R)$. Besides, since $x_c \in \mathcal{X} \subset B(0,R)$, we have $||u-x_c||<2R$ and
\begin{equation}
    w > \frac{||z-x_c||}{2R}.
\end{equation} 
Consequently, 
\begin{align}
    \mathds{1}_{\left\{x_c + \frac{z-x_c}{w} \in \mathcal{X}\right\}} \leq \mathds{1}_{\left\{w > \frac{||z-x_c||}{2R}\right\}}.
\end{align}
So finally
\begin{align}
    \mathds{1}_{\left\{x_c + \frac{z-x_c}{w} \in \mathcal{X}\right\}} = \mathds{1}_{\left\{x_c + \frac{z-x_c}{w} \in \mathcal{X}\right\}} \mathds{1}_{\left\{w > \frac{||z-x_c||}{2R}\right\}}.
\end{align}
    
Hence,
\begin{align}
    \nonumber &f_{Z}(z|X_c=x_c) \\
    &=(n-K-1)\binom{n-1}{K}\int_0^1\frac{1}{w^d}f_X\left(x_c+\frac{z-x_c}{w}\right) \mathds{1}_{\left\{x_c + \frac{z-x_c}{w} \in \mathcal{X}\right\}} \mathds{1}_{\left\{w > \frac{||z-x_c||}{2R}\right\}} \\ 
    &\qquad  \times \mathcal{B}\left(n-K-1,K ;1-\mu_X \left( B\left(x_c,\frac{||z-x_c||}{w}\right) \right)\right)\mathrm{d}w \\
    &=(n-K-1)\binom{n-1}{K}\int_{\frac{||z-x_c||}{2R}}^1\frac{1}{w^d}f_X\left(x_c+\frac{z-x_c}{w}\right)  \nonumber \\ 
    &\qquad \times \mathcal{B}\left(n-K-1,K ;1-\mu_X \left( B\left(x_c,\frac{||z-x_c||}{w}\right) \right)\right)\mathrm{d}w.
\end{align}

Now, let $0 < \alpha \leq 2R$ 
and $z \in \mathds{R}^d$ such that $||z-x_c||>\alpha$. In such a case, $w > \frac{\alpha}{2R}$ and:

\begin{align}
\label{ineg-convergence-1}
    & f_{Z}(z|X_c=x_c)\\
    &= (n-K-1)\binom{n-1}{K}\int_{\frac{\alpha}{2R}}^1\frac{1}{w^d}f_X\left(x_c+\frac{z-x_c}{w}\right) \nonumber \\
    & \qquad \times \mathcal{B}\left(n-K-1,K ;1-\mu_X \left( B\left(x_c,\frac{||z-x_c||}{w}\right) \right)\right)\mathrm{d}w \\
    & \leq (n-K-1)\binom{n-1}{K}\int_{\frac{\alpha}{2R}}^1\frac{1}{w^d}f_X\left(x_c+\frac{z-x_c}{w}\right) \nonumber \\
    & \qquad \times \mathcal{B}\left(n-K-1,K ;1-\mu_X \left( B\left(x_c, \alpha \right) \right)\right)\mathrm{d}w.
\end{align}
Let $\mu \in [0,1]$ and $S_n$ be a binomial random variable of parameters $(n-1, \mu)$. For all $K$,
\begin{align}
  \mathds{P}[S_n \leq K ] & = (n-K-1)\binom{n-1}{K}  \mathcal{B}\left(n-K-1,K ;1-\mu \right).
\end{align}
According to Hoeffding's inequality, we have, for all $K \leq (n-1) \mu $,
\begin{align}
   \mathds{P}[S_n \leq K ] & \leq \exp \left( -2 (n-1) \left( \mu  - \frac{K}{n-1}\right)^2\right). 
\end{align}
Thus, for all $z \notin B(x_c, \alpha)$, for all $K \leq (n-1) \mu_X \left( B\left(x_c, \alpha \right) \right)$,
\begin{align}
& f_{Z}(z|X_c=x_c)\\
    & \leq \exp \left( -2 (n-1) \left( \mu_X \left( B\left(x_c, \alpha \right) \right)  - \frac{K}{n-1}\right)^2\right) \int_{\frac{\alpha}{2R}}^1\frac{1}{w^d}f_X\left(x_c+\frac{z-x_c}{w}\right)  \mathrm{d}w \\ 
    & \leq C_2 \exp \left( -2 (n-1) \left( \mu_X \left( B\left(x_c, \alpha \right) \right)  - \frac{K}{n-1}\right)^2\right) \int_{\frac{\alpha}{2R}}^1\frac{1}{w^d}   \mathrm{d}w \\
    & \leq C_2 \eta(\alpha,R) \exp \left( -2 (n-1) \left( \mu_X \left( B\left(x_c, \alpha \right) \right)  - \frac{K}{n-1}\right)^2\right),
    \label{eq-th34-noC2}
\end{align}
with 
$$
    \eta(\alpha,R) = \left\{
        \begin{array}{ll}
             \ln\left(\frac{2R}{\alpha}\right) & \text{if $d=1$} \\
             \frac{1}{d-1} \left( \left( \frac{2R}{\alpha} \right)^{d-1} - 1 \right)   & \text{otherwise}
        \end{array}.
    \right.
$$
Letting 
\begin{align}
    \epsilon(n,\alpha,K,x_c) & = C_2 \eta(\alpha,R) \exp \left( -2 (n-1) \left( \mu_X \left( B\left(x_c, \alpha \right) \right)  - \frac{K}{n-1}\right)^2\right),
\end{align}
we have, for all $\alpha \in (0,2R)$, for all $K \leq (n-1) \mu_X \left( B\left(x_c, \alpha \right) \right)$, 
\begin{align}
    \mathds{P} \left( |Z - X_c| \geq \alpha | X_c = x_c \right) & =
   \int_{z \notin \mathcal{B}(x_c,\alpha), z\in \mathcal{X}} f_{Z}(z|X_c=x_c)\mathrm{d}z \\
   &\leq \int_{z \notin \mathcal{B}(x_c,\alpha), z\in \mathcal{X}} \varepsilon(n,\alpha,K,x_c)\mathrm{d}z \\
   &= \varepsilon(n,\alpha,K,x_c) \int_{z \notin \mathcal{B}(x_c,\alpha), z\in \mathcal{X}} \mathrm{d}z \\
   & \leq  c_d R^d \varepsilon(n,\alpha,K,x_c), 
\end{align}
as $\mathcal{X} \subset B(0,R)$. 
Since $x_c \in \mathcal{X}$, by definition of the support, we know that 
for all $\rho>0$, $\mu_X(B(x_c, \rho)) >0$. Thus, $\mu_X \left( B\left(x_c,\alpha \right) \right)>0$. Consequently, $\varepsilon(n,\alpha,K,x_c)$ tends to zero, as $K/n$ tends to zero. 
\end{proof}

\subsection{Proof of \Cref{proposition_distance_caracteristique}}

\label{proof_proposition_distance_caracteristique}

We adapt the proof of Theorem 2.1 and Theorem 2.4 in \citet{biau2015lectures} to the case where $X$ belongs to $B(0,R)$. We prove the following result. 

\begin{lemma}
\label{lemma_1proof}
Let $X$ takes values in $B(0,R)$. For all $d \geq 2$, 
\begin{align}
\mathds{E}[\|X_{(1)}(X) - X\|_2^2] \leq 36R^2 \left( \frac{k}{n+1}\right)^{2/d},
\end{align}
where $X_{(1)}(X)$ is the nearest neighbor of $X$ among $X_1, \hdots, X_n$. 
\end{lemma}

\begin{proof}[Proof of \Cref{lemma_1proof}]
    
Let us denote by $X_{(i,1)}$ the nearest neighbor of $X_i$ among $X_1, \hdots, X_{i-1}, X_{i+1}, \hdots, X_{n+1}$. By symmetry, we have
\begin{align}
\mathds{E}[\|X_{(1)}(X) - X\|_2^2]  & = \frac{1}{n+1} \sum_{i=1}^{n+1} \mathds{E} \| X_{(i,1)} - X_i \|_2^2.   
\end{align}
Let $R_i = \| X_{(i,1)} - X_i \|_2$ and $B_i = \{ x \in \mathds{R}^d : \|x - X_i\| < R_i/2 \}$. By construction,  $B_i$ are disjoint. Since $R_i \leq 2R$, we have
\begin{align}
    \cup_{i=1}^{n+1} B_i \subset B(0,3R),
\end{align}
which implies, 
\begin{align}
    \mu \left( \cup_{i=1}^{n+1} B_i \right) \leq (3R)^d c_d.
\end{align}
Thus, we have
\begin{align}
    \sum_{i=1}^{n+1} c_d \left( \frac{R_i}{2} \right)^d \leq  (3R)^d c_d.
\end{align}
Besides, for all $d \geq 2$, we have
\begin{align}
    \left( \frac{1}{n+1} \sum_{i=1}^{n+1} R_i^2 \right)^{d/2} & \leq \frac{1}{n+1} \sum_{i=1}^{n+1} R_i^d,
\end{align}
which leads to 
\begin{align}
\mathds{E}[\|X_{(1)}(X) - X\|_2^2] & =    \frac{1}{n+1} \sum_{i=1}^{n+1} \mathds{E} \| X_{(i,1)} - X_i \|_2^2 \\
& =  \mathds{E} \left[ \frac{1}{n+1} \sum_{i=1}^{n+1}   R_i^2 \right] \\
& \leq \left( \frac{(6R)^d}{n+1} \right)^{2/d}\\
& \leq 36 R^2 \left( \frac{1}{n+1} \right)^{2/d}.
\end{align}
\end{proof}

\begin{lemma}
    \label{lemma_proof2}
Let $X$ takes values in $B(0,R)$. For all $d \geq 2$, 
\begin{align}
\mathds{E}[\|X_{(k)}(X) - X\|_2^2] \leq (2^{1+2/d}) 36R^2 \left( \frac{k}{n}\right)^{2/d},
\end{align}
where $X_{(k)}(X)$ is the nearest neighbor of $X$ among $X_1, \hdots, X_n$. 
\end{lemma}

\begin{proof}[Proof of \Cref{lemma_proof2}]
Set $d \geq 2$. Recall that $\mathds{E}[\|X_{(k)}(X) - X\|_2^2] \leq 4R^2$. Besides, for all $k > n/2$, we have
\begin{align}
 (2^{1+2/d}) 36R^2 \left( \frac{k}{n}\right)^{2/d} & > (2^{1+2/d}) 36R^2 \left( \frac{1}{2}\right)^{2/d} \\
 &  > 72 R^2 \\
 & > \mathds{E}[\|X_{(k)}(X) - X\|_2^2]. 
\end{align}
Thus, the result is trivial for $k > n/2$. Set $k \leq n/2$. Now, following the argument of Theorem 2.4 in \citet{biau2015lectures}, let us partition the set $\{X_1, \hdots, X_n\}$ into $2k$ sets of sizes $n_1, \hdots, n_{2k}$ with 
\begin{align}
    \sum_{j=1}^{2k} n_j = n \quad \textrm{and} \quad \left\lfloor \frac{n}{2k} \right\rfloor \leq n_j \leq \left\lfloor \frac{n}{2k} \right\rfloor +1.
\end{align}
Let $X_{(1)}^\star(j)$ be the nearest neighbor of $X$ among all $X_i$ in the $j$th group. Note that 
\begin{align}
    \| X_{(k)}(X) - X \|^2 \leq \frac{1}{k} \sum_{j=1}^{2k} \| X_{(1)}^\star(j) - X \|^2,
\end{align}
since at least $k$ of these nearest neighbors have values larger than $\| X_{(k)}(X) - X \|^2$. By \Cref{lemma_1proof}, we have
\begin{align}
 \| X_{(k)}(X) - X \|^2 & \leq \frac{1}{k} \sum_{j=1}^{2k} 36 R^2 \left( \frac{1}{n_j +1}\right)^{2/d} \\
 & \leq \frac{1}{k} \sum_{j=1}^{2k} 36 R^2 \left( \frac{2k}{n}\right)^{2/d} \\
 & \leq 2^{1+2/d} \times 36R^2  \left( \frac{k}{n}\right)^{2/d}.
\end{align}
\end{proof}

\begin{proof}[Proof of \Cref{proposition_distance_caracteristique}]

Let $d \geq 2$. By Markov's inequality, for all $\varepsilon>0$, we have
\begin{align}
\mathds{P} \left[ \|X_{(k)}(X) - X\|_2 > \varepsilon \right] & \leq \frac{\mathds{E}[\|X_{(k)}(X) - X\|_2^2]}{\varepsilon^2}.
\end{align}
Let $\gamma \in (0,1/d)$ and $\varepsilon = 12R (k/n)^{\gamma}$, we have
\begin{align}
\mathds{P} \left[ \|X_{(k)}(X) - X\|_2 > 12R (k/n)^{\gamma} \right] & \leq \left( \frac{k}{n} \right)^{2/d - 2 \gamma}.    
\end{align}
Noticing that, by construction of a SMOTE observation $Z_{K,n}$, we have
\begin{align}
    \| Z_{K,n} - X \|_2 \leq \|X_{(K)}(X) - X \|_2.
\end{align}
Thus, 
\begin{align}
 \mathds{P} \left[ \| Z_{K,n} - X \|_2 > 12R (k/n)^{\gamma} \right] & \leq \mathds{P} \left[ \|X_{(K)}(X) - X \|_2^2 > 12R (k/n)^{1/d} \right]\\
 & \leq \left( \frac{k}{n} \right)^{2/d - 2 \gamma}.
\end{align}
\end{proof}

\subsection{Proof of \Cref{prop:boundaries}}
\begin{proof}[Proof of \Cref{prop:boundaries}]
\label{proof_prop:boundaries}
Let $\varepsilon >0$ and $z \in B(0,R)$ such that $\|z\| \geq R - \varepsilon$. Let $A_{\varepsilon} = \{ x \in B(0,R), \langle x-z, z \rangle \leq 0 \}$. 
Let $0<\alpha<2R$ and $\tilde{A}_{\alpha, \varepsilon} = A_{\varepsilon} \cap \left\{ x, \|z-x\| \geq \alpha \right \}$. An illustration is displayed in Figure~\ref{fig:illustration_boundaries}.
\begin{figure}[h!]
\vskip 0.2in
\begin{center}
\centerline{\includegraphics[width=0.2\columnwidth]{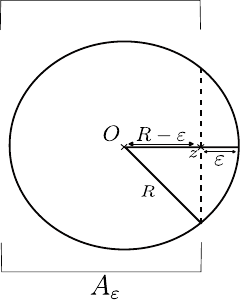}}
\caption{Illustration of \Cref{prop:boundaries}.}
\label{fig:illustration_boundaries}
\end{center}
\vskip -0.2in
\end{figure}

We have 
\begin{align}
    f_Z(z) &= \int_{x_c \in \tilde{A}_{\alpha, \varepsilon}} f_Z(z|X_c=x_c)f_X(x_c)\mathrm{d}x_c + \int_{x_c \in \tilde{A}_{\alpha, \varepsilon}^c}f_Z(z|X_c=x_c)f_X(x_c)\mathrm{d}x_c
\end{align}

\paragraph{First term} 

Let $x_c \in \tilde{A}_{\alpha, \varepsilon}$. 
In order to have $x_c + \frac{z-x_c}{w} = z + \left(-1+\frac{1}{w}\right)(z-x_c) \in B(0,R)$, it is necessary that 
\begin{align}
    &\left(-1+\frac{1}{w}\right)\|z-x_c\| \leq \sqrt{2\varepsilon R}
\end{align}
which leads to 
\begin{align}
    w \geq \frac{1}{1 +  \frac{\sqrt{2\varepsilon R}}{\|z-x_c\|}} \label{eq_proof_3}
\end{align}
Since $x_c \in \tilde{A}_{\alpha, \varepsilon}$, we have $\|x_c - z \| \geq \alpha$. Thus, according to inequality \eqref{eq_proof_3},  $x_c + \frac{z-x_c}{w} \in B(0,R)$ implies 
\begin{align}
    w 
    \geq \frac{1}{1 +  \frac{\sqrt{2\varepsilon R}}{\alpha}}.
\end{align}

Recall  that $x_c + \frac{z-x_c}{w} \in \mathcal{X}$. 
Consequently, according to \Cref{th-densite}, for all $x_c \in \tilde{A}_{\alpha, \varepsilon}$, 
\begin{align}
    & f_Z(z|X_c=x_c) \\
    & = (n-K-1)\binom{n-1}{K}\int_0^1\frac{1}{w^d}f_X\left(x_c+\frac{z-x_c}{w}\right) \nonumber \\
    & \qquad \times \mathcal{B}\left(n-K-1,K ;1-\mu_X \left( B\left(x_c,\frac{||z-x_c||}{w}\right) \right)\right)\mathrm{d}w \\
    & \leq C_2 (n-K-1)\binom{n-1}{K}\int_{\frac{1}{1 + \frac{\sqrt{2\varepsilon R}}{\alpha}}}^1\frac{1}{w^d}  \nonumber \\
    & \qquad \times \mathcal{B}\left(n-K-1,K ;1-\mu_X \left( B\left(x_c,\frac{||z-x_c||}{w}\right) \right)\right)\mathrm{d}w.
\end{align}
Besides,  
\begin{align}
 & (n-K-1) \binom{n-1}{K} \mathcal{B}\left(n-K-1,K ;1-\mu_X \left( B\left(x_c,\frac{||z-x_c||}{w}\right) \right)\right)  \\
=  &  \left( \frac{n-1}{K} \right) (n-K-1) \binom{n-2}{K-1} \mathcal{B}\left(n-K-1,K ;1-\mu_X \left( B\left(x_c,\frac{||z-x_c||}{w}\right) \right)\right) \\\leq & ~~ \frac{n-1}{K},
\end{align}
according to Lemma~\ref{lem_binomial}. Thus, 
\begin{align}
f_Z(z|X_c=x_c) 
    & \leq C_2 \left( \frac{n-1}{K} \right)  \int_{\frac{1}{1 + \frac{\sqrt{2\varepsilon R}}{\alpha}}}^1\frac{1}{w^d} \mathrm{d}w \\
    &\leq C_2 \left( \frac{n-1}{K} \right)  \eta(\alpha,R) ,
\end{align}
with 
$$
    \eta(\alpha,R) = \left\{
        \begin{array}{ll}
             \ln\left(1 +  \frac{\sqrt{2\varepsilon R}}{\alpha}\right) & \text{if $d=1$} \\
             \frac{1}{d-1} \left( \left( 1 +  \frac{\sqrt{2\varepsilon R}}{\alpha} \right)^{d-1} - 1 \right)   & \text{otherwise}
        \end{array}.
    \right.
$$

\paragraph{Second term}

According to \Cref{th-densite}, we have 
\begin{align}
f_{Z}(z|X_c=x_c)&=(n-K-1)\binom{n-1}{K}\int_0^1\frac{1}{w^d}f_X\left(x_c+\frac{z-x_c}{w}\right) \nonumber \\ & \qquad \times \mathcal{B}\left(n-K-1,K ;1-\mu_X \left( B\left(x_c,\frac{||z-x_c||}{w}\right) \right)\right)\mathrm{d}w \\
& \leq \left( \frac{n-1}{K} \right) \int_0^1\frac{1}{w^d}f_X\left(x_c+\frac{z-x_c}{w}\right) \mathrm{d}w
\label{eq_upp_boundcond_density}
\end{align}
Since $\mathcal{X} \subset B(0,R)$, all points $x,z \in \mathcal{X}$ satisfy $\|x-z\| \leq 2R$. 

Consequently, if $\|z-x_c\|/w > 2R$, 
\begin{align}
    x_c +  \frac{\|z-x_c\|}{w} \notin \mathcal{X}.
\end{align}
Hence, for all $w \leq \|z-x_c\|/2R$, 
\begin{align}
f_X\left(x_c+\frac{z-x_c}{w}\right)= 0.
\end{align}
Plugging this equality into \eqref{eq_upp_boundcond_density}, we have
\begin{align}
& f_{Z}(z|X_c=x_c)\\
& \leq \left( \frac{n-1}{K} \right) \int_{\|z-x_c\|/2R}^1 \frac{1}{w^d}f_X\left(x_c+\frac{z-x_c}{w}\right) \mathrm{d}w\\
& \leq C_2 \left( \frac{n-1}{K} \right) \int_{\|z-x_c\|/2R}^1 \frac{1}{w^d} \mathrm{d}w\\
& \leq C_2 \left( \frac{n-1}{K} \right) \left[ - \frac{1}{d-1} w^{-d+1} \right]_{\|z-x_c\|/2R}^1 \\
& \leq C_2 \left( \frac{n-1}{K} \right) \frac{(2R)^{d-1}}{d-1}    \frac{1}{\|z-x_c\|^{d-1}}.
\end{align}

Besides, note that, for all $\alpha >0$, we have
\begin{align}
& \int_{B(z, \alpha)} \frac{1}{\|z-x_c\|^{d-1}} f_X(x_c) \textrm{d}x_c \\
& \leq C_2 \int_{B(0, \alpha)} \frac{1}{r^{d-1}} r^{d-1} \sin^{d-2}(\varphi_1) \sin^{d-3}(\varphi_2) \hdots \sin(\varphi_{d-2}) \textrm{d}r \textrm{d} \varphi_1 \hdots \textrm{d}\varphi_{d-2},
\end{align}
where $r, \varphi_1, \hdots, \varphi_{d-2}$ are the spherical coordinates. A direct calculation leads to 
\begin{align}
&\int_{B(z, \alpha)} \frac{1}{\|z-x_c\|^{d-1}} f_X(x_c) \textrm{d}x_c  \nonumber\\ 
& \leq C_2 \int_{0}^{\alpha} \textrm{d} r \int_{S(0, \alpha)} \sin^{d-2}(\varphi_1) \sin^{d-3}(\varphi_2) \hdots \sin(\varphi_{d-2}) \textrm{d} \varphi_1 \hdots \textrm{d}\varphi_{d-2} \\
& \leq  \frac{2 C_2 \pi^{d/2}}{\Gamma(d/2)} \alpha,
\end{align}
as 
\begin{align}
    \int_{S(0, \alpha)} \sin^{d-2}(\varphi_1) \sin^{d-3}(\varphi_2) \hdots \sin(\varphi_{d-2}) \textrm{d} \varphi_1 \hdots \textrm{d}\varphi_{d-2}
\end{align}
is the surface of the $S^{d-1}$ sphere. Finally, for all $z \in \mathcal{X}$, for all $\alpha >0$, and for all $K, N$ such that $1 \leq K \leq N$, we have
\begin{align}
    \int_{B(z, \alpha)} f_Z(z | X_c = x_c) f_X(x_c) \textrm{d} x_c \leq  \frac{2 C_2^2 (2R)^{d-1} \pi^{d/2}}{(d-1) \Gamma(d/2)}\left( \frac{n-1}{K} \right)   \alpha. 
\end{align}

\paragraph{Final result}

Using Figure~\ref{fig:illustration_boundaries} and Pythagore's Theorem, we have  $a^2 \leq \sqrt{2\varepsilon R}$. 
Let $d > 1$ and $\epsilon > 0$. Then we have for all $\alpha$ such that $\alpha > a$.
\begin{align}
    &f_Z(z) \\ &= \int_{x_c \in \tilde{A}_{\alpha, \varepsilon}} f_Z(z|X_c=x_c)f_X(x_c)\mathrm{d}x_c + \int_{x_c \in \tilde{A}_{\alpha, \varepsilon}^c}f_Z(z|X_c=x_c)f_X(x_c)\mathrm{d}x_c \\
    & \leq \frac{C_2}{d-1}\left( \left( 1 +  \frac{\sqrt{2\varepsilon R}}{\alpha} \right)^{d-1} - 1 \right)  \left( \frac{n-1}{K} \right)  +  \frac{2 C_2^2 (2R)^{d-1} \pi^{d/2}}{(d-1) \Gamma(d/2)} \left( \frac{n-1}{K} \right) \alpha \\
    & = \frac{C_2}{d-1} \left( \frac{n-1}{K} \right) \left[\left( \left( 1 +  \frac{\sqrt{2\varepsilon R}}{\alpha} \right)^{d-1} - 1 \right) + \frac{2 C_2 (2R)^{d-1} \pi^{d/2}}{ \Gamma(d/2)} \alpha  \right],
\end{align}
But this inequality is true if $\alpha \geq a$. We know that $(1+x)^{d-1} \leq (2^{d-1} - 1)x + 1$ for $x \in [0,1]$ and $d-1 \geq 0$. Then, for $\alpha$ such that $\frac{\sqrt{2\varepsilon R}}{\alpha} \leq 1$,
\begin{align}
    &f_Z(z) \\
    & \leq \frac{C_2}{d-1} \left( \frac{n-1}{K} \right) \left[\left( \left( (2^{d-1} - 1)\frac{\sqrt{2\varepsilon R}}{\alpha} + 1 \right) - 1 \right) + \frac{2 C_2 (2R)^{d-1} \pi^{d/2}}{ \Gamma(d/2)} \alpha  \right] \\
    & \leq \frac{C_2}{d-1} \left( \frac{n-1}{K} \right) \left[ \left( (2^{d-1} - 1)\frac{\sqrt{2\varepsilon R}}{\alpha} \right)  + \frac{2 C_2 (2R)^{d-1} \pi^{d/2}}{ \Gamma(d/2)} \alpha  \right].
\end{align}
Since $\frac{\sqrt{2\varepsilon R}}{\alpha} \leq 1$, then $\alpha \geq \sqrt{2\varepsilon R} \geq a $. So our initial condition on $\alpha$ to get the upper bound of the second term is still true. 
Now, we choose $\alpha$ such that,
\begin{align}
    & (2^{d-1} - 1)\frac{\sqrt{2\varepsilon R}}{\alpha} \leq \frac{2 C_2 (2R)^{d-1} \pi^{d/2}}{ \Gamma(d/2)} \alpha,
\end{align}
which leads to the following condition
\begin{align}
    \alpha \geq \left( \frac{ \Gamma(d/2) (2^{d-1} - 1)\sqrt{2\varepsilon R} }{2 C_2 (2R)^{d-1} \pi^{d/2}} \right)^{1/2},
\end{align}
assuming that 
\begin{align}
\left( \frac{\varepsilon}{R} \right)^{1/2} \leq \frac{1}{\sqrt{2} d C_2} Vol (B_d(0,1)).
\end{align}
Finally, for
\begin{align}
    \alpha = \left(\frac{ \Gamma(d/2) (2^{d-1} - 1)\sqrt{2\varepsilon R} }{2 C_2 (2R)^{d-1} \pi^{d/2}}\right)^{1/2},
\end{align}
we have,
\begin{align}
    f_Z(z) & \leq \frac{C_2}{d-1} \left( \frac{n-1}{K} \right) \left[  \frac{4 C_2 (2R)^{d-1} \pi^{d/2}}{ \Gamma(d/2)} \alpha  \right] \\
    & \leq \frac{C_2}{d-1} \left( \frac{n-1}{K} \right) \left[  \frac{4 C_2 (2R)^{d-1} \pi^{d/2}}{ \Gamma(d/2)}  \left( \frac{ \Gamma(d/2) (2^{d-1} - 1)\sqrt{2\varepsilon R} }{2 C_2 (2R)^{d-1} \pi^{d/2}} \right)^{1/2} \right] \\
    & = 2^{d+2}  \left( \frac{n-1}{K} \right) \left( \frac{  C_2^{3} Vol(B_d(0,1))}{d } \right)^{1/2} \left( \frac{\varepsilon}{R} \right)^{1/4}.
\end{align}

\end{proof}

\subsection{Proof of \Cref{lemma:balanced_accuracy}}

\begin{proof}[Proof of \Cref{lemma:balanced_accuracy}]

    Let $\hat{f}$ be a trained classifier. Let $\mathcal{D}_{test} = \{(X_i,Y_i)\}_{i=1}^{N}$ be a test set for $\hat{f}$, with $n \in \mathds{N}$ and $\sum_{i=1}^{N} Y_i = n$. We have $Y \in \{0,1\}$.

    The balanced accuracy of $\hat{f}$, denoted by $\text{Acc}_{Balanced}(\hat{f})$, is defined as the average of the recall and the specificity of $\hat{f}$ on $\mathcal{D}_{test}$. Let us consider the weighted accuracy of $\hat{f}$, denoted by $Acc_{Weighted}(\hat{f})$, which is nothing but the standard accuracy where samples in $\mathcal{D}_{test}$ are weighted by $\frac{N}{2(N-n)}$ if they belong to the class $Y=0$ and by $\frac{N}{2n}$ if they belong to the class $Y=1$. 

    We denote respectively by $TP$,$TN$,$FP$,$FN$ the number of true positive samples, the number of true negative, the number of false positive and  the number of false negative from $\hat{f}$ predictions on $\mathcal{D}_{test}$. Thus we have $N= TP+TN+FP+FN$ and  $n=TP + FN$.
    
    We have:
    \label{proof_lemma:balanced_accuracy}
    \begin{align}
        \text{Acc}_{Weighted}(\hat{f}) &= \frac{TP \frac{1}{2}\frac{N}{n} + TN \frac{1}{2}\frac{N}{N-n}}{N} \\
        &= \frac{1}{2}\left(\frac{TP}{TP + FN} + \frac{TN}{TN + FP}\right) \\
        &= \frac{1}{2}\left(Recall + Specificity\right)\\
        &= \text{Acc}_{Balanced}(\hat{f}).
    \end{align}
\end{proof}

\section{TECHNICAL LEMMAS}
\label{section:tech_lemmas}
\subsection{Cumulative distribution function of a binomial distribution}
\noindent
\begin{lemma}[Cumulative distribution function of a binomial distribution]
\label{lem_binomial}
\begin{enumerate}
Let $X$ be a random variable following a binomial law of parameter $n\in\mathbf{N}$ and $p\in [0,1]$. The cumulative distribution function $F$ of $X$ can be expressed as \cite{wadsworth1961introduction}:

\item[$(i)$] \begin{align*}
    F(k;n,p) = \mathds{P}(X \leq k) = \sum_{i=0}^{\lfloor k \rfloor}\binom{n}{i}p^{i}(1-p)^{n-i},
\end{align*}
\item[$(ii)$]  \begin{align*}
\label{res-binomial-repart-beta}
    F(k;n,p) = (n-k)\binom{n}{k} \int_{0}^{1-p}t^{n-k-1}(1-t)^k\mathrm{d}t \\
    =(n-k)\binom{n}{k} \mathcal{B}(n-k,k+1 ;1-p),
\end{align*}
\end{enumerate}
with $\mathcal{B}(a,b;x) = \int_{t=0}^{x}t^{a-1}(1-t)^{b-1}\mathrm{d}t$,  the incomplete beta function.
\end{lemma}
\begin{proof}
    see \cite{wadsworth1961introduction}.
\end{proof}

\subsection{Bounds for the incomplete beta function}

\begin{lemma}
\label{lemma-ineg-beta}
Let $B(a,b;x) = \int_{t=0}^{x}t^{a-1}(1-t)^{b-1}\mathrm{d}t$, be the incomplete beta function. Then we have 
\begin{align*}
   \frac{x^a}{a} \leq B(a,b;x) \leq x^{a-1} \left( \frac{1 -(1-x)^b}{b} \right),
\end{align*}
for $a>0$.
\end{lemma}

\begin{proof}
We have
\begin{align*}
    B(a,b;x) &= \int_{t=0}^{x}t^{a-1}(1-t)^{b-1}\mathrm{d}t \\
    &\le \int_{t=0}^{x}x^{a-1}(1-t)^{b-1}\mathrm{d}t \\
    &= x^{a-1}\int_{t=0}^{x}(1-t)^{b-1}\mathrm{d}t \\
    &= x^{a-1} \left[ (-1)\frac{(1-t)^b}{b} \right]_0^x  \\
    &= x^{a-1} \left[ -\frac{(1-x)^b}{b} + \frac{1}{b} \right] \\
    &= x^{a-1} \frac{1 -(1-x)^b}{b}.
\end{align*}

On the other hand,
\begin{align*}
    B(a,b;x) &= \int_{t=0}^{x}t^{a-1}(1-t)^{b-1}\mathrm{d}t \\
    &\ge \int_{t=0}^{t}x^{a-1}\mathrm{d}t \\
    &= \left[ \frac{t^a}{a}\right]_0^x \\
    &=\frac{x^a}{a}-\frac{0^a}{a} \\
    &=\frac{x^a}{a}. 
\end{align*}

\end{proof}

\bigskip
\subsection{Upper bound for binomial coefficient}
\begin{lemma}
\label{lemma-coeff-binomial}
For $k,n \in \mathds{N}$ such that $k<n$, we have

\begin{align}
    \binom{n}{k} \leq \left(\frac{en}{k}\right)^k.
\end{align}
\end{lemma}

\begin{proof}
\label{proof-lemma-coeff-binomial}
We have,

\begin{align}
    \binom{n}{k} = \frac{n(n-1)\hdots (n-k+1)}{k!} \leq \frac{n^k}{k!}.
\end{align}
Besides,
\begin{align}
    e^k = \sum_{i=0}^{+\infty}\frac{k^i}{i!} \Longrightarrow e^k \geq \frac{k^k}{k!} \Longrightarrow \frac{e^k}{k^k} \geq \frac{1}{k!}.
\end{align}
Hence,
\begin{align}
    \binom{n}{k} = \frac{n(n-1)\hdots (n-k+1)}{k!} \leq \frac{n^k}{k!} \leq \left(\frac{en}{k}\right)^k.
\end{align}

\end{proof}

\bigskip
\subsection{Inequality $ x\ln\left(\frac{1}{x}\right)\leq \sqrt{x}$ }
\begin{lemma}
\label{lemma-ineg-rc-xln}

For $x \in ]0,+ \infty[$,
\begin{align}
    x\ln\left(\frac{1}{x}\right)\leq \sqrt{x}.
\end{align}
\end{lemma}

\begin{proof}
Let,
\label{proof_lemma-ineg-rc-xln}
\begin{align}
    f(x) &= \sqrt{x}-x\ln\left(\frac{1}{x}\right) \\
    &= \sqrt{x} + x\ln(x).
\end{align}
\noindent
Then,
\begin{align}
    f^{'}(x) = \frac{1}{2\sqrt{x}}+ \ln{x} + 1.
\end{align}
\noindent
And,
\begin{align}
    f^{''}(x) = \frac{1}{x} - \frac{1}{4x^{3/2}}.
\end{align}
\noindent
We have,
\begin{align}
\label{proof_lemma-ineg-rc-xln_eq-depart}
    \nonumber f^{''}(x) \geq 0 &\Longrightarrow \frac{1}{x} - \frac{1}{4x^{3/2}} \geq 0 \\
    &\Longrightarrow \frac{1}{x} \geq \frac{1}{4x^{3/2}} 
\end{align}
Since $x \in ]0,+ \infty[$,
\begin{align}
    \Cref{proof_lemma-ineg-rc-xln_eq-depart} &\Longrightarrow \frac{x^{3/2}}{x} \geq \frac{1}{4} \\
    &\Longrightarrow \sqrt{x} \geq \frac{1}{4} \\
    &\Longrightarrow x \geq \frac{1}{16}.
\end{align}

This result leads to,

\medskip
\begin{align}
\begin{tikzpicture}
   \tkzTabInit{$x$ / 1 , $f^{''}$ / 1, $f^{'}$ / 1.5}{$0$, $\frac{1}{16}$, $+\infty$}
   \tkzTabLine{, -, z, +, }
   \tkzTabVar{+/ , -/$2 + \ln(\frac{1}{16}) +1 $ , +/ }
\end{tikzpicture}
\end{align}

\medskip
We have $2 + \ln(\frac{1}{16}) +1 > 0$. So $f^{'}(x) > 0$ for all $x \in ]0, \infty[$. Furthermore $\lim_{x \to 0^+}f(x) = 0$, hence $f(x) > 0$ for all $x \in ]0,\infty[$, therefore $\sqrt{x} > \-x\ln\left(\frac{1}{x}\right)$ for all $x \in ]0,\infty[$.

\end{proof}

\section{TABLES}
\label{appendix:tables}
In this section, the additional tables from the experimental on real world data sets are displayed.

\paragraph{Tables}
The tables in appendix can be divided into 3 categories. First, we have the tables related to random forests. Then the tables related to logistic regression. Finally, we have the tables of LightGBM classifiers. Here are some details fore each group:
\begin{itemize}
    \item \textbf{Random Forest: } In \Cref{tab:annexe_resultats_rf_tuned_others}, PR AUC of data sets not presented in \cref{tab:corps_resultats_rf_tuned} are displayed. In \Cref{tab:annexe_resultats_rf_std} and \Cref{tab:annexe_resultats_rf_others} PR AUC of default random forests are displayed for all the data sets. In \Cref{tab:annexe_resultats_depth_rus} is displayed default random forests PR AUC with a max tree depth fixed to the value of RUS depth.
    \item \textbf{LightGBM: } The PR AUC of the extremely imbalanced data sets when using a LightGBM, with tuned maximum tree depth,are displayed in \Cref{tab:lgbm_tuned_std}. \Cref{tab:resultats_multiclass_lgbm_pr} shows the PR AUC of LightGBM (not tuned) on three multiclass data set.
    \item \textbf{Logistic Regression: } \Cref{tab:resultats_annexe_logreg} dispalys PR AUC of of extremely imbalanced data sets when using Logistic regression. \Cref{tab:reglog_longtailed} shows PR AUC of None, LDAM and Focal loss strategies when using a logistic regression reimplemented using PyTorch.

\end{itemize}

\begin{table}[h]
\centering
\caption{\Cref{tab:corps_resultats_rf_tuned} with standard deviations. Random Forest (tree depth tuned with PR AUC) PR AUC for different rebalancing strategies and different data sets.
The other data sets are displayed in \Cref{tab:annexe_resultats_rf_tuned_others}.}
\label{tab:annexe_resultats_rf_tuned_std}
\setlength\tabcolsep{4pt}
\begin{tabular}{lcccccccccccr}
     \toprule 
       \small Strategy & \small None  & \small CW  & \small RUS & \small ROS & \small NM1  & \small BS1 & \small BS2   & \scriptsize SMOTE & \scriptsize SMOTE & \small MGS & \scriptsize CT & \scriptsize Forest   \\
       &   &   &  &  & \small   &  &  &   & \scriptsize $K$-tuned &  \scriptsize ($d+1$) & \scriptsize GAN & \scriptsize Diff  \\
     \midrule
     \small CreditCard \tiny($0.2\%$) &0.776	&0.783	&0.751	&\textbf{0.787}	&0.586	&0.786	&0.784	&0.773  &0.771 &0.747 &0.757	&0.752 \\
     \tiny std &\tiny0.0	&\tiny0.0	&\tiny0.008	&\tiny0.003	&\tiny0.0	&\tiny0.003	&\tiny0.003	&\tiny0.002  &\tiny0.006 & \tiny 0.01 &\tiny0.012	&\tiny0.0 \\
     \small Abalone \tiny($1\%$) &0.048	&0.055	&0.052	&0.052	&0.024	&0.044	&0.039	&0.046	&0.051	&\textbf{0.056} &0.050	&\textbf{0.056}    \\
     \tiny std &\tiny0.006	&\tiny0.009	&\tiny0.017	&\tiny0.011	&\tiny0.007	&\tiny0.008	&\tiny0.006	&\tiny0.011	&\tiny0.012	&\tiny0.014  &\tiny0.016	&\tiny0.006\\
     \small \textit{Phoneme} \tiny($1\%$) & 0.198	&0.211	&0.086	&0.200	&0.054	&0.222	&0.211	&0.224	&0.240	&\textbf{0.269} &0.130	&0.249\\
     \tiny std &\tiny 0.026	&\tiny0.026	&\tiny0.031	&\tiny0.037	&\tiny0.014	&\tiny0.03	&\tiny0.028	&\tiny0.035	&\tiny0.035	&\tiny0.032 &\tiny0.042	&\tiny0.047 \\ 
     \small \textit{Haberman} \tiny($10\%$) &0.252	&0.268	&0.289	&0.295	&0.268	&0.308	&0.294	&0.306	&0.307	&\textbf{0.311} &0.280	&0.306 \\
     \tiny std &\tiny0.033	&\tiny0.035	&\tiny0.045	&\tiny0.038	&\tiny0.04	&\tiny0.036	&\tiny0.039	&\tiny0.034	&\tiny0.036	&\tiny0.032 &\tiny0.024	&\tiny0.042\\
     \small \textit{MagicTel} \tiny($20\%$) &0.754	&0.759	&0.741	&\textbf{0.760}	&0.338	&0.739	&0.742	&0.757	&0.757 &0.750
 &0.693	&0.756\\
     \tiny std &\tiny0.002	&\tiny0.003	&\tiny0.004	&\tiny0.004	&\tiny0.009	&\tiny0.004	&\tiny0.004	&\tiny0.004	&\tiny0.003 &\tiny0.003 &\tiny0.012	&\tiny0.003\\
     \small \textit{{California}} \tiny($1\%$) &0.298	&0.283	&0.207	&0.277	&0.019	&0.234	&0.226	&0.242	&0.258	&\textbf{0.322}  &0.152	&0.294\\
     \tiny std &\tiny0.015	&\tiny0.02	&\tiny0.023	&\tiny0.017	&\tiny0.001	&\tiny0.012	&\tiny0.015	&\tiny0.02	&\tiny0.019	&\tiny0.024 &\tiny0.023&\tiny0.013 \\
     \bottomrule
\end{tabular}
\end{table}

\begin{table}
\centering
\caption{Remaining data sets (without those of \Cref{tab:corps_resultats_rf_tuned}). Random Forest (tree depth tuned with PR AUC) PR AUC for different rebalancing strategies and different data sets. Only datasets such that the None strategy is on par with the best strategies are displayed. Other data sets are presented in \Cref{tab:annexe_resultats_rf_tuned_others}. Mean standard deviations are computed. The best strategy is highlighted in bold for each data set.}
\label{tab:annexe_resultats_rf_tuned_others}
\setlength\tabcolsep{3.5pt}
\begin{tabular}{lcccccccccccr}
     \toprule 
       \small Strategy & \small None  & \small CW  & \small RUS & \small ROS & \small NM1  & \small BS1 & \small BS2   & \scriptsize SMOTE & \scriptsize SMOTE & \small MGS & \scriptsize CT & \scriptsize Forest   \\
       &   &   &  &  & \small   &  &  &   & \scriptsize  $K$-tuned &  \scriptsize ($d+1$) & \scriptsize GAN & \scriptsize Diff  \\
     \midrule

     \small Phoneme 	&\textbf{0.920}	&0.918	&0.883	&0.919	&0.844	&0.912 &0.911	&0.916&0.915	&0.911	&0.910 &0.912
\\
     \tiny std &\tiny0.0	&\tiny0.0	&\tiny0.004	&\tiny0.002	&\tiny0.003	&\tiny0.001 &\tiny0.001	&\tiny0.001&\tiny0.002	&\tiny0.0	&\tiny0.005 &\tiny 0.001\\
     \small \textit{Phoneme} \tiny($20\%$)  &\textbf{0.864}	&0.859	&0.781	&0.861	&0.577	&0.844	&0.847	&0.857&0.857	&0.854	&0.824 &0.003\\
     \tiny std &\tiny0.002	&\tiny0.004	&\tiny0.012	&\tiny0.005	&\tiny0.013	&\tiny0.005	&\tiny0.004	&\tiny0.004&\tiny0.001	&\tiny0.005	&\tiny0.008 &\tiny 0.004
 \\
     \small \textit{Phoneme} \tiny($10\%$)  &\textbf{0.714}	&0.700	&0.540	&0.705	&0.278	&0.676	&0.690	&0.696&0.691	&0.705	&0.609 &0.686\\
     \tiny std &\tiny0.009	&\tiny0.008	&\tiny0.015	&\tiny0.014	&\tiny0.015	&\tiny0.014	&\tiny0.013	&\tiny0.006&\tiny0.014	&\tiny0.007	&\tiny0.015 &\tiny 0.003
 \\
     \small Yeast &0.830	&\textbf{0.846}	&0.839	&0.845	&0.732	&0.787	&0.778	&0.831&0.824	&0.814	&0.823 &0.703\\
     \tiny std &\tiny0.015	&\tiny0.01	&\tiny0.012	&\tiny0.007	&\tiny0.014	&\tiny0.013	&\tiny0.015	&\tiny0.012&\tiny0.011	&\tiny0.005	&\tiny0.019 & \tiny 0.016\\
     \small \textit{Yeast} \tiny($1\%$) &\textbf{0.380}	&0.323	&0.254	&0.298	&0.149	&0.333	&0.320	&0.333	&0.335	&0.295 &0.252	&0.243 \\
     \tiny std &\tiny0.053	&\tiny0.054	&\tiny0.069	&\tiny0.06	&\tiny0.051	&\tiny0.058	&\tiny0.041	&\tiny0.055	&\tiny0.06	&\tiny0.063 &\tiny0.016	&\tiny0.031\\
     \small Wine \tiny($4\%$) &0.601	&0.598	&0.387	&\textbf{0.606}	&0.149	&0.578	&0.585	&0.589	&0.589	&0.545 &0.424	&0.559  \\
     \tiny std &\tiny0.024	&\tiny0.029	&\tiny0.024	&\tiny0.026	&\tiny0.014	&\tiny0.021	&\tiny0.025	&\tiny0.025	&\tiny0.027	&\tiny0.023 &\tiny0.031	&\tiny0.022 \\
     \small Pima	&\textbf{0.713}	&0.705	&0.701	&0.712	&0.697	&0.681	&0.684	&0.698&0.706	&0.701	&0.691 &0.703\\
     \tiny std &\tiny0.012	&\tiny0.016	&\tiny0.007	&\tiny0.005	&\tiny0.008	&\tiny0.004	&\tiny0.006	&\tiny0.01&\tiny0.002	&\tiny0.009	&\tiny0.004 &\tiny 0.001 \\
     \small \textit{Pima} \tiny($20\%$) &\textbf{0.527}	&0.513	&0.517	&0.519	&0.484	&0.502	&0.503	&0.504	&0.507	&0.520 &0.495	&0.481\\
     \tiny std &\tiny 0.018	&\tiny0.018	&\tiny0.015	&\tiny0.018	&\tiny0.017	&\tiny0.017	&\tiny0.013	&\tiny0.016	&\tiny0.013	&\tiny0.015 &\tiny0.019	&\tiny0.017 \\
     \small Haberman &0.461	&\textbf{0.477}	&0.441	&0.457	&0.471	&0.457	&0.455	&0.462&0.440	&0.468	&0.433 &0.455 \\
     \tiny std &\tiny0.01	&\tiny0.013	&\tiny0.031	&\tiny0.025	&\tiny0.023	&\tiny0.005	&\tiny0.015	&\tiny0.03&\tiny0.001	&\tiny0.028	&\tiny0.011 &\tiny 0.006 \\
     \small \textit{{California}} \tiny($20\%$)  &\textbf{0.889}	&0.885	&0.872	&0.884	&0.674	&0.875 &0.878	&0.882	&0.882	&0.880	&0.854	&0.852
\\
     \tiny std &\tiny 0.001	&\tiny0.0	&\tiny0.002	&\tiny0.001	&\tiny0.0 &\tiny0.000	&\tiny0.001	&\tiny0.001	&\tiny0.0	&\tiny0.0	&\tiny0.005	&\tiny0.002 \\
     \small \textit{{California}} \tiny($10\%$)   &\textbf{0.803}	&0.796	&0.757	&0.797	&0.411	&0.773 &0.778	&0.787	&0.784	&0.791	&0.718	&0.728\\
     \tiny std &\tiny 0.000	&\tiny0.002	&\tiny0.001	&\tiny0.004	&\tiny0.007 &\tiny 0.002	&\tiny0.005	&\tiny0.004	&\tiny0.003	&\tiny0.002	&\tiny0.007	&\tiny0.002 \\
     \small House\_16H \tiny($20\%$)	&\textbf{0.858}	&0.846	&0.834	&0.849	&0.587	&0.829	&0.832	&0.840	&0.838	&0.835	&0.846	&0.814\\
     \tiny std &\tiny 0.001&\tiny	0.001	&\tiny0.001	&\tiny0.000	&\tiny0.006	&\tiny0.002	&\tiny0.001	&\tiny0.001	&\tiny0.002	&\tiny0.001	&\tiny0.002	&\tiny0.000 \\
     \small \textit{House\_16H} \tiny($10\%$) 	&\textbf{0.756}	&0.729	&0.693	&0.733	&0.291	&0.701	&0.706	&0.714	&0.712	&0.692	&0.723	&0.648\\
     \tiny std &\tiny 0.001	&\tiny0.001	&\tiny0.002	&\tiny0.002	&\tiny0.001	&\tiny0.000	&\tiny0.003	&\tiny0.003	&\tiny0.005	&\tiny0.004	&\tiny0.008	&\tiny0.001 \\
     \small \textit{House\_16H} \tiny($1\%$) &\textbf{0.317}	&0.244	&0.144	&0.220	&0.034	&0.210	&0.227	&0.194	&0.199	&0.178	&0.231	&0.187\\
     \tiny std &\tiny0.011	&\tiny0.021	&\tiny0.018	&\tiny0.004	&\tiny0.003	&\tiny0.009	&\tiny0.003	&\tiny0.004	&\tiny0.000	&\tiny0.001	&\tiny0.007	&\tiny0.002 \\
     \small Vehicle &0.978 &0.979	&0.961	&0.980	&0.962	&0.976	&0.978	&0.978&0.978	&\textbf{0.981}	&0.980 & 0.980\\
     \tiny std &\tiny0.003	&\tiny0.001	&\tiny0.002	&\tiny0.001	&\tiny0.005	&\tiny0.003	&\tiny0.004	&\tiny0.001&\tiny0.001	&\tiny0.004	&\tiny0.0 &\tiny 0.001 \\
     \small \textit{Vehicle} \tiny($10\%$)&\textbf{0.958}	&0.950	&0.872	&0.935	&0.710	&0.927	&0.928	&0.953&0.944	&0.956	&0.939 & 0.954\\
     \tiny std &\tiny0.003	&\tiny0.003	&\tiny0.011	&\tiny0.015	&\tiny0.02	&\tiny0.003	&\tiny0.003	&\tiny0.005&\tiny0.009	&\tiny0.003	&\tiny0.011 & \tiny 0.001\\
     \small Ionosphere &\textbf{0.975}	&0.973	&0.963	&0.971	&0.933	&0.974	&0.973	&0.971	&0.971	&0.962	&0.969	&0.969\\
     \tiny std &\tiny0.003	&\tiny0.001	&\tiny0.009	&\tiny0.004	&\tiny0.002	&\tiny0.003	&\tiny0.001	&\tiny0.001	&\tiny0.000	&\tiny0.002	&\tiny0.001	&\tiny0.000\\
     \small \textit{Ionosphere} \tiny($20\%$) &\textbf{0.972}	&0.963	&0.943	&0.966	&0.688	&0.928	&0.949	&0.937	&0.929	&0.945	&0.958	&0.923\\
     \tiny std &\tiny0.002	&\tiny0.005	&\tiny0.002	&\tiny0.000	&\tiny0.013	&\tiny0.025	&\tiny0.008	&\tiny0.013	&\tiny0.015	&\tiny0.007	&\tiny0.001	&\tiny0.014 \\
     \small \textit{Ionosphere} \tiny($10\%$) &\textbf{0.950}	&0.929	&0.885	&0.925	&0.522	&0.840	&0.868	&0.815	&0.848	&0.826	&0.916	&0.854\\
     \tiny std &\tiny0.001	&\tiny0.005	&\tiny0.006	&\tiny0.006	&\tiny0.056	&\tiny0.014	&\tiny0.011	&\tiny0.029	&\tiny0.002	&\tiny0.023	&\tiny0.013	&\tiny0.027\\
     \small Breast Cancer 	&0.985	&0.984	&0.986	&0.987	&0.987	&0.984	&0.979	&0.985	&\textbf{0.988}	&0.987	&0.985	&0.984\\
     \tiny std &\tiny0.004	&\tiny0.004	&\tiny0.002	&\tiny0.000	&\tiny0.001	&\tiny0.003	&\tiny0.000	&\tiny0.003	&\tiny0.001	&\tiny0.000	&\tiny0.001	&\tiny0.002 \\
     \small \textit{Breast Cancer} \tiny($20\%$)&0.982	&0.976	&0.975	&0.980	&0.984	&0.972	&0.960	&0.982	&0.982	&0.981	&\textbf{0.985}	&0.976\\
     \tiny std &\tiny0.002	&\tiny0.008	&\tiny0.001	&\tiny0.001	&\tiny0.002	&\tiny0.001	&\tiny0.007	&\tiny0.001	&\tiny0.003	&\tiny0.003	&\tiny0.001	&\tiny0.007
\\
     \small \textit{Breast Cancer} \tiny($10\%$)&0.967	&0.952	&0.903	&0.958	&0.976	&0.920	&0.915	&0.947	&0.934	&0.966	&\textbf{0.979}	&0.948\\
     \tiny std &\tiny0.006	&\tiny0.004	&\tiny0.002	&\tiny0.002	&\tiny0.008	&\tiny0.010	&\tiny0.002	&\tiny0.004	&\tiny0.011	&\tiny0.009	&\tiny0.005	&\tiny0.004
	 \\
     \bottomrule
\end{tabular}
\end{table}

\begin{table}[h]
\centering
\caption{Extremely imbalanced data sets. Random Forest (tree depth=$\infty$) PR AUC for different rebalancing strategies and different data sets. Data sets artificially undersampled for minority class are in italics. Other data sets are presented in \Cref{tab:annexe_resultats_rf_others}. Mean standard deviations are computed. }
\label{tab:annexe_resultats_rf_std}
\setlength\tabcolsep{4pt}
\begin{tabular}{lcccccccccccccr} 
     \toprule 
      \small Strategy & \small None  & \small CW  & \small RUS & \small ROS & \small NM1  & \small BS1 & \small BS2   & \scriptsize SMOTE & \scriptsize SMOTE & \small MGS & \scriptsize CT & \scriptsize Forest   \\
       &   &   &  &  & \small   &  &  &   & \scriptsize  $K$-tuned &  \scriptsize ($d+1$) & \scriptsize GAN & \scriptsize Diff  \\
     \midrule
     \small CreditCard \tiny ($0.2\%$) &0.772	&0.776	&0.751	&0.780	&0.524	&\textbf{0.784}	&0.780	&0.772	&0.770	&0.731 &0.754	&0.743   \\
     \tiny std &\tiny0.004	&\tiny0.003	&\tiny0.010	&\tiny0.003	&\tiny0.024	&\tiny0.003	&\tiny0.003	&\tiny0.004	&\tiny0.004	&\tiny0.005 &\tiny0.008	&\tiny0.005 \\
     \small Abalone \tiny ($1\%$) &0.046	&0.040	&0.046	&0.041	&0.024	&0.047	&0.039	&0.040	&0.051	&\textbf{0.056}	&0.048	&0.047	 \\
     \tiny std &\tiny0.009	&\tiny0.008	&\tiny0.014	&\tiny0.008	&\tiny0.006	&\tiny0.008	&\tiny0.005	&\tiny0.007	&\tiny0.009	&\tiny0.010	&\tiny0.007	&\tiny0.010 \\
     \small \textit{Phoneme} \tiny ($1\%$) &0.210	&0.228	&0.082	&0.227	&0.052	&0.237	&0.220	&0.241	&0.259	&\textbf{0.269}	&0.133	&0.255	\\
     \tiny std &\tiny0.027	&\tiny0.030	&\tiny0.030	&\tiny0.030	&\tiny0.013	&\tiny0.034	&\tiny0.027	&\tiny0.037	&\tiny0.034	&\tiny0.026	&\tiny0.036	&\tiny0.036	\\
     \small \textit{Haberman} \tiny ($10\%$) &0.216	&0.223	&0.284	&0.228	&0.266	&0.278	&0.255	&0.275  &0.289	&\textbf{0.294}	&0.277	&0.265	 \\
     \tiny std &\tiny0.022	&\tiny0.019	&\tiny0.043	&\tiny0.021	&\tiny0.034	&\tiny0.029	&\tiny0.030	&\tiny0.027   &\tiny0.033	&\tiny0.035	&\tiny0.041	&\tiny0.035	\\
     \small \textit{MagicTel} \tiny ($20\%$) &0.751	&\textbf{0.758}	&0.737	&\textbf{0.758}	&0.340	&0.736	&0.739	&0.754	&0.753	&0.746	&0.695	&0.754\\
     \tiny std &\tiny0.003	&\tiny0.003	&\tiny0.005	&\tiny0.003	&\tiny0.007	&\tiny0.003	&\tiny0.004	&\tiny0.003	&\tiny0.003	&\tiny0.003	&\tiny0.007	&\tiny0.003	 \\
     \small \textit{California} \tiny ($1\%$) &0.294	&0.292	&0.197	&0.282	&0.019	&0.234	&0.230	&0.249 &0.265	&\textbf{0.322}	&0.195	&0.294 \\
     \tiny std &\tiny0.012	&\tiny0.016	&\tiny0.023	&\tiny0.016	&\tiny0.003	&\tiny0.015	&\tiny0.016	&\tiny0.018	&\tiny0.017	&\tiny0.024	&\tiny0.025	&\tiny0.018	 \\
    \bottomrule
\end{tabular}
\end{table}

\begin{table}[h]
\centering
\caption{Remaining data sets (without those of \Cref{tab:corps_resultats_rf_tuned}). Random Forest (tree depth=$\infty$) PR AUC for different rebalancing strategies and different data sets. Only datasets such that the None strategy is on par with the best strategies are displayed. Other data sets are presented in \Cref{tab:annexe_resultats_rf_std}. Mean standard deviations are computed. The best strategy is highlighted in bold for each data set.}
\label{tab:annexe_resultats_rf_others}
\setlength\tabcolsep{2.5pt}
\begin{tabular}{lcccccccccccccr} 
     \toprule 
      \small Strategy & \small None  & \small CW  & \small RUS & \small ROS & \small NM1  & \small BS1 & \small BS2   & \scriptsize SMOTE & \scriptsize SMOTE & \small MGS & \scriptsize CT & \scriptsize Forest   \\
       &   &   &  &  & \small   &  &  &   & \scriptsize $K$-tuned &  \scriptsize ($d+1$) & \scriptsize GAN & \scriptsize Diff  \\
     \midrule
     \small Phoneme  &\textbf{0.919}	&0.916	&0.882	&0.919	&0.845	&0.910	&0.913	&0.916	&0.914	&0.913	&0.905	&0.911 \\ 
     \tiny std &\tiny0.003	&\tiny0.003	&\tiny0.004	&\tiny0.003	&\tiny0.006	&\tiny0.004	&\tiny0.003	&\tiny0.003	&\tiny0.003	&\tiny0.003	&\tiny0.003	&\tiny0.003 \\
     \small \textit{Phoneme}\tiny($20\%$) &\textbf{0.862}	&0.858	&0.774	&0.861	&0.572	&0.843	&0.848	&0.854	&0.854	&0.854	&0.824	&0.849\\
     \tiny std &\tiny0.004	&\tiny0.003	&\tiny0.011	&\tiny0.004	&\tiny0.023	&\tiny0.005	&\tiny0.005	&\tiny0.005	&\tiny0.003	&\tiny0.004	&\tiny0.007	&\tiny0.004 \\
     \small \textit{Phoneme}\tiny($10\%$) &\textbf{0.724}	&0.707	&0.538	&0.713	&0.271	&0.675	&0.685	&0.695	&0.693	&0.701	&0.620	&0.685\\ 
     \tiny std &\tiny0.011	&\tiny0.012	&\tiny0.020	&\tiny0.013	&\tiny0.023	&\tiny0.015	&\tiny0.011	&\tiny0.012	&\tiny0.011	&\tiny0.009	&\tiny0.021	&\tiny0.014 \\
     \small Pima &\textbf{0.700}	&0.699	&0.688	&0.695	&0.691	&0.674	&0.674	&0.685	&0.680	&0.684	&0.699	&0.689\\
     \tiny std &\tiny0.011	&\tiny0.010	&\tiny0.012	&\tiny0.009	&\tiny0.009	&\tiny0.009	&\tiny0.013	&\tiny0.011	&\tiny0.016	&\tiny0.011	&\tiny0.014	&\tiny0.013 \\
     \small \textit{Pima} \tiny ($20\%$) &\textbf{0.506}	&0.501	&\textbf{0.506}	&0.497	&0.481	&0.480	&0.482	&0.479	&0.485	&0.484	&0.488	&0.472	 \\
     \tiny std &\tiny0.024	&\tiny0.019	&\tiny0.019	&\tiny0.020	&\tiny0.017	&\tiny0.017	&\tiny0.020	&\tiny0.019   &\tiny0.021	&\tiny0.018	&\tiny0.023	&\tiny0.020	 \\
     \small Yeast &0.836	&\textbf{0.839}	&0.823	&0.830	&0.716	&0.790	&0.788	&0.824	&0.822	&0.820	&0.830	&0.814\\
     \tiny std	&\tiny0.009	&\tiny0.010	&\tiny0.010	&\tiny0.013	&\tiny0.013	&\tiny0.014	&\tiny0.011	&\tiny0.010	&\tiny0.009	&\tiny0.015	&\tiny0.010	&\tiny0.013
 \\
     \small \textit{Yeast} \tiny ($1\%$) &0.314	&0.308	&0.242	&0.267	&0.112	&0.340	&0.312	&\textbf{0.353}	&0.351	&0.288	&0.300	&0.269	\\
     \tiny std &\tiny0.056	&\tiny0.057	&\tiny0.066	&\tiny0.046	&\tiny0.027	&\tiny0.059	&\tiny0.044	&\tiny0.057	&\tiny0.047	&\tiny0.050	&\tiny0.073	&\tiny0.063 \\
     \small Wine \tiny ($4\%$) &0.602	&0.599	&0.385	&\textbf{0.604}	&0.150	&0.579	&0.587	&0.587	&0.585	&0.543	&0.460	&0.567	  \\
     \tiny std &\tiny0.024	&\tiny0.027	&\tiny0.032	&\tiny0.028	&\tiny0.014	&\tiny0.022	&\tiny0.025	&\tiny0.026	&\tiny0.025	&\tiny0.020	&\tiny0.030	&\tiny0.025	 \\
     \small Haberman &0.423	&0.423	&0.436	&0.410	&\textbf{0.466}	&0.414	&0.412	&0.421	&0.430	&0.434	&0.427	&0.419 \\
     \tiny std &\tiny0.023	&\tiny0.027	&\tiny0.024	&\tiny0.022	&\tiny0.019	&\tiny0.022	&\tiny0.023	&\tiny0.021	&\tiny0.022	&\tiny0.023	&\tiny0.031	&\tiny0.024 \\
     \small \textit{California} \tiny($20\%$)&\textbf{0.887}	&0.884	&0.868	&0.885	&0.672	&0.873	&0.878	&0.881	&0.880	&0.880	&0.855	&0.875\\
     \tiny std &\tiny0.001	&\tiny0.001	&\tiny0.002	&\tiny0.002	&\tiny0.004	&\tiny0.001	&\tiny0.002	&\tiny0.002	&\tiny0.002	&\tiny0.002	&\tiny0.005	&\tiny0.001 \\
     \small \textit{California} \tiny($10\%$)&\textbf{0.799}	&0.795	&0.756	&0.796	&0.382	&0.772	&0.777	&0.786	&0.784	&0.791	&0.698	&0.782\\
     \tiny std &\tiny0.003	&\tiny0.003	&\tiny0.006	&\tiny0.003	&\tiny0.012	&\tiny0.004	&\tiny0.004	&\tiny0.003	&\tiny0.003	&\tiny0.003	&\tiny0.016	&\tiny0.003 \\
     \textit{House\_16H} \tiny($20\%$) &\textbf{0.854}	&0.847	&0.830	&0.847	&0.577	&0.826	&0.829	&0.836	&0.836	&0.831	&0.847	&0.821\\
     \tiny std &\tiny0.001	&\tiny0.001	&\tiny0.002	&\tiny0.000	&\tiny0.005	&\tiny0.001	&\tiny0.000	&\tiny0.000	&\tiny0.001	&\tiny0.002	&\tiny0.001	&\tiny0.000 \\
     \small \textit{House\_16H} \tiny($10\%$)&\textbf{0.752}	&0.730	&0.686	&0.728	&0.240	&0.702	&0.704	&0.708	&0.705	&0.690	&0.725	&0.671\\
     \tiny std &\tiny0.002	&\tiny0.002	&\tiny0.002	&\tiny0.002	&\tiny0.010	&\tiny0.001	&\tiny0.002	&\tiny0.002	&\tiny0.003	&\tiny0.002	&\tiny0.006	&\tiny0.003\\
     \small \textit{House\_16H} \tiny($1\%$)&\textbf{0.303}	&0.244	&0.157	&0.241	&0.032	&0.212	&0.235	&0.205	&0.212	&0.193	&0.186	&0.187\\
     \tiny std &\tiny0.010	&\tiny0.014	&\tiny0.018	&\tiny0.013	&\tiny0.006	&\tiny0.013	&\tiny0.010	&\tiny0.011	&\tiny0.012	&\tiny0.011	&\tiny0.022	&\tiny0.011 \\
     Vehicle &\textbf{0.982}	&0.980	&0.964	&0.981	&0.958	&0.980	&0.982	&0.978	&0.979	&0.983	&0.979	&0.979\\
     \tiny std &\tiny0.002	&\tiny0.003	&\tiny0.006	&\tiny0.003	&\tiny0.008	&\tiny0.003	&\tiny0.003	&\tiny0.002	&\tiny0.002	&\tiny0.003	&\tiny0.003	&\tiny0.003 \\
     \small Vehicle \tiny($10\%$) &0.952	&0.941	&0.866	&0.939	&0.717	&0.932	&0.925	&0.943	&0.945	&\textbf{0.954}	&0.940	&0.951\\
     \tiny std &\tiny0.007	&\tiny0.011	&\tiny0.029	&\tiny0.010	&\tiny0.025	&\tiny0.012	&\tiny0.014	&\tiny0.008	&\tiny0.008	&\tiny0.008	&\tiny0.011	&\tiny0.006\\
     \small Ionosphere &\textbf{0.971}	&0.971	&0.964	&0.970	&0.933	&0.969	&0.971	&0.968	&0.968	&0.964	&0.970	&0.970\\
     \tiny std &\tiny0.003	&\tiny0.003	&\tiny0.005	&\tiny0.003	&\tiny0.008	&\tiny0.003	&\tiny0.003	&\tiny0.003	&\tiny0.004	&\tiny0.004	&\tiny0.004	&\tiny0.003 \\
     \small \textit{Ionosphere} \tiny(20\%)&\textbf{0.968}&0.960	&0.942	&0.962	&0.701	&0.911	&0.942	&0.926	&0.917	&0.956	&0.964	&0.941\\
     \tiny std &\tiny0.001	&\tiny0.000	&\tiny0.004	&\tiny0.004	&\tiny0.012	&\tiny0.018	&\tiny0.007	&\tiny0.005	&\tiny0.009	&\tiny0.001	&\tiny0.009	&\tiny0.006 \\
     \small \textit{Ionosphere} \tiny(10\%) &\textbf{0.932}	&0.922	&0.860	&0.925	&0.642	&0.819	&0.852	&0.792	&0.804	&0.880	&0.928	&0.904\\
     \tiny std &\tiny0.007	&\tiny0.000	&\tiny0.014	&\tiny0.001	&\tiny0.076	&\tiny0.020	&\tiny0.013	&\tiny0.008	&\tiny0.037	&\tiny0.025	&\tiny0.009	&\tiny0.003 \\
     \small Breast Cancer &\textbf{0.987}	&0.986	&0.984	&0.986	&0.986	&0.982	&0.983	&0.985	&0.986	&0.987	&0.986	&0.987\\
     \tiny std &\tiny0.002	&\tiny0.003	&\tiny0.004	&\tiny0.003	&\tiny0.002	&\tiny0.003	&\tiny0.004	&\tiny0.003	&\tiny0.003	&\tiny0.002	&\tiny0.003	&\tiny0.003 \\
     \small \textit{Breast Cancer} \tiny(20\%) &\textbf{0.984}	&0.981	&0.970	&0.978	&0.982	&0.972	&0.971	&0.978	&0.977	&0.981	&0.981	&0.979\\
     \tiny std &\tiny0.006	&\tiny0.006	&\tiny0.008	&\tiny0.007	&\tiny0.007	&\tiny0.007	&\tiny0.007	&\tiny0.006	&\tiny0.007	&\tiny0.006	&\tiny0.007	&\tiny0.005 \\
     \small \textit{Breast Cancer} \tiny(10\%) &\textbf{0.978}	&0.966	&0.933	&0.965	&0.976	&0.940	&0.938	&0.964	&0.959	&0.967	&0.969	&0.967\\
     \tiny std &\tiny0.008	&\tiny0.009	&\tiny0.020	&\tiny0.009	&\tiny0.010	&\tiny0.014	&\tiny0.015	&\tiny0.009	&\tiny0.013	&\tiny0.012	&\tiny0.014	&\tiny0.010 \\
 \bottomrule
\end{tabular}
\end{table}

\begin{table}[h]
\centering
\caption{Extremely imbalanced data sets. PR AUC Random Forest with tree depth=RUS. On the last column, the value of maximal depth when using Random forest (max\_depth=$\infty$) with RUS strategy for each data set. For each data set, the best strategy is highlighted in bold and the mean of the standard deviation is computed (and rounded to $10^{-3}$). } 
\label{tab:annexe_resultats_depth_rus}
\setlength\tabcolsep{2pt}
\begin{tabular}{lccccccccccccr} 
     \toprule 
      \small Strategy & \small None  & \small CW  & \small RUS & \small ROS & \small NM1  & \small BS1 & \small BS2   & \scriptsize SMOTE & \scriptsize SMOTE & \scriptsize MGS & \scriptsize CTGAN & \scriptsize Forest & \tiny depth   \\
       &   &   &  &  & \small   &  &  &   & \scriptsize $K$-tuned &  \scriptsize ($d+1$) & &\scriptsize Diff &  \\
     \midrule
     \small CreditCard \tiny($\pm 0.005$)&0.773	&\textbf{0.781}	&0.751	&0.774	&0.508	&\textbf{0.781}	&0.777	&0.752	&0.773	&0.760 &0.714	&0.773
 & \small 10\\
     \small Abalone \tiny($1\%$)\tiny($\pm0.01 $)&0.048	&0.055	&0.049	&0.057	&0.023	&0.044	&0.041	&0.054	&0.054	&\textbf{0.058}	&0.048	&0.057 & \small 7  \\
     \small \textit{Phoneme} \tiny($1\%$)\tiny($\pm  0.025$)&0.162	&0.102	&0.087	&0.104	&0.049	&0.166	&0.142	&0.124	&\textbf{0.255}	&0.146	&0.084	&0.120 & \small 6\\
     \scriptsize \textit{Haberman} \tiny($10\%$)\tiny($\pm 0.033$)&0.242	&0.260	&0.292	&0.261	&0.268	&0.301	&0.280	&0.291	&0.294	&\textbf{0.321}	&0.278	&0.305& 7  \\
     \small \textit{MagicTel} \tiny ($20\%$) \tiny($\pm 0.004$)&0.755	&\textbf{0.759}	&0.741	&0.760	&0.339	&0.737	&0.740	&0.758	&0.752	&0.751	&0.697	&0.758& \small 20   \\
     \small \textit{California} \tiny ($1\%$)\tiny($\pm 0.017$)&\textbf{0.295}	&0.232	&0.200	&0.204	&0.018	&0.192	&0.190	&0.223	&0.265	&0.301	&0.158	&0.288 & \small 10\\  
     \bottomrule
\end{tabular}
\end{table}

\begin{table*}
\centering
\caption{Extremely imbalanced data sets. LightGBM (max\_depth=tuned on PR AUC) PR AUC.}
\label{tab:lgbm_tuned_std}
\setlength\tabcolsep{3.5pt}
\begin{tabular}{lcccccccccccr}
     \toprule 
        \small Strategy & \small None  & \small CW  & \small RUS & \small ROS & \small NM1  & \small BS1 & \small BS2   & \scriptsize SMOTE & \scriptsize SMOTE & \small MGS & \scriptsize CT & \scriptsize Forest   \\
       &   &   &  &  & \small   &  &  &   & \scriptsize $K$-tuned &  \scriptsize ($d+1$) & \scriptsize GAN & \scriptsize Diff  \\
       \midrule
        CreditCard \tiny($0.2\%$) &0.431	&0.767	&0.727	&\textbf{0.772}	&0.306	&0.730	&0.734	&0.742 &0.749 &0.754 &0.755&0.770\\
        \tiny std &\tiny0.03	&\tiny0.0	&\tiny0.015	&\tiny0.006	&\tiny0.0	&\tiny0.006	&\tiny0.012	&\tiny0.007 &\tiny0.011 &\tiny0.001 &\tiny 0.013 &\tiny0.005
  \\
        Abalone \tiny($1\%$)  &0.051	&0.048	&0.039	&0.051	&0.031	&0.041	&0.040	&0.046	&\textbf{0.053} &\textbf{0.053}	&0.048 &0.050 \\
        \tiny std &\tiny0.013	&\tiny0.01	&\tiny0.013	&\tiny0.014	&\tiny0.009	&\tiny0.005	&\tiny0.01	&\tiny0.007	&\tiny0.008 &\tiny0.013	&\tiny0.013 &\tiny0.007  \\
        \textit{Phoneme} \tiny($1\%$) &0.181	&0.177	&0.052	&0.174	&0.039	&0.257	&0.248	&0.237	&0.245	&0.239	&0.106 &\textbf{0.268}  \\
        \tiny std &\tiny0.03	&\tiny0.029	&\tiny0.013	&\tiny0.039	&\tiny0.005	&\tiny0.031	&\tiny0.039	&\tiny0.039	&\tiny0.037	&\tiny0.027	&\tiny0.026 &\tiny0.05 \\
        \textit{Haberman} \tiny($10\%$)  &0.297	&0.316	&0.125	&0.300	&0.134	&0.294	&0.263	&0.302 &0.322	&\textbf{0.323}	&0.305 &0.282 \\
        \tiny std &\tiny0.047	&\tiny0.038	&\tiny0.023	&\tiny0.041	&\tiny0.027	&\tiny0.039	&\tiny0.041	&\tiny0.039 &\tiny0.055	&\tiny0.037	&\tiny0.047 &\tiny0.033 \\
        \textit{MagicTel} \tiny($20\%$) &0.762	&\textbf{0.764}	&0.730	&0.763	&0.281	&0.729	&0.729	&0.760	&0.761 &0.746	&0.702 &0.763 \\
        \tiny std &\tiny0.004	&\tiny0.003	&\tiny0.005	&\tiny0.004	&\tiny0.008	&\tiny0.005	&\tiny0.006	&\tiny0.004 &\tiny0.004 	&\tiny0.003	&\tiny0.009 &\tiny0.003 \\
        \textit{California} \tiny($1\%$) &0.358	&0.342	&0.229	&0.330	&0.037	&0.308	&0.292	&0.342	&0.352 &\textbf{0.401}	&0.209 &0.352 \\
        \tiny std &\tiny0.031	&\tiny0.024	&\tiny0.031	&\tiny0.028	&\tiny0.009	&\tiny0.017	&\tiny0.021	&\tiny0.022 &\tiny0.015 &\tiny0.023 &\tiny0.033 & \tiny 0.012\\
     \bottomrule
\end{tabular}
\end{table*}

\begin{table*}[h]
\centering
\caption{Extremely imbalanced data sets.. Logistic Regression  PR AUC. For each data set, the best strategy is highlighted in bold and the mean of the standard deviation is computed (and rounded to $10^{-3}$).}
\label{tab:resultats_annexe_logreg}
\setlength\tabcolsep{3pt}
\begin{tabular}{lcccccccccccr} 
     \toprule 
       \small Strategy & \small None  & \small CW  & \small RUS & \small ROS & \small NM1  & \small BS1 & \small BS2   & \scriptsize SMOTE & \scriptsize SMOTE & \small MGS & \scriptsize CT & \scriptsize Forest   \\
       &   &   &  &  & \small   &  &  &   & \scriptsize $K$-tuned &  \scriptsize ($d+1$) & \scriptsize GAN & \scriptsize Diff  \\
     \midrule
     \small CreditCard \tiny($0.2\%$)\tiny($\pm 0.007$)&0.618	&0.726	&0.456	&0.733	&0.008	&0.635	&0.559	&0.707	&0.715	&0.677&0.714	&\textbf{0.753} \\
     \small Abalone \tiny($1\%$)\tiny($\pm 0.012$)&0.086	&0.080	&0.078	&0.079	&0.058	&0.077	&0.084	&0.077	&0.077	&0.076	&0.051	&\textbf{0.088}   \\
     \small \textit{Phoneme} \tiny($1\%$)\tiny($\pm 0.015$)&\textbf{0.091}	&0.060	&0.066	&0.063	&0.052	&0.043	&0.044	&0.047	&0.044	&0.043	&0.040	&0.068 \\
     \small \textit{Haberman} \tiny($10\%$)\tiny($\pm 0.029$)&0.348	&0.365	&0.326	&0.360	&0.350	&0.359	&0.357	&0.358	&\textbf{0.366}	&0.353	&0.230	&0.365   \\
     \small \textit{MagicTel} \tiny ($20\%$) \tiny($\pm 0.004$)&\textbf{0.558}	&0.525	&0.524	&0.524	&0.178	&0.464	&0.462	&0.521	&0.521	&0.518	&0.423	&0.525   \\
     \small \textit{California} \tiny ($1\%$)\tiny($\pm 0.013$)&\textbf{0.316}	&0.278	&0.211	&0.278	&0.049	&0.313	&0.305	&0.293	&0.290	&0.309	&0.103	&0.291    \\
 \bottomrule
\end{tabular}
\end{table*}

\begin{table}[h]
\centering
\caption{Extremely imbalanced data sets PR AUC. Logistic regression reimplemented in PyTorch using the implementation of \cite{cao2019learning}.} 
\label{tab:reglog_longtailed}
\begin{tabular}{ lcccr } 
 \toprule
 Strategy  & None & LDAM loss& Focal loss  & L2RW\\ 
 \midrule
    \small CreditCard & $0.677\pm0.041$  & $0.677\pm0.041$ & $\mathbf{0.776}\pm0.007$ & $0.714 \pm 0.033$ \\
    \small Abalone  & $0.048\pm0.0143$  & $0.048\pm0.014$ & $0.085\pm0.013$ & $\mathbf{0.091}\pm0.010$ \\
    \small \textit{Phoneme} \tiny ($1\%$) & $\mathbf{0.078}\pm0.029$  & $\mathbf{0.078}\pm0.029$ & $0.055\pm0.016$ & $0.033\pm0.003$  \\
     \small \textit{Haberman \tiny ($10\%$)} & $0.273\pm0.273$  & $0.273\pm0.273$ & $\mathbf{0.349}\pm0.026$ & $0.154\pm0.015$\\ 
     \small \textit{MagicTel} \tiny ($20\%$) & $0.521\pm0.031$  & $0.521\pm0.031$ & $\mathbf{0.555}\pm0.002$ & $0.404\pm0.012$\\
     \small \textit{California} \tiny ($1\%$) & $\mathbf{0.311}\pm0.012$  & $\mathbf{0.311}\pm0.012$ & $0.250\pm0.010$ & $0.238\pm0.018$\\
 \bottomrule
\end{tabular}
\end{table}

\begin{table}
\centering
\caption{LightGBM (not tuned) PR AUC computed on three \textbf{multiclass} data sets (average of one-vs-rest PR AUC).}
\label{tab:resultats_multiclass_lgbm_pr}
\setlength\tabcolsep{4pt}
\begin{tabular}{lcccccccccccr}
     \toprule 
       \small Strategy & \small None  & \small CW  & \small RUS & \small ROS & \small NM1  & \small BS1 & \small BS2   & \scriptsize SMOTE & \scriptsize SMOTE& \small MGS & \scriptsize CT & \scriptsize Forest   \\
       &   &   &  &  & \small   &  &  &   & \scriptsize $K$-tuned &  \scriptsize ($d+1$) & \scriptsize GAN & \scriptsize Diff  \\
     \midrule
     \small Wine  &0.496	&0.496	&0.418	&0.491	&0.412	&0.496	&\textbf{0.502}	&0.490 &0.462	&0.436	&0.430	&0.454  \\
     \tiny std  &\tiny0.001	&\tiny0.006	&\tiny0.006	&\tiny0.003	&\tiny0.004	&\tiny0.001	&\tiny0.000	&\tiny0.001	&\tiny0.003	&\tiny0.006	&\tiny0.003	&\tiny0.003 \\
     \small Yeast   &0.673	&0.671	&0.647	&0.680	&0.619	&0.671	&0.677	&0.673	&0.683	&\textbf{0.689}	&0.686	&0.684 \\
     \tiny std  &\tiny0.002	&\tiny0.000	&\tiny0.006	&\tiny0.006	&\tiny0.002	&\tiny0.002	&\tiny0.000	&\tiny0.002	&\tiny0.003	&\tiny0.000	&\tiny0.008	&\tiny0.007 \\
     \small Ecoli  &\textbf{0.880}	&0.873	&0.848	&0.873	&0.833	&0.866	&0.872	&0.871	&0.871	&0.867	&0.868	&0.868 	 \\
     \tiny std  &\tiny0.002	&\tiny0.001	&\tiny0.008	&\tiny0.012	&\tiny0.009	&\tiny0.010	&\tiny0.002	&\tiny0.003	&\tiny0.015	&\tiny0.004	&\tiny0.005	&\tiny0.006 \\
     \bottomrule
\end{tabular}
\end{table}

\end{document}